%% file: aistats2024_TW_Final.tex
\documentclass[twoside]{article}

\usepackage[accepted]{aistats2024}
\usepackage{amsmath,amsfonts,amssymb,amsthm,mathrsfs,accents}
\usepackage{algorithmic}
\usepackage{algorithm}
\usepackage[breaklinks=true,bookmarks=false]{hyperref}
\usepackage[caption=false,font=small,labelfont=sf,textfont=sf]{subfig}
\usepackage{verbatim}
\usepackage{graphicx}
\usepackage{epstopdf}
\usepackage{float,url,multirow,enumitem,makecell}

\makeatletter
\def\widebar{\accentset{{\cc@style\underline{\mskip10mu}}}}
\makeatother

\makeatletter
\newsavebox\myboxA
\newsavebox\myboxB
\newlength\mylenA

\newcommand*\xoverline[2][0.5]{%
    \sbox{\myboxA}{$\m@th#2$}%
    \setbox\myboxB\null
    \ht\myboxB=\ht\myboxA%
    \dp\myboxB=\dp\myboxA%
    \wd\myboxB=#1\wd\myboxA
    \sbox\myboxB{$\m@th\overline{\copy\myboxB}$}
    \setlength\mylenA{\the\wd\myboxA}
    \addtolength\mylenA{-\the\wd\myboxB}%
    \ifdim\wd\myboxB<\wd\myboxA%
       \rlap{\hskip 0.8\mylenA\usebox\myboxB}{\usebox\myboxA}%
    \else
        \hskip -0.5\mylenA\rlap{\usebox\myboxA}{\hskip 0.5\mylenA\usebox\myboxB}%
    \fi}
\makeatother

\newcommand\scalemath[2]{\scalebox{#1}{\mbox{\ensuremath{\displaystyle #2}}}}

\newcommand*\tcircle[1]{%
  \raisebox{-0.5pt}{%
    \textcircled{\fontsize{7pt}{0}\fontfamily{phv}\selectfont #1}%
  }%
}

\theoremstyle{plain}
\newtheorem{theorem}{Theorem}

\newtheorem{lemma}{Lemma}

\newtheorem{corollary}{Corollary}
\newtheorem{definition}{Definition}

\graphicspath{{Figures/}}

\input{_defs}

\begin{document}

%

%

\twocolumn[

\aistatstitle{Robust Data Clustering with Outliers via Transformed Tensor Low-Rank Representation}

\aistatsauthor{ Tong Wu }

\aistatsaddress{ Beijing Institute for General Artificial Intelligence  \\  \texttt{wutong@bigai.ai} } ]

\begin{abstract}
Recently, tensor low-rank representation (TLRR) has become a popular tool for tensor data recovery and clustering, due to its empirical success and theoretical guarantees. However, existing TLRR methods consider Gaussian or gross sparse noise, inevitably leading to performance degradation when the tensor data are contaminated by outliers or sample-specific corruptions. This paper develops an outlier-robust tensor low-rank representation (OR-TLRR) method that provides outlier detection and tensor data clustering simultaneously based on the t-SVD framework. For tensor observations with arbitrary outlier corruptions, OR-TLRR has provable performance guarantee for exactly recovering the row space of clean data and detecting outliers under mild conditions. Moreover, an extension of OR-TLRR is proposed to handle the case when parts of the data are missing. Finally, extensive experimental results on synthetic and real data demonstrate the effectiveness of the proposed algorithms. We release our code at \url{https://github.com/twugithub/2024-AISTATS-ORTLRR}.
\end{abstract}

\section{INTRODUCTION}

In this work, we are interested in clustering 3-way tensor data in the presence of outliers. Formally, suppose we are given a data tensor $\bcX \in \R^{n_1 \times n_2 \times n_3}$ with each lateral slice corresponding to one data sample, and we know the tensor can be decomposed as
\begin{align}
\bcX = \bcL_0 + \bcE_0,
\end{align}
where $\bcL_0$ is a low-rank tensor with the tensor columns drawn from a union of tensor subspaces (see Section~\ref{ssec:prelim}), and $\bcE_0$ is a column-sparse tensor. Both components are of arbitrary magnitudes. In particular, we do not know the locations of the nonzero columns of $\bcE_0$, not even how many there are. We aim to exactly recover the low-rank tensor $\bcL_0$ from $\bcX$ and cluster the samples in $\bcL_0$ according to their respective tensor subspaces. This problem is important for many applications, including image/video denoising and inpainting \cite{LiuMWY.PAMI2013,LuFCLLY.PAMI2020}, data clustering \cite{ElhamifarV.PAMI2013,ZhouLFLY.PAMI2021}, background subtraction \cite{XiaCSL.TNNLS2021}, and network traffic monitoring \cite{MardaniMG.TSP2015}.

This problem has been well studied in the matrix domain. Subspace clustering methods, for example, involve treating vectorized data samples as lying near a union of subspaces and clustering the samples based on the self-expressive model \cite{ElhamifarV.PAMI2013,LiuLYSYM.PAMI2013}. While these matrix techniques are remarkably effective, one major shortcoming is that they can only handle 2-way (matrix) data. To deal with multi-way data such as color images, hyperspectral images and videos, one needs to reshape the tensor data into a matrix. Such a vectorization process would destroy the original spatial structure of data and lead to compromised performance \cite{LuFCLLY.PAMI2020}. This suggests that it might be useful to maintain the intrinsic structure of multi-dimensional data for efficient and reliable subsequent processing.

Recently, several tensor-based clustering algorithms that exploit the spatial aspects of tensor data using multilinear algebra tools have been proposed \cite{KernfeldAK.arxiv2014,FuGTLH.TNNLS2016,ZhangLJLS.SPL2018,WuB.SPL2018,Wu.PR2020,ZhouLFLY.PAMI2021,Yang.etal.TNNLS2022,Wu.TCSVT2023}. In particular, motivated by the notion of the tensor-tensor product (t-product) \cite{KilmerM.LAA2011} which generalizes the matrix multiplication for third-order tensors based on the Discrete Fourier Transform (DFT), one can assume that multi-way data lie near a union of tensor subspaces \cite{KernfeldAK.arxiv2014,WuB.SPL2018,Wu.PR2020,ZhouLFLY.PAMI2021}. As a consequence, each data sample in a union of tensor subspaces can be well represented by a t-linear combination of other points in the dataset. Interestingly, the t-product can be generalized by replacing DFT with any invertible linear transforms \cite{KernfeldKA.LAA2015}. Such a new transform based t-product is of great interest because it allows to use data-adaptive transforms for different types of tensor data. Driven by this advantage, extensive numerical examples have shown its effectiveness in many applications \cite{LuPW.CVPR2019,SongNZ.NLAA2020,JiangNZH.TIP2020,Lu.ICCV2021,KongLL.ML2021,Wang.etal.TIP2021,Qin.etal.TIP2022,Yang.etal.TNNLS2022,JiangZZN.TNNLS2023}.

However, all the aforementioned tensor clustering methods simply assume the noise is sparse and uniformly distributed across $\bcE_0$. Thus, such methods cannot effectively handle outliers or sample-specific corruptions, which are common in real scenarios because of sensor failures, uncontrolled environments, etc. Another limitation of existing tensor subspace clustering methods is that most of them are designed for complete data. When dealing with incomplete data, we need to first recover the incomplete tensor data using any off-the-shelf tensor completion methods and then cluster the recovered data. This two-step approach may lead to sub-optimal clustering results because the downstream clustering task is independent of the optimization problem that is used for tensor recovery \cite{Yang.etal.TNNLS2022}. While the work in \cite{Yang.etal.TNNLS2022} can simultaneously recover and cluster the incomplete and noisy data, it still requires learning a dictionary from the incomplete data before clustering can be carried out.

Our goal in this paper is to perform robust low-rank tensor analysis on both complete and missing data for improved data clustering in the presence of outliers. We put forth an outlier-robust tensor low-rank representation (OR-TLRR) method which clusters the observed tensor data and detects outliers through convex optimization. Our model is more generic as it is allowed to use any invertible linear transforms that satisfy certain conditions. Interestingly, as we show in Section~\ref{ssec:theory}, the row space of $\bcL_0$ determines the segmentation result of the ``authentic" samples. We theoretically show that OR-TLRR guarantees the exact recovery of the row space of $\bcL_0$ when $\bcL_0$ and $\bcE_0$ satisfy certain assumptions. Different from \cite{ZhouLFLY.PAMI2021,Yang.etal.TNNLS2022} that use the $\ell_1$ norm of $\bcE$ to handle sparse noise, we focus on a different problem and thus there exist critical differences in theoretical analysis and guarantees. Lastly, we generalize the OR-TLRR algorithm so that it can obtain a tensor low-rank representation of the incomplete data using only the observed entries. Experiments on synthetic and real data show the efficacy of the proposed algorithms. The main contributions of this work can be summarized as follows:
\begin{itemize}
\item We propose a novel OR-TLRR method for low-rank tensor analysis based on any invertible linear transforms. OR-TLRR handles outliers and sample-specific corruptions which cannot be well handled by existing tensor clustering methods.
\item We provide theoretical performance guarantees for OR-TLRR: under mild conditions, OR-TLRR can exactly recover the row space of $\bcL_0$ and detect outliers.
\item We extend OR-TLRR for tensor subspace clustering with missing entries called OR-TLRR by entry wise zero fill (OR-TLRR-EWZF), in which the self-expressiveness error is restricted only to the observed entries.
\end{itemize}

\section{RELATED WORK}
\label{sec:review}

Existing subspace clustering methods \cite{LiuY.ICCV2011,ErikssonBN.AISTATS2012,LiuXY.AISTATS2012,ElhamifarV.PAMI2013,LiuLYSYM.PAMI2013,TangLSZ.TNNLS2014,YangRV.ICML2015,YouRV.CVPR2016,LuFLMY.PAMI2019} that are based on the self-expressive model mainly follow a two-stage approach: ($i$) learning an affinity matrix by solving an optimization problem from the data using different regularizations, and ($ii$) applying spectral clustering (e.g., Ncut \cite{ShiM.PAMI2000}) on the affinity matrix. The key component in subspace learning is to construct a good affinity graph. In particular, the matrix-based low-rank representation (LRR) \cite{LiuXY.AISTATS2012,LiuLYSYM.PAMI2013} algorithm imposes a low-rank constraint on the representation coefficient matrix to capture the global structure of data. Due to its physical interpretation and promising performance, many variants of LRR have been further studied. The latent LRR \cite{LiuY.ICCV2011} extends LRR to handle the hidden effects and integrates subspace segmentation and feature extraction into a unified framework. The structure-constrained LRR \cite{TangLSZ.TNNLS2014} improves LRR for the disjoint subspace segmentation. Typical subspace clustering methods, despite their effectiveness, neglect the spatial information within the data. Due to the powerful learning ability, deep clustering methods can capture the complex nonlinear structure within the data \cite{XieGF.ICML2016,JiZLSR.NIPS2017,JiVH.ICCV2019,LiHLPZP.AAAI2021}. However, they cannot solve the problem of clustering high-dimensional but small-scale tensor data considered in this work. In addition, they lack the theoretical guarantee for data clustering.

In order to better capture the geometry of tensor data, a high-order model based upon t-product \cite{KilmerM.LAA2011,KilmerBHH.SIAM2013} that has in particular helped advance the state-of-the-art in many applications is the union-of-tensor-subspaces model \cite{KernfeldAK.arxiv2014,WuB.SPL2018,ZhouLFLY.PAMI2021}, which dictates that data lie near a mixture of low-dimensional tensor subspaces. One striking advantage of the t-product is its capability in capturing the ``spatial-shifting" correlation that is ubiquitous in real-world data. For the purpose of inducing tensor subspace-preserving solutions, existing methods use different regularizations on the representation coefficients. For example, sparse submodule clustering (SSmC) \cite{KernfeldAK.arxiv2014} imposes a group sparsity constraint on the representation tensor, while the tensor nuclear norm \cite{LuFCLLY.PAMI2020} is adopted in tensor low-rank representation (TLRR) \cite{ZhouLFLY.PAMI2021} as a convex surrogate of the tensor tubal rank. Motivated by the general t-product definition performed on unitary transforms \cite{KernfeldKA.LAA2015}, a more general TLRR model that can handle incomplete data is proposed in \cite{Yang.etal.TNNLS2022}. However, both \cite{ZhouLFLY.PAMI2021} and \cite{Yang.etal.TNNLS2022} assume that noise is Gaussian or sparse and thus cannot effectively handle outlier corruptions. Indeed, robust tensor subspace clustering with outliers appears to be a challenging problem that has not been well studied both in practice and in theory. We conclude by noting that tensor nuclear norm has been utilized for multi-view clustering \cite{XieTZLZQ.IJCV2018,SuHWL.SP2023}. Nonetheless, these works are still based on the union of subspaces model and reduce to different variants of \cite{ElhamifarV.PAMI2013,LiuLYSYM.PAMI2013} for single-view data.

\section{NOTATIONS AND PRELIMINARIES}
\label{sec:notationprelim}

In this section, we will introduce the notations and give the basic definitions used throughout the paper.

\subsection{Notations}

For brevity, we summarize the notations in Table~\ref{tab:notation}. Throughout this paper, the fields of real number and complex number are denoted as $\R$ and $\C$, respectively. The definitions of the identity tensor and the conjugate transpose of a tensor will be given in Section~\ref{ssec:prelim}. The frontal-slice-wise product of two third-order tensors $\bcA \in \C^{n_1 \times n_2 \times n_3}$ and $\bcB \in \C^{n_2 \times n_4 \times n_3}$, denoted by $\bcA \odot \bcB$, is a tensor $\bcC \in \C^{n_1 \times n_4 \times n_3}$ such that $\bC^{(i)} = \bA^{(i)} \bB^{(i)}$, $i = 1, \dots, n_3$ \cite{KernfeldKA.LAA2015}.

\begin{table}[t]
\tiny
\centering
\caption{Notational convention.}
\setlength\tabcolsep{-0.03em}{\begin{tabular}{l|l|l|l}
\hline
$\bcA$ & A tensor. & $\ba$ & A vector.  \\
$\bA$ & A matrix. & $a$ & A scalar.  \\
\hline
$\bI_n$ & The identity matrix. & $\bA^H$ & The conjugate transpose of $\bA$.  \\
$\bA_{i,j}$ & The ($i,j$)-th element of $\bA$. & $\tr(\bA)$ & The trace of $\bA$.  \\
$\ba_j$ & The $j$-th column of $\bA$. & $\|\bA\|$ & The largest singular value of $\bA$.  \\
$\ba_j^T$ & The $j$-th row of $\bA$. & $\|\bA\|_{\ast}$ & Sum of singular values of $\bA$.  \\
\hline
$\bcI_n$ & The identity tensor. & $\bA^{(i)}$ & $\bA^{(i)} = \bcA(:,:,i)$.  \\
$\bcA_{i,j,k}$ & The ($i,j,k$)-th entry of $\bcA$. & $\bcA_{(i)}$ & $\bcA_{(i)} = \bcA(:,i,:)$.  \\
$\bcA(:,i,:)$ & The $i$-th lateral slice of $\bcA$. & $\bcA^H$ & The conjugate transpose of $\bcA$.  \\
$\bcA(:,:,i)$ & The $i$-th frontal slice of $\bcA$. & $\|\bcA\|_{2,1}$ & $\|\bcA\|_{2,1} = \sum_j \|\bcA(:,j,:)\|_F$.  \\
$\bcA(i,j,:)$ & The ($i,j$)-th tube of $\bcA$. & $\|\bcA\|_F$ & $\|\bcA\|_F = \sqrt{\sum_{i,j,k} |\bcA_{i,j,k}|^2}$.  \\
\hline
$\bcP_{\bcU} (\bcA)$ & $\bcP_{\bcU} (\bcA) = \bcU \ast_{\boldsymbol{L}} \bcU^H \ast_{\boldsymbol{L}} \bcA$. & $\bcP_{\mathbf{\Theta}}$ & The projection onto $\mathbf{\Theta}$.  \\
$\bcP_{\bcV} (\bcA)$ & $\bcP_{\bcV} (\bcA) = \bcA \ast_{\boldsymbol{L}} \bcV \ast_{\boldsymbol{L}} \bcV^H$. & $\bcP_{\mathbf{\Theta}^{\perp}}$ & The projection onto $\mathbf{\Theta}^c$.  \\
$\bcP_{\bcV}^L (\bcA)$ & $\bcP_{\bcV}^L (\bcA) = \bcV \ast_{\boldsymbol{L}} \bcV^H \ast_{\boldsymbol{L}} \bcA$. & $\bcB (\bcE)$ & $\{ \tilde{\bcE}: \tilde{\bcE}(:,i,:) = \frac{\bcE(:,i,:)}{\|\bcE(:,i,:)\|_F}$  \\
 & & & $(i \in \mathbf{\Theta}); \bcP_{\mathbf{\Theta}^{\perp}} \tilde{\bcE} = \mathbf{0} \}$.  \\
\hline
\end{tabular}}
\label{tab:notation}
\end{table}

\subsection{Preliminaries}
\label{ssec:prelim}

The third-order t-product under an invertible linear transform $L$ establishes the groundwork for the development of our algorithm and its definition was first given in \cite{KernfeldKA.LAA2015}. In this paper, the transformation matrix $\boldsymbol{L}$ defining the transform $L$ is restricted to be orthogonal, i.e., $\boldsymbol{L} \in \C^{n_3 \times n_3}$ satisfying
\begin{align}    \label{eqn:lconstraint}
\boldsymbol{L} \boldsymbol{L}^H = \boldsymbol{L}^H \boldsymbol{L} = \tau \bI_{n_3},
\end{align}
where $\tau > 0$ is a constant. We define the associated linear transform $L(\cdot) : \R^{n_1 \times n_2 \times n_3} \to \C^{n_1 \times n_2 \times n_3}$ which gives $\xoverline{\bcA}$ by performing a linear transform on any $\bcA \in \R^{n_1 \times n_2 \times n_3}$ along the 3rd dimension with inverse mapping $L^{-1}(\cdot)$ as
\begin{align}    \label{eqn:mode3prod}
\xoverline{\bcA} = L(\bcA) = \bcA \times_3 \boldsymbol{L} ~ \text{and} ~ L^{-1}(\bcA) = \bcA \times_3 \boldsymbol{L}^{-1}.
\end{align}
Here, $\times_3$ denotes the mode-3 tensor-matrix product \cite{KoldaB.Rev2009}, i.e., $\bcX \times_3 \bA = \mathtt{fold}_3 (\bA \cdot \mathtt{unfold}_3 (\bcX))$, where $\mathtt{unfold}_3 : \C^{n_1 \times n_2 \times n_3} \to \C^{n_3 \times n_1 n_2}$ is the mode-3 unfolding operator and $\mathtt{fold}_3$ is its inverse operator \cite{KoldaB.Rev2009}.
\begin{definition}[t-product \cite{KernfeldKA.LAA2015}]
The t-product of any $\bcA \in \C^{n_1 \times n_2 \times n_3}$ and $\bcB \in \C^{n_2 \times n_4 \times n_3}$ under the invertible linear transform $L$ in \eqref{eqn:mode3prod}, is defined as $\bcC = \bcA \ast_{\boldsymbol{L}} \bcB \in \C^{n_1 \times n_4 \times n_3}$ such that $L(\bcC) = L(\bcA) \odot L(\bcB)$.
\end{definition}
We denote $\widebar{\bA} \in \C^{n_1n_3 \times n_2 n_3}$ as
\begin{align*}
\widebar{\bA} = \mathtt{bdiag}(\xoverline{\bcA}) =
\begin{bmatrix}
\widebar{\bA}^{(1)} &  &  \\
 & \ddots &  \\
 & & \widebar{\bA}^{(n_3)}
\end{bmatrix},
\end{align*}
where $\mathtt{bdiag}$ is an operator which maps $\xoverline{\bcA}$ to $\widebar{\bA}$. The conjugate transpose of a tensor $\bcA \in \C^{n_1 \times n_2 \times n_3}$ is the tensor $\bcA^H \in \C^{n_2 \times n_1 \times n_3}$ satisfying $L(\bcA^H)^{(i)} = (L(\bcA)^{(i)})^H$, $i = 1, 2, \dots, n_3$. The identity tensor $\bcI_n \in \C^{n \times n \times n_3}$ is a tensor such that each frontal slice of $L(\bcI_n) = \xoverline{\bcI}_n$ is the identity matrix $\bI_n$. Then $\bcI_n = L^{-1}(\xoverline{\bcI}_n)$ gives the identity tensor under $L$. A tensor $\bcQ \in \C^{n \times n \times n_3}$ is orthogonal if it satisfies $\bcQ \ast_{\boldsymbol{L}} \bcQ^H = \bcQ^H \ast_{\boldsymbol{L}} \bcQ = \bcI_n$. A tensor is called f-diagonal if each of its frontal slices is a diagonal matrix. The spectral norm of $\bcA$ is defined as $\| \bcA \| = \| \widebar{\bA} \|$.
\begin{definition}[t-SVD \cite{SongNZ.NLAA2020}]
Let $L$ be any invertible linear transform in \eqref{eqn:mode3prod}. For any $\bcA \in \C^{n_1 \times n_2 \times n_3}$, it can be factorized as $\bcA = \bcU \ast_{\boldsymbol{L}} \bcS \ast_{\boldsymbol{L}} \bcV^H$, where $\bcU \in \C^{n_1 \times n_1 \times n_3}$ and $\bcV \in \C^{n_2 \times n_2 \times n_3}$ are orthogonal, and $\bcS \in \C^{n_1 \times n_2 \times n_3}$ is an f-diagonal tensor.
\end{definition}
\begin{definition}[Tensor tubal rank \cite{SongNZ.NLAA2020}]
For any $\bcA \in \C^{n_1 \times n_2 \times n_3}$, the tensor tubal rank $\text{rank}_t(\bcA)$ under $L$ in \eqref{eqn:mode3prod} is defined as the number of nonzero singular tubes of $\bcS$, i.e., $\text{rank}_t(\bcA) = \# \{ i: \bcS(i,i,:) \neq \mathbf{0}\}$, where $\bcS$ is from the t-SVD of $\bcA = \bcU \ast_{\boldsymbol{L}} \bcS \ast_{\boldsymbol{L}} \bcV^H$.
\end{definition}
For computational simplicity, we use the skinny t-SVD throughout this paper. For any $\bcA \in \C^{n_1 \times n_2 \times n_3}$ with tubal rank $r$, the skinny t-SVD is given by $\bcA = \bcU \ast_{\boldsymbol{L}} \bcS \ast_{\boldsymbol{L}} \bcV^H$, where $\bcU \in \C^{n_1 \times r \times n_3}$, $\bcS \in \C^{r \times r \times n_3}$, and $\bcV \in \C^{n_2 \times r \times n_3}$, in which $\bcU^H \ast_{\boldsymbol{L}} \bcU = \bcV^H \ast_{\boldsymbol{L}} \bcV = \bcI_r$.
\begin{definition}[Tensor nuclear norm \cite{SongNZ.NLAA2020}]
The tensor nuclear norm of $\bcA \in \C^{n_1 \times n_2 \times n_3}$ under $L$ in \eqref{eqn:mode3prod} is defined as $\| \bcA \|_{\ast} = \frac{1}{\tau} \sum_{i=1}^{n_3} \| \widebar{\bA}^{(i)} \|_{\ast} = \frac{1}{\tau} \| \widebar{\bA} \|_{\ast}$.
\end{definition}
\begin{definition}[Tensor subspace \cite{ZhouLFLY.PAMI2021}]    \label{def:tensorsubspace}
Given a third-order tensor $\bcD = [ \bcD_{(1)}, \dots, \bcD_{(p)} \} \in \R^{n_1 \times p \times n_3}$ in which the elements $\bcD_{(i)}$'s are linearly independent, i.e., there is not a nonzero $\bcC \in \R^{p \times 1 \times n_3}$ satisfying $\bcD \ast_{\boldsymbol{L}} \bcC = \mathbf{0}$. Then the set $\cS_{n_3}^{n_1} = \{ \bcY | \bcY = \bcD \ast_{\boldsymbol{L}} \bcC, \forall \bcC \in \R^{p \times 1 \times n_3} \}$ is called a tensor subspace of dimension $\text{dim}(\cS_{n_3}^{n_1}) = p$. Here, $\bcD_{(1)}, \dots, \bcD_{(p)}$ are called the spanning basis of $\cS_{n_3}^{n_1}$.
\end{definition}
\begin{definition}[Tensor column space \cite{ZhouF.CVPR2017}]    \label{def:columnsubspace}
For an arbitrary tensor $\bcA \in \R^{n_1 \times n_2 \times n_3}$ with tubal rank $r$, assume the t-SVD of $\bcA$ is $\bcA = \bcU \ast_{\boldsymbol{L}} \bcS \ast_{\boldsymbol{L}} \bcV^H$. Then its column space $\text{Range} (\bcA)$ is the t-linear space spanned by the columns of $\bcU \in \R^{n_1 \times r \times n_3}$, i.e., $\text{Range} (\bcA) = \{ \bcY | \bcY = \bcU \ast_{\boldsymbol{L}} \bcC, \forall \bcC \in \R^{r \times 1 \times n_3} \}$.
\end{definition}

\section{OUTLIER-ROBUST TLRR}
\label{sec:completealgo}

In this section, we propose our OR-TLRR model for tensor data clustering in the presence of outliers, followed by the theoretical analysis on the recovery guarantee of OR-TLRR. The optimization details and proofs of all the Lemmas and Theorems are provided in the supplementary material.

\subsection{Formulation of OR-TLRR}

Assume that we are given $n_2$ data samples of size $n_1 \times n_3$, stored as lateral slices in a tensor $\bcX = \bcL_0 + \bcE_0 \in \R^{n_1 \times n_2 \times n_3}$, where the columns in $\bcL_0$ belong to $c$ distinct tensor subspaces and $\bcE_0$ is a tensor corresponding to outliers. The challenge is to recover the low-rank tensor $\bcL_0$ from $\bcX$ and segment the samples in $\bcL_0$ into $c$ clusters. In many cases, the error term $\bcE_0$ has sparse column supports compared to the data size. This implies that the tensor $\ell_{2,1}$-norm is appropriate to characterize $\bcE_0$. Let $L$ be any invertible linear transform in \eqref{eqn:mode3prod} and it satisfies \eqref{eqn:lconstraint}. Mathematically, the OR-TLRR model can be described as the following optimization program based on the tensor nuclear norm:
\begin{align}    \label{eqn:ortlrrprob}
\min_{\bcZ,\bcE} \|\bcZ\|_{\ast} + \lambda \|\bcE\|_{2,1} \quad \text{s.t.} \quad \bcX = \bcX \ast_{\boldsymbol{L}} \bcZ + \bcE,
\end{align}
where $\bcZ \in \R^{n_2 \times n_2 \times n_3}$ corresponds to a coefficient tensor and $\lambda > 0$ is a parameter. In this work, we directly use the raw tensor data $\bcX$ as the dictionary. In this setting, the learned representation tensor $\bcZ$ depicts the relations among samples and can be used for clustering. Suppose we have an optimal solution $(\bcZ_{\star}, \bcE_{\star})$ for \eqref{eqn:ortlrrprob}. For clustering, we first perform outlier detection based on the magnitude of the residual $\bcE_{\star}$, and we use $\mathbf{\Theta}$ to denote the indices of the ``detected" outliers. Then we compute an affinity matrix $\hat{\bZ} \in \R^{|\mathbf{\Theta}^c| \times |\mathbf{\Theta}^c|}$ by setting $\hat{\bZ} = \frac{1}{2 n_3} \sum_{k=1}^{n_3} (|\bZ_{\star [\mathbf{\Theta}^c,\mathbf{\Theta}^c]}^{(k)}| + |(\bZ_{\star [\mathbf{\Theta}^c,\mathbf{\Theta}^c]}^{(k)})^H|)$, where $\bA_{[\mathbf{\Theta}, \mathbf{\Theta}]}$ is a submatrix of $\bA$ corresponding to rows and columns indexed by $\mathbf{\Theta}$. Finally, clustering of the ``detected" normal samples can be carried out using existing tools such as Ncut \cite{ShiM.PAMI2000}.

\subsection{Tensor Linear Representation}

Before introducing the proposed algorithm for clustering, we explain the connection between the tensor linear representation $\bcX = \bcA \ast_{\boldsymbol{L}} \bcZ$ and canonical vector linear representation. We introduce two operators: $\mathtt{vec}$ and $\mathtt{ivec}$, where $\mathtt{vec}$ vectorizes each sample $\bcX_{(j)}$ to a vector $\bx_j$ of dimension $n_1n_3$ and $\mathtt{ivec}$ is its inverse operation. For each sample $\bcX_{(j)}$, if its vectorized version $\bx_j$ can be linearly represented by a dictionary $\bA \in \R^{n_1n_3 \times p}$ as
\begin{align}    \label{eqn:linearrep}
\bx_j = \bA \bz_j, ~ \forall j = 1, \dots, n_2,
\end{align}
then one can always find two tensors $\bcA \in \R^{n_1 \times p \times n_3}$ and $\bcZ \in \C^{p \times n_2 \times n_3}$ such that tensor linear representation $\bcX = \bcA \ast_{\boldsymbol{L}} \bcZ$ holds, as stated below.
\begin{theorem}    \label{thm:linearrep}
If \eqref{eqn:linearrep} holds, then there exist two tensors $\bcA \in \R^{n_1 \times p \times n_3}$ and $\bcZ \in \C^{p \times n_2 \times n_3}$ such that
\begin{align}    \label{eqn:tensorlinearrep}
\bcX_{(j)} = \mathtt{ivec} (\bx_j) = \bcA \ast_{\boldsymbol{L}} \bcZ_{(j)}, \quad \forall j = 1, \dots, n_2,
\end{align}
where $\bcA$ can be constructed from $\bA$ by setting $\bcA_{(j)} = \mathtt{ivec} (\ba_j)$ and $\bcZ$ can be computed by $\bcZ = L^{-1} (\xoverline{\bcZ})$ in which $\xoverline{\bcZ}(:,j,k) = \bz_j$ $(k = 1, \dots, n_3)$. However, if there exists a tensor such that \eqref{eqn:tensorlinearrep} holds, \eqref{eqn:linearrep} may not hold.
\end{theorem}
Theorem~\ref{thm:linearrep} implies that tensor linear representation under linear transform can capture complex structures underlying the data which cannot be fully described by vector linear representation. We would like to emphasize that this result generalizes the one from \cite[Theorem 1]{ZhouLFLY.PAMI2021} which only considered the DFT. In their proof, some special properties of DFT were used, which lead to several key differences between the two proofs.

\subsection{Optimization of OR-TLRR}

We can use Alternating Direction Method of Multipliers (ADMM) \cite{LinLS.NIPS2011} to solve problem \eqref{eqn:ortlrrprob}. The major computation lies in the update of $\bcZ$, which requires computing $n_3$ SVDs of $n_2 \times n_2$ matrices in the transform domain. Thus directly applying ADMM will be highly time consuming when the number of data samples $n_2$ is large. To reduce the computational cost, we reformulate \eqref{eqn:ortlrrprob} as follows. Assume $\bcU_{\bcX} \ast_{\boldsymbol{L}} \bcS_{\bcX} \ast_{\boldsymbol{L}} \bcV_{\bcX}^H$ is the skinny t-SVD of $\bcX$ and $r_{\bcX} = \text{rank}_t(\bcX)$. One can replace $\bcX$ and $\bcZ$ in \eqref{eqn:ortlrrprob} with $\bcD = \bcU_{\bcX} \ast_{\boldsymbol{L}} \bcS_{\bcX} \in \R^{n_1 \times r_{\bcX} \times n_3}$ and $\bcV_{\bcX} \ast_{\boldsymbol{L}} \bcZ'$, respectively, where $\bcZ' \in \R^{r_{\bcX} \times n_2 \times n_3}$ is one variable to be updated. This gives the following equivalent formulation:
\begin{align}    \label{eqn:ortlrrreducedprob}
\min_{\bcZ',\bcE} \|\bcZ'\|_{\ast} + \lambda \|\bcE\|_{2,1} \quad \text{s.t.} \quad \bcX = \bcD \ast_{\boldsymbol{L}} \bcZ' + \bcE.
\end{align}
After obtaining a solution $(\bcZ'_{\star}, \bcE_{\star})$ to the problem \eqref{eqn:ortlrrreducedprob}, the optimal solution to \eqref{eqn:ortlrrprob} can be recovered by $(\bcV_{\bcX} \ast_{\boldsymbol{L}} \bcZ'_{\star}, \bcE_{\star})$. For any invertible linear transform $L$, the per-iteration complexity is $\cO(r_{\bcX} n_1 n_2 n_3 + r_{\bcX} (n_1 + n_2) n_3^2)$. For some special transforms, e.g., DFT, the per-iteration complexity is $\cO(r_{\bcX} n_1 n_2 n_3 + r_{\bcX} (n_1 + n_2) n_3 \log(n_3))$.

\subsection{Exact Recovery Performance Guarantee}
\label{ssec:theory}

Here we provide the theoretical performance guarantee for OR-TLRR. The skinny t-SVD of $\bcX$ and $\bcL_0$ are denoted as $\bcU_{\bcX} \ast_{\boldsymbol{L}} \bcS_{\bcX} \ast_{\boldsymbol{L}} \bcV_{\bcX}^H = \bcX$ and $\bcU_0 \ast_{\boldsymbol{L}} \bcS_0 \ast_{\boldsymbol{L}} \bcV_0^H = \bcL_0$, respectively. We use $\mathbf{\Theta}_0$ to denote the column support of $\bcE_0$. For convenience, we denote $n_{(1)} = \max(n_1, n_2)$ and $n_{(2)} = \min(n_1, n_2)$. To better illustrate our intuition, we begin with the ``ideal" case where there is no outlier in the data, i.e., $\bcX = \bcL_0$ and $\bcE_0 = \mathbf{0}$. Then the OR-TLRR problem degenerates to
\begin{align}    \label{eqn:ortlrrreducedform}
\min_{\bcZ} \|\bcZ\|_{\ast} \quad \text{s.t.} \quad \bcX = \bcX \ast_{\boldsymbol{L}} \bcZ.
\end{align}
As shown in \cite{ZhouLFLY.PAMI2021,Yang.etal.TNNLS2022}, this problem has a unique solution given by $\bcZ_{\star} = \bcV_0 \ast_{\boldsymbol{L}} \bcV_0^H$, which suggests that the solution of OR-TLRR identifies the row space of $\bcL_0$ in this special case. When the data tensor is contaminated by outliers, we expect the row space of $\bcL_0$ can still be exactly recovered. We first establish the following lemma, which states that the optimal solution to OR-TLRR always locates within the row space of raw tensor data $\bcX$.
\begin{lemma}    \label{lem:cleansol}
For any optimal solution $(\bcZ_{\star}, \bcE_{\star})$ to the OR-TLRR problem \eqref{eqn:ortlrrprob}, we have $\bcZ_{\star} \in \bcP_{\bcV_{\bcX}}^L$.
\end{lemma}
This lemma states that the column space of $\bcZ_{\star}$ is a subspace of $\bcV_{\bcX}$. Thus, in order for $\bcZ_{\star}$ to exactly recover $\bcV_0$, a necessary condition is that $\bcV_0$ is a subspace of $\bcV_{\bcX}$, i.e., $\bcV_0 \in \bcP_{\bcV_{\bcX}}^L$. Intuitively, if some outliers in $\bcE_0$ lie exactly on the span of the tensor subspaces, then we can think of $\bcX$ as containing more authentic samples than $\bcL_0$ and this condition will be violated. To show how it can hold, we establish the following lemma:
\begin{lemma}    \label{lem:sublemma}
If $\text{Range} (\bcL_0)$ and $\text{Range} (\bcE_0)$ are independent to each other, i.e., $\text{Range} (\bcL_0) \cap \text{Range} (\bcE_0) = \{ \mathbf{0} \}$, then $\bcV_0 \in \bcP_{\bcV_{\bcX}}^L$.
\end{lemma}

Next, similar to TLRR \cite{ZhouLFLY.PAMI2021}, exactly separating $\bcX$ as the low-rank term $\bcL_0$ plus the column-sparse term $\bcE_0$ requires $\bcL_0$ is not column-sparse and $\bcE_0$ is not low-rank. To characterize this intuition, we need the following two mild conditions.

\noindent \textbf{Incoherence Condition on Low-Rank Term:} Let $L$ be any invertible linear transform in \eqref{eqn:mode3prod} and it satisfies \eqref{eqn:lconstraint}. For $\bcL \in \R^{n_1 \times n_2 \times n_3}$, assume that $\text{rank}_t (\bcL) = r$ and it has the skinny t-SVD $\bcL = \bcU \ast_{\boldsymbol{L}} \bcS \ast_{\boldsymbol{L}} \bcV^H$. Then the tensor column-incoherence condition with parameter $\mu$ is defined as
\begin{align}    \label{eqn:condL}
\mu \geq \frac{n_2 \tau}{r} \max_{j=1, \dots, n_2} \| \hat{\bcV}^H \ast_{\boldsymbol{L}} \mathring{\boldsymbol{\ce}}_j \|_F^2,
\end{align}
where $\mathring{\boldsymbol{\ce}}_j$ is a tensor of size $n_2 \times 1 \times n_3$ with the entries of the $(j,1)$-th tube of $L(\mathring{\boldsymbol{\ce}}_j)$ equaling 1 and the rest equaling 0. A small value of $\mu$ implies the low-rank term $\bcL$ is not column-sparse.

\noindent \textbf{Ambiguity Condition on Column Sparse Term:} Another identifiability issue arises if the outlier tensor is both column-sparse and low-rank. To avoid this pathological situation, we introduce an unambiguity condition on $\bcE$:
\begin{align}    \label{eqn:condE}
\| \cB(\bcE) \| \leq \log(n_2) / \theta
\end{align}
for some constant $\theta \geq 4$. In fact, condition \eqref{eqn:condE} holds as long as the directions of the nonzero columns of $\bcE$ scatter sufficiently randomly. Thus it guarantees that the tensor $\bcE$ cannot be low-rank when the column sparsity of $\bcE$ is comparable to $n_2$. We have the following exact recovery guarantee for problem \eqref{eqn:ortlrrprob}.

\begin{theorem}    \label{thm:condortlrr}
Suppose $\text{Range} (\bcL_0) = \text{Range} (\bcP_{\mathbf{\Theta}_0^{\perp}} (\bcL_0))$ and $\bcE_0(:,j,:) \notin \text{Range} (\bcL_0)$ for $j \in \mathbf{\Theta}_0$. Then any solution $(\bcV_0 \ast_{\boldsymbol{L}} \bcV_0^H + \bcH, \bcE_0 - \bcX \ast_{\boldsymbol{L}} \bcH)$ to \eqref{eqn:ortlrrprob} with $\lambda = \frac{\theta}{4 \sqrt{\log(n_{(1)})} \| \bcX \|}$ exactly recovers the row space of $\bcL_0$ and the column support of $\bcE_0$ with a probability at least $1 - c_1 n_{(1)}^{-10}$, where $c_1$ is a positive constant, if the column support set $\mathbf{\Theta}_0$ is uniformly distributed among all sets of cardinality $|\mathbf{\Theta}_0|$ and
\begin{align}
\text{rank}_t (\bcL_0) \leq \frac{\rho_r n_2}{\mu n_1 n_3 \| \bcX \|^2} ~ \text{and} ~ |\mathbf{\Theta}_0| \leq \rho_s n_2,
\end{align}
where $\rho_r$ and $\rho_s$ are constants, $\bcX \ast_{\boldsymbol{L}} (\bcV_0 \ast_{\boldsymbol{L}} \bcV_0^H + \bcP_{\mathbf{\Theta}_0} \bcP_{\bcV_0}^L (\bcH))$ satisfies the column-incoherence condition \eqref{eqn:condL} and $\bcE_0 - \bcX \ast_{\boldsymbol{L}} \bcP_{\mathbf{\Theta}_0} \bcP_{\bcV_0}^L (\bcH)$ satisfies the unambiguity condition \eqref{eqn:condE}.
\end{theorem}

The incoherence and ambiguity conditions on $\hat{\bcL} = \bcX \ast_{\boldsymbol{L}} \hat{\bcZ}$ and $\hat{\bcE} = \bcE_0 - \bcX \ast_{\boldsymbol{L}} \bcP_{\mathbf{\Theta}_0} \bcP_{\bcV_0}^L (\bcH)$ is not surprising, where $\hat{\bcZ} = \bcV_0 \ast_{\boldsymbol{L}} \bcV_0^H + \bcP_{\mathbf{\Theta}_0} \bcP_{\bcV_0}^L (\bcH)$. In fact, the column space of $\hat{\bcZ}$ is the same as the row space of $\text{Range} (\bcL_0)$ and $\hat{\bcE}$ has the same column support as that of $\bcE_0$. Also, notice that $\bcX = \bcX \ast_{\boldsymbol{L}} \hat{\bcZ} + \hat{\bcE}$. So we can consider $\hat{\bcL}$ and $\hat{\bcE}$ as the underlying low-rank and sparse terms, and we assume incoherence and ambiguity conditions on them instead of $\bcL_0$ and $\bcE_0$. The above results demonstrate that with high probability OR-TLRR can exactly recover the row space of $\bcP_{\mathbf{\Theta}_0} (\bcL_0)$ and the support set $\mathbf{\Theta}_0$ of $\bcE_0$.

\section{OR-TLRR WITH MISSING ENTRIES}
\label{sec:missalgo}

Let us now consider the case where some entries of the data tensor $\bcX$ are missing. Specifically, let $\bcW \in \{0,1\}^{n_1 \times n_2 \times n_3}$ be a binary matrix such that $\bcW_{i,j,k} = 1$ if $\bcX_{i,j,k}$ is observed and $\bcW_{i,j,k} = 0$ otherwise. The locations of the observed entries for the entire data can be indexed by the set $\mathbf{\Omega} = \{ (i,j,k) : \bcW_{i,j,k} = 1 \}$. Given the observed entries of $\bcX$, $\{ \bcX_{i,j,k} \}_{(i,j,k) \in \mathbf{\Omega}}$, our goal is to segment the tensor columns of $\bcX$ into their corresponding tensor subspaces in the presence of outliers. Since we do not have access to the complete data, it is impossible to directly solve \eqref{eqn:ortlrrprob}. One approach is to fill the missing entries in $\bcX$ with 0s, i.e., to replace $\bcX$ by $\bcX_{\mathrm{miss}} = \bcW \odot \bcX$ and then solve \eqref{eqn:ortlrrprob}. As such, the constraint in \eqref{eqn:ortlrrprob} will become $\bcX_{\mathrm{miss}} = \bcX_{\mathrm{miss}} \ast_{\boldsymbol{L}} \bcZ + \bcE$ and we essentially enforce the $(i,j,k)$-th entry of the reconstructed tensor $\bcX_{\mathrm{miss}} \ast_{\boldsymbol{L}} \bcZ$ to be close to zero whenever $\bcW_{i,j,k} = 0$, while this term should be not penalized because we do not observe $\bcX_{i,j,k}$. Followed by \cite{YangRV.ICML2015}, we propose to incorporate the constraint $\mathscr{P}_{\mathbf{\Omega}}(\bcX_{\mathrm{miss}}) = \mathscr{P}_{\mathbf{\Omega}}(\bcX_{\mathrm{miss}} \ast_{\boldsymbol{L}} \bcZ + \bcE)$ into \eqref{eqn:ortlrrprob} and solve the following problem:
\begin{align}    \label{eqn:ortlrrmissprob}
& \min_{\bcZ,\bcE} \|\bcZ\|_{\ast} + \lambda \|\mathscr{P}_{\mathbf{\Omega}}(\bcE)\|_{2,1}    \nonumber \\
&~ \text{s.t.} \quad \mathscr{P}_{\mathbf{\Omega}}(\bcX_{\mathrm{miss}}) = \mathscr{P}_{\mathbf{\Omega}}(\bcX_{\mathrm{miss}} \ast_{\boldsymbol{L}} \bcZ + \bcE),
\end{align}
where $\mathscr{P}_{\mathbf{\Omega}}: \R^{n_1 \times n_2 \times n_3} \to \R^{n_1 \times n_2 \times n_3}$ is the orthogonal projector onto the span of tensors vanishing outside $\mathbf{\Omega}$ so that the ($i,j,k$)-th component of $\mathscr{P}_{\mathbf{\Omega}}(\bcA)$ is equal to $\bcA_{i,j,k}$ if $(i,j,k) \in \mathbf{\Omega}$ and zero otherwise. We further assume that $\bcU_{\bcX} \ast_{\boldsymbol{L}} \bcS_{\bcX} \ast_{\boldsymbol{L}} \bcV_{\bcX}^H$ is the skinny t-SVD of $\bcX_{\mathrm{miss}}$ and $r_{\bcX} = \text{rank}_t(\bcX_{\mathrm{miss}})$, \eqref{eqn:ortlrrmissprob} can then be transformed into a simpler problem as follows:
\begin{align}    \label{eqn:ortlrrmissreducedprob}
& \min_{\bcZ',\bcE} \|\bcZ'\|_{\ast} + \lambda \|\mathscr{P}_{\mathbf{\Omega}}(\bcE)\|_{2,1}    \nonumber \\
&~ \text{s.t.} \quad \mathscr{P}_{\mathbf{\Omega}}(\bcX_{\mathrm{miss}}) = \mathscr{P}_{\mathbf{\Omega}}(\bcD \ast_{\boldsymbol{L}} \bcZ' + \bcE),
\end{align}
where $\bcD = \bcU_{\bcX} \ast_{\boldsymbol{L}} \bcS_{\bcX}$. We refer the reader to the supplementary material for deduction details. We dub this approach \emph{Outlier-Robust Tensor LRR by Entry-Wise Zero-Fill} (OR-TLRR-EWZF). Note that we can also robustify TLRR \cite{ZhouLFLY.PAMI2021} by replacing $\lambda \|\mathscr{P}_{\mathbf{\Omega}}(\bcE)\|_{2,1}$ in \eqref{eqn:ortlrrmissprob} with $\lambda \|\mathscr{P}_{\mathbf{\Omega}}(\bcE)\|_1$, and we call the resulting algorithm TLRR-EWZF in our experiments.

\section{EXPERIMENTS}
\label{sec:experiment}

In this section, we perform extensive experiments on both synthetic and real data to demonstrate the usefulness of the proposed algorithms. All experiments are conducted using Matlab R2021b on an AMD Ryzen 9 5950X 3.40GHz CPU with 64GB RAM.

\subsection{Synthetic Experiments}

\begin{table}[t]
\fontsize{5}{6}\selectfont
\centering
\caption{Exact recovery on random problems of varying sizes. The Discrete Fourier Transform (DFT) is used as the invertible linear transform $L$.}
\setlength\tabcolsep{0.1em}{\begin{tabular}{c||c|c|c|c}
\multicolumn{5}{@{}c}{$\rho = 0.2$, $r_{\ell} = 0.1 n_1$, $\lambda = 4 / (\sqrt{\log(n_{(1)})} \| \bcX \|)$}  \\
\hline
$n_1$ & $\text{rank}_t (\bcP_{\mathbf{\Theta}_0^{\perp}} (\widetilde{\bcX}))$ & $\frac{\| \bcP_{\bcV_0}^L - \bcP_{\widetilde{\bcU}} \|_F}{\| \bcP_{\bcV_0}^L \|_F}$ & $\frac{\| \bcP_{\mathbf{\Theta}_0^{\perp}} (\bcL_0) - \bcP_{\mathbf{\Theta}_0^{\perp}} (\widetilde{\bcX}) \|_F}{\| \bcP_{\mathbf{\Theta}_0^{\perp}} (\bcL_0) \|_F}$ & $\text{dist} (\mathbf{\Theta}_0, \widetilde{\mathbf{\Theta}})$  \\
\hline
\hline
60 & 30 & 2.5664e-15 & 2.7575e-09 & 0  \\
100 & 50 & 2.7290e-15 & 4.0480e-09 & 0  \\
150 & 75 & 3.0159e-15 & 1.0702e-08 & 0  \\
\hline
\end{tabular}}
\newline
\setlength\tabcolsep{0.1em}{\begin{tabular}{c||c|c|c|c}
\multicolumn{5}{@{}c}{$\rho = 0.4$, $r_{\ell} = 0.1 n_1$, $\lambda = 4 / (\sqrt{\log(n_{(1)})} \| \bcX \|)$}  \\
\hline
$n_1$ & $\text{rank}_t (\bcP_{\mathbf{\Theta}_0^{\perp}} (\widetilde{\bcX}))$ & $\frac{\| \bcP_{\bcV_0}^L - \bcP_{\widetilde{\bcU}} \|_F}{\| \bcP_{\bcV_0}^L \|_F}$ & $\frac{\| \bcP_{\mathbf{\Theta}_0^{\perp}} (\bcL_0) - \bcP_{\mathbf{\Theta}_0^{\perp}} (\widetilde{\bcX}) \|_F}{\| \bcP_{\mathbf{\Theta}_0^{\perp}} (\bcL_0) \|_F}$ & $\text{dist} (\mathbf{\Theta}_0, \widetilde{\mathbf{\Theta}})$  \\
\hline
\hline
60 & 30 & 2.5826e-15 & 1.5814e-27 & 0  \\
100 & 50 & 2.8084e-15 & 2.7471e-27 & 0  \\
150 & 75 & 3.0104e-15 & 5.5133e-27 & 0  \\
\hline
\end{tabular}}
\label{tab:syntheticDFT}
\vspace{-3mm}
\end{table}

\begin{table}[htbp]
\fontsize{5}{6}\selectfont
\centering
\caption{Exact recovery on random problems of varying sizes. The Discrete Cosine Transform (DCT) is used as the invertible linear transform $L$.}
\setlength{\tabcolsep}{0.1em}{\begin{tabular}{c||c|c|c|c}
\multicolumn{5}{@{}c}{$\rho = 0.2$, $r_{\ell} = 0.1 n_1$, $\lambda = 40 / (\sqrt{\log(n_{(1)})} \| \bcX \|)$}  \\
\hline
$n_1$ & $\text{rank}_t (\bcP_{\mathbf{\Theta}_0^{\perp}} (\widetilde{\bcX}))$ & $\frac{\| \bcP_{\bcV_0}^L - \bcP_{\widetilde{\bcU}} \|_F}{\| \bcP_{\bcV_0}^L \|_F}$ & $\frac{\| \bcP_{\mathbf{\Theta}_0^{\perp}} (\bcL_0) - \bcP_{\mathbf{\Theta}_0^{\perp}} (\widetilde{\bcX}) \|_F}{\| \bcP_{\mathbf{\Theta}_0^{\perp}} (\bcL_0) \|_F}$ & $\text{dist} (\mathbf{\Theta}_0, \widetilde{\mathbf{\Theta}})$  \\
\hline
\hline
60 & 30 & 2.4522e-15 & 1.1794e-05 & 0  \\
100 & 50 & 2.5648e-15 & 5.3367e-06 & 0  \\
150 & 75 & 2.9294e-15 & 5.1267e-06 & 0  \\
\hline
\end{tabular}}
\setlength{\tabcolsep}{0.1em}{\begin{tabular}{c||c|c|c|c}
\multicolumn{5}{@{}c}{$\rho = 0.4$, $r_{\ell} = 0.1 n_1$, $\lambda = 40 / (\sqrt{\log(n_{(1)})} \| \bcX \|)$}  \\
\hline
$n_1$ & $\text{rank}_t (\bcP_{\mathbf{\Theta}_0^{\perp}} (\widetilde{\bcX}))$ & $\frac{\| \bcP_{\bcV_0}^L - \bcP_{\widetilde{\bcU}} \|_F}{\| \bcP_{\bcV_0}^L \|_F}$ & $\frac{\| \bcP_{\mathbf{\Theta}_0^{\perp}} (\bcL_0) - \bcP_{\mathbf{\Theta}_0^{\perp}} (\widetilde{\bcX}) \|_F}{\| \bcP_{\mathbf{\Theta}_0^{\perp}} (\bcL_0) \|_F}$ & $\text{dist} (\mathbf{\Theta}_0, \widetilde{\mathbf{\Theta}})$  \\
\hline
\hline
60 & 30 & 2.5582e-15 & 2.8832e-13 & 0  \\
100 & 50 & 2.6576e-15 & 8.6553e-15 & 0  \\
150 & 75 & 2.9046e-15 & 1.1909e-15 & 0  \\
\hline
\end{tabular}}
\label{tab:syntheticDCT}
\vspace{-3mm}
\end{table}

\begin{table}[htbp]
\fontsize{5}{6}\selectfont
\centering
\caption{Exact recovery on random problems of varying sizes. The Random Orthogonal Matrix (ROM) is used as the invertible linear transform $L$.}
\setlength{\tabcolsep}{0.1em}{\begin{tabular}{c||c|c|c|c}
\multicolumn{5}{@{}c}{$\rho = 0.2$, $r_{\ell} = 0.1 n_1$, $\lambda = 40 / (\sqrt{\log(n_{(1)})} \| \bcX \|)$}  \\
\hline
$n_1$ & $\text{rank}_t (\bcP_{\mathbf{\Theta}_0^{\perp}} (\widetilde{\bcX}))$ & $\frac{\| \bcP_{\bcV_0}^L - \bcP_{\widetilde{\bcU}} \|_F}{\| \bcP_{\bcV_0}^L \|_F}$ & $\frac{\| \bcP_{\mathbf{\Theta}_0^{\perp}} (\bcL_0) - \bcP_{\mathbf{\Theta}_0^{\perp}} (\widetilde{\bcX}) \|_F}{\| \bcP_{\mathbf{\Theta}_0^{\perp}} (\bcL_0) \|_F}$ & $\text{dist} (\mathbf{\Theta}_0, \widetilde{\mathbf{\Theta}})$  \\
\hline
\hline
60 & 30 & 3.7444e-14 & 1.0834e-05 & 0  \\
100 & 50 & 4.2402e-14 & 6.0406e-06 & 0  \\
150 & 75 & 4.0723e-14 & 5.6253e-06 & 0  \\
\hline
\end{tabular}}
\setlength{\tabcolsep}{0.1em}{\begin{tabular}{c||c|c|c|c}
\multicolumn{5}{@{}c}{$\rho = 0.4$, $r_{\ell} = 0.1 n_1$, $\lambda = 40 / (\sqrt{\log(n_{(1)})} \| \bcX \|)$}  \\
\hline
$n_1$ & $\text{rank}_t (\bcP_{\mathbf{\Theta}_0^{\perp}} (\widetilde{\bcX}))$ & $\frac{\| \bcP_{\bcV_0}^L - \bcP_{\widetilde{\bcU}} \|_F}{\| \bcP_{\bcV_0}^L \|_F}$ & $\frac{\| \bcP_{\mathbf{\Theta}_0^{\perp}} (\bcL_0) - \bcP_{\mathbf{\Theta}_0^{\perp}} (\widetilde{\bcX}) \|_F}{\| \bcP_{\mathbf{\Theta}_0^{\perp}} (\bcL_0) \|_F}$ & $\text{dist} (\mathbf{\Theta}_0, \widetilde{\mathbf{\Theta}})$  \\
\hline
\hline
60 & 30 & 3.4499e-14 & 1.2335e-07 & 0  \\
100 & 50 & 3.9048e-14 & 3.4851e-15 & 0  \\
150 & 75 & 3.7525e-14 & 1.2012e-15 & 0  \\
\hline
\end{tabular}}
\label{tab:syntheticROM}
\vspace{-3mm}
\end{table}

Here, we verify the recovery guarantee in Theorem~\ref{thm:condortlrr} on randomly generated tensors. Three invertible linear transforms $L$ are adopted: (a) Discrete Fourier Transform (DFT); (b) Discrete Cosine Transform (DCT); (c) Random Orthogonal Matrix (ROM). We consider $c = 5$ tensor subspaces and generate the random tensor $\bcQ = [\bcQ_1, \dots, \bcQ_5] \in \R^{n_1 \times n_2 \times n_3}$ with $\bcQ_{\ell} = \bcA_{\ell} \ast_{\boldsymbol{L}} \bcB_{\ell}$, where the entries of $\bcA_{\ell} \in \R^{n_1 \times r_{\ell} \times n_3}$ and $\bcB_{\ell} \in \R^{r_{\ell} \times s_{\ell} \times n_3}$ are independently sampled from $\cN(0, 1/n_1)$ distribution. We simply set $s_{\ell} = n_1$, hence $n_2 = 5 n_1$. Then the lateral slices are independently sampled with probability $\rho$ as outliers, and the entries of these selected columns are randomly sampled from i.i.d. $\cN(0, \zeta/(n_1 n_3))$ distribution, where $\zeta$ is the averaged magnitude of the samples in $\bcQ$. In this way, the samples and outliers approximately have the same magnitude. The remaining entries of $\bcE_0$ are 0s. We denote the column support of $\bcE_0$ as $\mathbf{\Theta}_0$. Finally, the observed tensor $\bcX$ is generated by $\bcX = \bcL_0 + \bcE_0$, where $\bcL_0$ is obtained by setting the lateral slices of $\bcQ$ corresponding to nonzero columns of $\bcE_0$ to be $\mathbf{0}$s.

To verify that OR-TLRR can perform well for various tensor sizes, we set $n_3 = 100$, $n_1 = [ 60, 100, 150 ]$, $r_{\ell} = 0.1 n_1$, and $\rho = [ 0.2, 0.4 ]$. We run every experiment for 20 times and the results reported here correspond to the average of these 20 random trials. We set $\lambda = \alpha / (\sqrt{\log(n_{(1)})} \| \bcX \|)$ and empirically choose $\alpha$ from $[1, 2, 4, 8, 10]$ for DFT and from $[5, 10, 20, 30, 40, 50]$ for DCT and ROM. Note that $\widetilde{\bcX}$ denotes the recovered tensor, $\widetilde{\bcU}$ is the column space of $\widetilde{\bcZ}$, and the recovered outlier support set of $\widetilde{\bcE}$ is denoted as $\widetilde{\mathbf{\Theta}}$. Table~\ref{tab:syntheticDFT} - Table~\ref{tab:syntheticROM} respectively give the recovery results for the three different choices of the linear transforms $L$. The recovered tubal rank of $\bcP_{\mathbf{\Theta}_0^{\perp}} (\bcX)$ is exactly equal to $5 r_{\ell}$ and the relative errors $\| \bcP_{\bcV_0}^L - \bcP_{\widetilde{\bcU}} \|_F / \| \bcP_{\bcV_0}^L \|_F$ and $\| \bcP_{\mathbf{\Theta}_0^{\perp}} (\bcL_0) - \bcP_{\mathbf{\Theta}_0^{\perp}} (\widetilde{\bcX}) \|_F / \| \bcP_{\mathbf{\Theta}_0^{\perp}} (\bcL_0) \|_F$ are very small (less than $10^{-4}$). The Hamming distance $\text{dist} (\mathbf{\Theta}_0, \widetilde{\mathbf{\Theta}})$ between $\mathbf{\Theta}_0$ and $\widetilde{\mathbf{\Theta}}$ is always 0. Thus, the recovery guarantee claimed in Theorem~\ref{thm:condortlrr} has been fully verified by these numerical results. We also investigate the OR-TLRR-EWZF performance with missing data, where we set the percentage of missing entries $\delta$ to be 10\% and 20\%. We implement OR-TLRR-EWZF with $\lambda = 2 / (\sqrt{\log(n_{(1)})} \| \bcX_{\mathrm{miss}} \|)$ for DFT and $\lambda = 30 / (\sqrt{\log(n_{(1)})} \| \bcX_{\mathrm{miss}} \|)$ for DCT and ROM. The resulting Hamming distance $\text{dist} (\mathbf{\Theta}_0, \widetilde{\mathbf{\Theta}})$ is again always 0, which suggests that our algorithm can detect the outliers successfully in this setting.

\begin{table}[t]
\centering
\scriptsize
\caption{Description of datasets.}
\begin{tabular}{ccccc}
\hline
Dataset & \#Class & \#Total image & Size & Type  \\
\hline
\hline
ORL & 40 & 400 & 32 $\times$ 32 & Face  \\
COIL20 & 20 & 1440 & 32 $\times$ 32 & Object  \\
Umist & 20 & 575 & 56 $\times$ 46 & Face  \\
FRDUE & 152 & 3040 & 25 $\times$ 22 & Face  \\
USPS & 10 & 9298 & 16 $\times$ 16 & Handwritten  \\
\hline
\end{tabular}
\label{tab:datadesc}
\end{table}

\begin{figure}[t]
\centering
\subfloat[ORL]{\includegraphics[height=0.3in,width=0.9in]{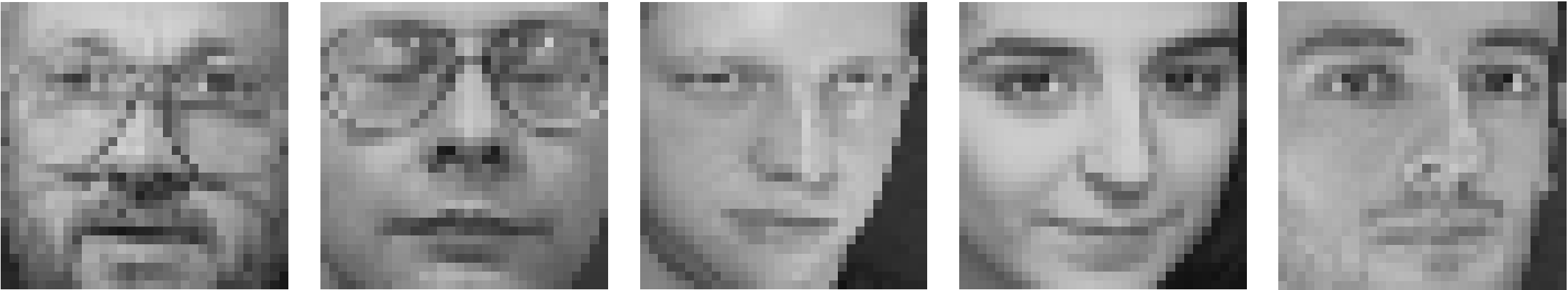}}
\quad
\subfloat[COIL20]{\includegraphics[height=0.3in,width=0.9in]{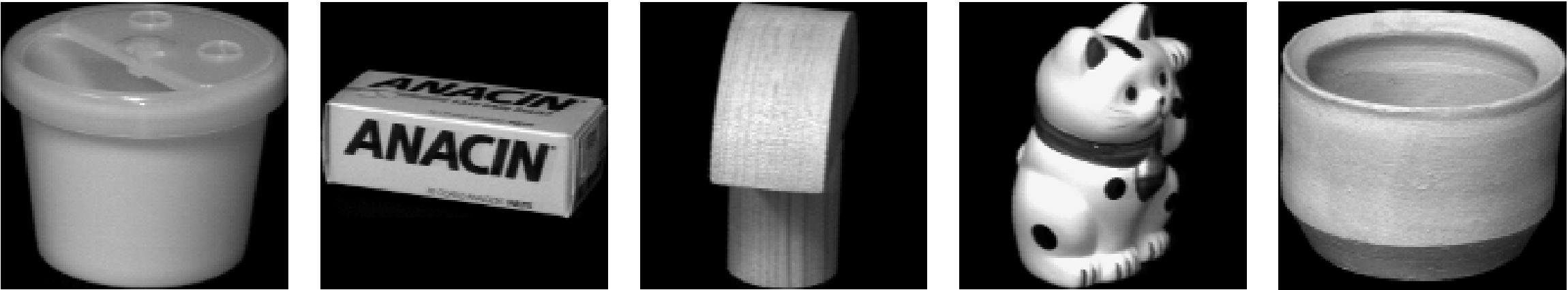}}
\quad
\subfloat[Umist]{\includegraphics[height=0.3in,width=0.9in]{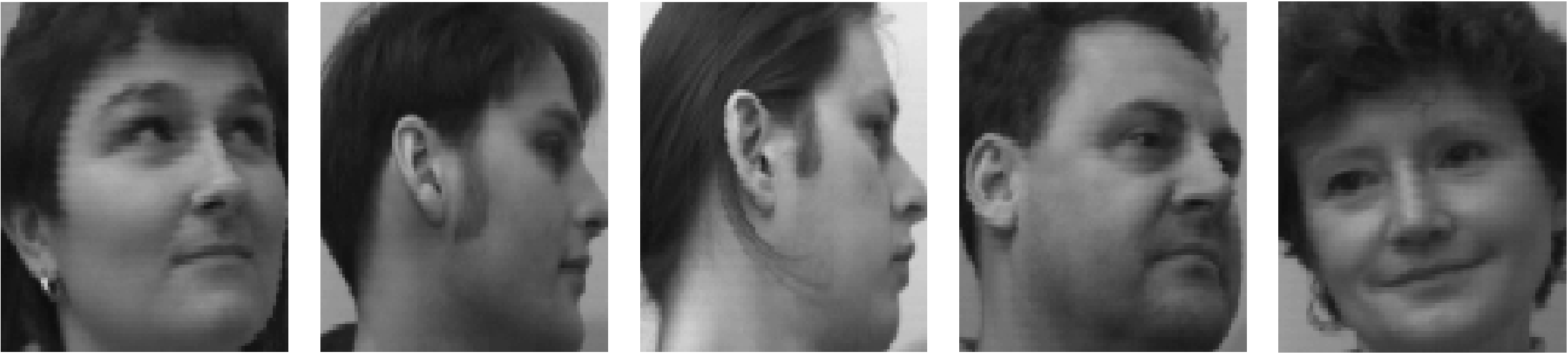}}
\\
\subfloat[FRDUE]{\includegraphics[height=0.3in,width=0.9in]{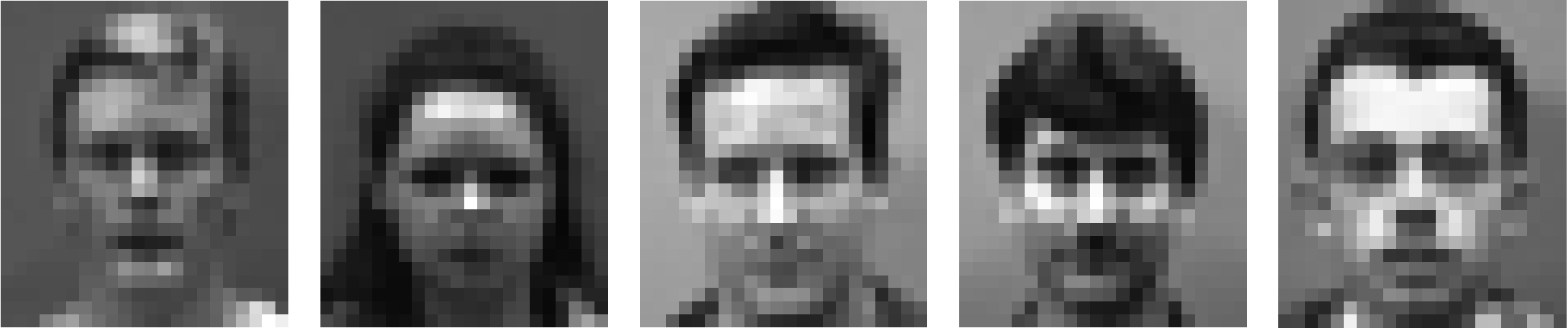}}
\quad
\subfloat[USPS]{\includegraphics[height=0.3in,width=0.9in]{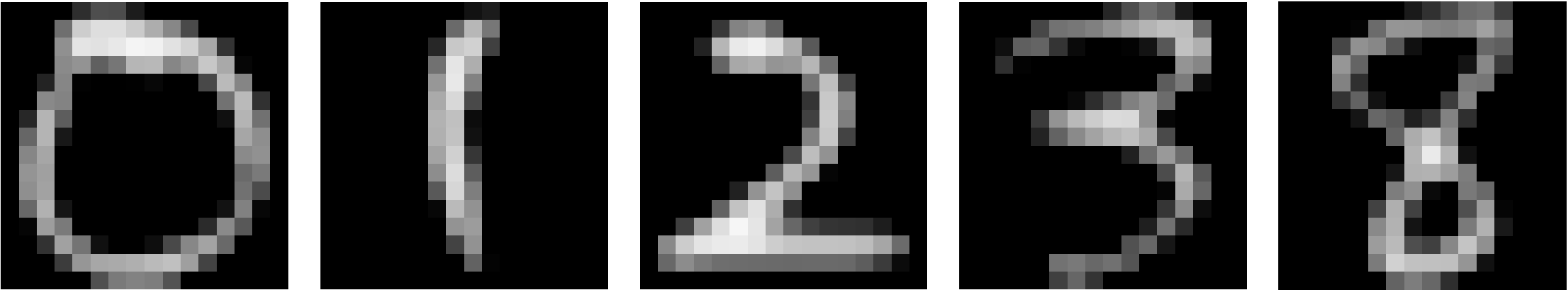}}
\quad
\subfloat[Mirflickr-25k]{\includegraphics[height=0.3in,width=0.9in]{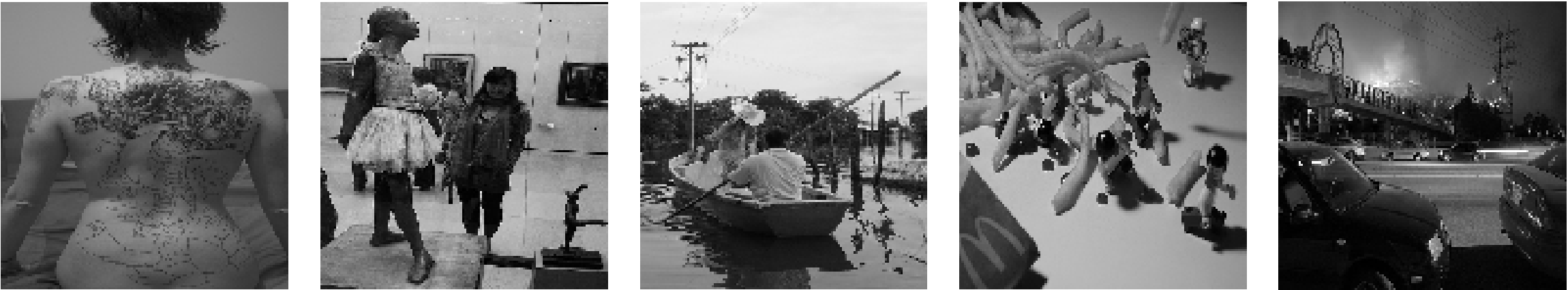}}
\caption{Sample images from the datasets.}
\label{fig:imgexample}
\end{figure}

\subsection{Real-World Applications}

In this subsection, the proposed algorithms are based on two different linear transforms for all experiments, i.e., DFT and DCT. The corresponding methods of OR-TLRR and OR-TLRR-EWZF are called OR-TLRR-DFT/OR-TLRR-DCT and OR-TLRR-EWZF-DFT/OR-TLRR-EWZF-DCT, respectively for short. We fix $\lambda = 1 / (\sqrt{\log(n_{(1)})} \| \bcX \|)$ for OR-TLRR-DFT and $\lambda = 1 / (\sqrt{\log(n_{(1)})} \| \bcX_{\mathrm{miss}} \|)$ for OR-TLRR-EWZF-DFT, and set $\lambda = 100 / (\sqrt{\log(n_{(1)})} \| \bcX \|)$ for OR-TLRR-DCT and $\lambda = 100 / (\sqrt{\log(n_{(1)})} \| \bcX_{\mathrm{miss}} \|)$ for OR-TLRR-EWZF-DCT. Table~\ref{tab:datadesc} gives the brief introduction of five testing databases, including ORL\footnote{\scriptsize \url{http://cam-orl.co.uk/facedatabase.html}}, COIL20\footnote{\scriptsize \url{https://www.cs.columbia.edu/CAVE/software/softlib/coil-20.php}}, Umist\footnote{\scriptsize \url{https://www.visioneng.org.uk/datasets/}}, FRDUE\footnote{\scriptsize \url{https://bamdevmishra.in/codes/tensorcompletion/}}, and USPS \cite{Hull.PAMI1994}.

\begin{table*}[t]
\footnotesize
\centering
\caption{AUC of outlier detection and clustering results (ACC, NMI and PUR) on ORL-MIRFLICKR-25k and COIL20-MIRFLICKR-25k datasets for complete data experiments.}
\setlength\tabcolsep{0.1em}{\begin{tabular}{c||c|c|c|c||c|c|c|c||c|c|c|c}
\hline
Datasets (\# Outlier) & \multicolumn{4}{c||}{ORL-MIR-25k (100)} & \multicolumn{4}{c||}{COIL20-MIR-25k (100)} & \multicolumn{4}{c}{COIL20-MIR-25k (200)}  \\
\hline
Methods & AUC & ACC & NMI & PUR & AUC & ACC & NMI & PUR & AUC & ACC & NMI & PUR  \\
\hline
R-PCA & \textbf{0.9734} & 0.4972 & 0.7115 & 0.5327 & \textbf{0.8732} & 0.6171 & 0.7301 & 0.6606 & 0.8444 & 0.6253 & 0.7215 & 0.6468  \\
OR-PCA & 0.9733 & 0.4972 & 0.7073 & 0.5328 & 0.8636 & 0.6228 & 0.7311 & 0.6663 & 0.8368 & 0.6130 & 0.7143 & 0.6510  \\
LRR & 0.7448 & 0.5058 & 0.6843 & 0.5360 & 0.8630 & 0.5002 & 0.6358 & 0.5956 & \textbf{0.8836} & 0.5252 & 0.6489 & 0.6114  \\
TRPCA-DFT & 0.9548 & 0.5549 & 0.7271 & 0.5806 & 0.7687 & 0.5679 & 0.7056 & 0.6441 & 0.7381 & 0.6087 & 0.7171 & 0.6581  \\
TRPCA-DCT & 0.9586 & 0.5522 & 0.7237 & 0.5779 & 0.7740 & 0.5344 & 0.6912 & 0.6308 & 0.7430 & 0.5740 & 0.7011 & 0.6377  \\
OR-TPCA & 0.7297 & 0.3851 & 0.5955 & 0.4098 & 0.0267 & 0.5675 & 0.6610 & 0.6027 & 0.0100 & 0.5750 & 0.6671 & 0.6085  \\
TLRR-DFT & 0.9550 & 0.5094 & 0.6812 & 0.5373 & 0.8219 & 0.5635 & 0.6479 & 0.6043 & 0.8049 & 0.5682 & 0.6454 & 0.6050  \\
TLRR-DCT & 0.9573 & 0.5488 & 0.7193 & 0.5753 & 0.7871 & 0.6222 & 0.7127 & 0.6647 & 0.7688 & 0.6624 & 0.7356 & 0.6947  \\
\hline
OR-TLRR-DFT & 0.9401 & \textbf{0.5915} & \textbf{0.7525} & \textbf{0.6184} & 0.7910 & \textbf{0.6700} & \textbf{0.7488} & \textbf{0.6967} & 0.7666 & \textbf{0.6725} & \textbf{0.7471} & \textbf{0.6966}  \\
OR-TLRR-DCT & 0.9460 & 0.5691 & 0.7444 & 0.5969 & 0.7643 & 0.5552 & 0.6855 & 0.6245 & 0.7384 & 0.6087 & 0.7165 & 0.6583  \\
\hline
\end{tabular}}
\label{tab:orl-coil20-comp}
\vspace{-3mm}
\end{table*}

\begin{table*}[t]
\footnotesize
\centering
\caption{AUC of outlier detection and clustering results (ACC, NMI and PUR) on ORL-MIRFLICKR-25k dataset for missing data experiments.}
\setlength{\tabcolsep}{0.1em}{\begin{tabular}{c||c|c|c|c||c|c|c|c}
\hline
\multirow{2}{*}{Methods} & \multicolumn{4}{c||}{$\delta = 10\%$} & \multicolumn{4}{c}{$\delta = 20\%$}  \\
\cline{2-9}
 & AUC & ACC & NMI & PUR & AUC & ACC & NMI & PUR  \\
\hline
TNN+TRPCA-DFT & 0.9552 & 0.5551 & 0.7260 & 0.5806 & 0.9535 & 0.5493 & 0.7224 & 0.5753  \\
TNN+TRPCA-DCT & \textbf{0.9581} & 0.5513 & 0.7241 & 0.5774 & \textbf{0.9583} & 0.5529 & 0.7244 & 0.5786  \\
TNN+OR-TPCA-DFT & 0.6970 & 0.3772 & 0.5871 & 0.4025 & 0.6752 & 0.3735 & 0.5820 & 0.3976  \\
TNN+TLRR-DFT & 0.9537 & 0.5077 & 0.6766 & 0.5336 & 0.9570 & 0.4847 & 0.6587 & 0.5105  \\
TNN+TLRR-DCT & 0.9580 & 0.5420 & 0.7136 & 0.5690 & 0.9578 & 0.5504 & 0.7175 & 0.5773  \\
\hline
TNN+OR-TLRR-DFT & 0.9384 & \textbf{0.5853} & \textbf{0.7475} & \textbf{0.6108} & 0.9376 & \textbf{0.5843} & \textbf{0.7481} & \textbf{0.6132}  \\
TNN+OR-TLRR-DCT & 0.9445 & 0.5722 & 0.7441 & 0.6009 & 0.9450 & 0.5674 & 0.7409 & 0.5949  \\
TLRR-EWZF-DFT & 0.9549 & 0.5088 & 0.6760 & 0.5341 & 0.9544 & 0.4892 & 0.6590 & 0.5152  \\
TLRR-EWZF-DCT & 0.9568 & 0.5398 & 0.7128 & 0.5687 & 0.9568 & 0.5398 & 0.7128 & 0.5674  \\
OR-TLRR-EWZF-DFT & 0.9391 & 0.5813 & 0.7449 & 0.6089 & 0.9407 & 0.5802 & 0.7417 & 0.6080  \\
OR-TLRR-EWZF-DCT & 0.9456 & 0.5608 & 0.7382 & 0.5899 & 0.9468 & 0.5631 & 0.7331 & 0.5898  \\
\hline
\end{tabular}}
\label{tab:orl-miss}
\vspace{-3mm}
\end{table*}

\subsubsection{Application to Outlier Detection}

Here we evaluate the effectiveness of OR-TLRR on outlier detection and clustering tasks. For the complete data experiments, we compare OR-TLRR-DFT/OR-TLRR-DCT with other state-of-the-art low-rank matrix/tensor factorization methods including R-PCA \cite{CandesLMW.JACM2011}, OR-PCA \cite{ZhangLZC.AAAI2015}, LRR \cite{LiuLYSYM.PAMI2013}, TRPCA-DFT \cite{LuFCLLY.PAMI2020}, TRPCA-DCT \cite{Lu.ICCV2021}, OR-TPCA \cite{ZhouF.CVPR2017}, TLRR-DFT \cite{ZhouLFLY.PAMI2021} and TLRR-DCT \cite{Yang.etal.TNNLS2022}. For TRPCA/OR-TPCA, let the estimated clean ``inlier'' data be $\bcL$. We first compute a matrix $\bL_0 = \sum_{k=1}^{n_3} \bL^{(k)}$ and perform SVD on $\bL_0 = \bU_0 \bS_0 \bV_0^T$. Then we apply spectral clustering on $\bV_0 \bV_0^T$ to obtain the result. For the case of data having missing entries, we compare OR-TLRR-EWZF/TLRR-EWZF with a baseline that first completes the unknown entries of $\bcX$ using tensor nuclear norm based tensor completion model \cite{LuFLY.IJCAI2018,LuPW.CVPR2019}, and then do clustering using TRPCA \cite{LuFCLLY.PAMI2020,Lu.ICCV2021}, OR-TPCA \cite{ZhouF.CVPR2017}, TLRR \cite{ZhouLFLY.PAMI2021,Yang.etal.TNNLS2022}, and OR-TLRR with the same linear transform that is used in the tensor completion step. The percentage of missing entries $\delta$ is again set to be 10\% and 20\%. Unless otherwise stated, all regularization parameters involved in the competing methods were either tuned or selected as suggested in the respective papers.

\begin{table*}[t]
\fontsize{5}{8}\selectfont
\centering
\caption{Clustering results (ACC, NMI and PUR) on Umist and FRDUE datasets for complete data experiments.}
\setlength{\tabcolsep}{0.1em}{\begin{tabular}{c|c||c|c|c|c|c|c|c|c|c|c|c|c}
\hline
Dataset & Metric & R-PCA & OR-PCA & LRR & SSC & SSC-OMP & TRPCA-DFT & TRPCA-DCT & OR-TPCA & TLRR-DFT & TLRR-DCT & OR-TLRR-DFT & OR-TLRR-DCT  \\
\hline
\hline
\multirow{3}{*}{Umist} & ACC & 0.4299 & 0.4155 & 0.3515 & 0.4972 & 0.4030 & 0.5774 & 0.5741 & 0.3454 & 0.5452 & 0.5007 & \textbf{0.6117} & 0.5125  \\
 & NMI & 0.5831 & 0.5786 & 0.4718 & 0.6818 & 0.5633 & 0.7282 & 0.7331 & 0.3930 & 0.6648 & 0.6573 & \textbf{0.7536} & 0.6756  \\
 & PUR & 0.4758 & 0.4661 & 0.4167 & 0.5897 & 0.5078 & 0.6374 & 0.6478 & 0.3871 & 0.6226 & 0.5743 & \textbf{0.6790} & 0.5652  \\
\hline
\multirow{3}{*}{FRDUE} & ACC & 0.6784 & 0.6760 & 0.7765 & 0.8200 & 0.4913 & 0.7058 & 0.7011 & 0.6807 & 0.8192 & 0.8109 & \textbf{0.8525} & 0.8231  \\
 & NMI & 0.8624 & 0.8619 & 0.9181 & 0.9399 & 0.7378 & 0.8804 & 0.8785 & 0.8757 & 0.9365 & 0.9321 & \textbf{0.9549} & 0.9403  \\
 & PUR & 0.7162 & 0.7131 & 0.8151 & 0.8649 & 0.5319 & 0.7435 & 0.7384 & 0.7243 & 0.8548 & 0.8403 & \textbf{0.8832} & 0.8578  \\
\hline
\end{tabular}}
\label{tab:umist-frdue-comp}
\vspace{-3mm}
\end{table*}

\begin{table*}[htbp]
\footnotesize
\centering
\caption{Clustering results (ACC, NMI and PUR) on Umist and FRDUE datasets for missing data experiments.}
\setlength{\tabcolsep}{0.1em}{\begin{tabular}{c||c|c|c||c|c|c||c|c|c||c|c|c}
\hline
Datasets & \multicolumn{6}{c||}{Umist} & \multicolumn{6}{c}{FRDUE}  \\
\hline
\multirow{2}{*}{Methods} & \multicolumn{3}{c||}{$\delta = 10\%$} & \multicolumn{3}{c||}{$\delta = 20\%$} & \multicolumn{3}{c||}{$\delta = 10\%$} & \multicolumn{3}{c}{$\delta = 20\%$}  \\
\cline{2-13}
 & ACC & NMI & PUR & ACC & NMI & PUR & ACC & NMI & PUR & ACC & NMI & PUR  \\
\hline
TNN+TRPCA-DFT & 0.5715 & 0.7261 & 0.6319 & 0.5737 & 0.7267 & 0.6333 & 0.7134 & 0.8828 & 0.7496 & 0.7274 & 0.8883 & 0.7626  \\
TNN+TRPCA-DCT & 0.5730 & 0.7297 & 0.6379 & 0.5743 & 0.7308 & 0.6390 & 0.7153 & 0.8825 & 0.7515 & 0.7237 & 0.8865 & 0.7590  \\
TNN+OR-TPCA-DFT & 0.3524 & 0.4037 & 0.4014 & 0.3495 & 0.3995 & 0.3951 & 0.7285 & 0.8941 & 0.7653 & 0.7467 & 0.9006 & 0.7817  \\
TNN+TLRR-DFT & 0.5298 & 0.6512 & 0.6024 & 0.5261 & 0.6529 & 0.6064 & 0.8389 & 0.9460 & 0.8701 & 0.8476 & 0.9504 & 0.8775  \\
TNN+TLRR-DCT & 0.5020 & 0.6557 & 0.5751 & 0.4982 & 0.6507 & 0.5719 & 0.8195 & 0.9388 & 0.8525 & 0.8274 & 0.9439 & 0.8597  \\
\hline
TNN+OR-TLRR-DFT & \textbf{0.6017} & \textbf{0.7477} & \textbf{0.6707} & \textbf{0.6015} & \textbf{0.7476} & \textbf{0.6707} & 0.8577 & 0.9577 & 0.8873 & 0.8682 & 0.9625 & 0.8967  \\
TNN+OR-TLRR-DCT & 0.5048 & 0.6749 & 0.5611 & 0.5015 & 0.6722 & 0.5573 & 0.8339 & 0.9458 & 0.8650 & 0.8377 & 0.9487 & 0.8696  \\
TLRR-EWZF-DFT & 0.5704 & 0.7249 & 0.6461 & 0.5512 & 0.7112 & 0.6289 & 0.8616 & 0.9568 & 0.8928 & 0.8726 & 0.9624 & 0.9012  \\
TLRR-EWZF-DCT & 0.5071 & 0.6731 & 0.5795 & 0.5030 & 0.6776 & 0.5713 & 0.8520 & 0.9570 & 0.8831 & 0.8546 & 0.9573 & 0.8842  \\
OR-TLRR-EWZF-DFT & 0.5705 & 0.7215 & 0.6356 & 0.5554 & 0.7034 & 0.6203 & \textbf{0.8641} & \textbf{0.9626} & \textbf{0.8959} & \textbf{0.8893} & \textbf{0.9752} & \textbf{0.9160}  \\
OR-TLRR-EWZF-DCT & 0.5079 & 0.6740 & 0.5639 & 0.5145 & 0.6787 & 0.5716 & 0.8496 & 0.9566 & 0.8820 & 0.8555 & 0.9579 & 0.8858  \\
\hline
\end{tabular}}
\label{tab:umist-frdue-miss}
\end{table*}

We create two datasets by combining ORL/COIL20 with the MIRFLICKR-25k dataset (containing 25,000 images) \cite{HuiskesL.ACM2008}, which we call ORL-MIRFLICKR-25k and COIL20-MIRFLICKR-25k, respectively. For ORL-MIRFLICKR-25k, we randomly select 100 images from MIRFLICKR-25k as outliers and the random selection is repeated 20 times. Fig.~\ref{fig:imgexample} (a) and Fig.~\ref{fig:imgexample} (f) show some examples of ORL and MIRFLICKR-25k, respectively. For COIL20-MIRFLICKR-25k, the outliers consist of 100 or 200 images selected from MIRFLICKR-25k. Theorem~\ref{thm:condortlrr} implies the optimal solution $\bcE$ can help detect outliers in the data. Since all the methods have the sparse component $\bcE$, in this paper, we use $k$-means to cluster all the $\| \bcE(:,j,:) \|_F^2$'s into two classes (outliers vs. non-outliers) for outlier detection. The performance of outlier detection is evaluated by AUC. While investigating the segmentation performance, we first remove the outliers detected by $k$-means (some normal samples may be removed, and these samples are given wrong labels) and apply Ncut \cite{ShiM.PAMI2000} on the remaining samples (possibly include undetected outliers) for clustering. We evaluate the clustering performance of the normal samples using three metrics, including ACC, NMI \cite{VinhEB.JMLR2010} and PUR \cite{ManningRS.2010}.

Table~\ref{tab:orl-coil20-comp} summarizes the results for the complete data experiments. From these results, we can find that matrix-based algorithms seem to work better than tensor-based ones in terms of outlier detection. OR-TLRR-DFT always achieves the best clustering performance. On the widely used ACC metric, OR-TLRR-DFT respectively improves by 3.9\%, 7.6\% and 1.5\% over the runner-up on the three testing cases (left-right). Such improvement can be attributed to the reason that (1) it takes advantage of the multi-dimensional structure of tensor data; (2) it uses the $\ell_{2,1}$ norm which can better depict outliers than the $\ell_1$ norm. From Table~\ref{tab:orl-miss} and Table~\ref{tab:coil20-miss} (in the supplementary material), we observe that TNN+OR-TLRR-DFT outperforms all other methods in terms of clustering on the ORL-MIRFLICKR-25k dataset, and it has very marginal improvement over the second best, OR-TLRR-EWZF-DFT. As for the COIL20-MIRFLICKR-25k dataset, OR-TLRR-EWZF-DCT consistently achieves the best clustering performance in ACC, NMI and PUR. This may be explained because the tensor completion step fails as the data tensor is not low-rank.

\subsubsection{Application to Image Clustering}

Here, we conduct experiments on Umist and FRDUE datasets for face clustering. Due to shadows and facial expressions on face images displayed in Fig.~\ref{fig:imgexample} (c) and Fig.~\ref{fig:imgexample} (d), it can be argued that the $\ell_{2,1}$ norm can characterize the noise better than the $\ell_1$ norm. Table~\ref{tab:umist-frdue-comp} presents the clustering results for all the methods. As can be seen from this table, OR-TLRR-DFT shows noticeable improvement in all the metrics compared to other algorithms. The accuracy outperforms the second best by 5.9\% and 3.6\% on Umist and FRDUE datasets, respectively. Similar to previous results, we observe from Table~\ref{tab:umist-frdue-miss} that TNN+OR-TLRR-DFT performs the best on Umist dataset and OR-TLRR-EWZF-DFT ourperforms other methods on FRDUE dataset. Understanding the root cause behind this phenomenon is the subject of future research. We also examine the robustness of our methods for dealing with sample-specific corruptions, and we include these results in the supplementary material.

\section{CONCLUSION}
\label{sec:conclude}

In this work, we proposed an outlier-robust tensor low-rank representation method for tensor data clustering in the presence of outliers. We proved that OR-TLRR can recover the row space of the authentic data and identify the outliers with high probability. We further extended the proposed method to the missing data scenario. Extensive numerical results on synthetic and real data demonstrated the effectiveness of our algorithms. In this work, we directly use the raw tensor data as the dictionary. While the performance and the corresponding theoretical guarantee of other possible options for the dictionary, e.g., apply TRPCA on the raw data and use the estimation of the clean data as the dictionary, will be investigated in future work. Our future work also includes learning the optimal linear transform for different types of data.

\bibliography{TW_aistats_bib}
\bibliographystyle{apalike}

\onecolumn
\aistatstitle{Supplementary Material: \\ Robust Data Clustering with Outliers via Transformed Tensor Low-Rank Representation}

\setcounter{section}{0}
\renewcommand{\thesection}{\Alph{section}}
\renewcommand{\thesubsection}{\thesection.\arabic{subsection}}

\section{EXTENSION OF PRELIMINARIES}
\label{sec:preliminproof}

\subsection{Notations}

The operator norm of an operator on tensor is $\| \mathscr{L} \| = \sup_{\|\bcA\|_F = 1} \| \mathscr{L}(\bcA) \|_F$. The inner product between two matrices $\bA$ and $\bB$ in $\C^{n_1 \times n_2}$ is defined as $\langle \bA, \bB \rangle = \tr(\bA^H \bB)$, whereas the inner product between two tensors $\bcA$ and $\bcB$ in $\C^{n_1 \times n_2 \times n_3}$ is defined as $\langle \bcA, \bcB \rangle = \sum_{i=1}^{n_3} \langle \bA^{(i)}, \bB^{(i)} \rangle$. The dual norm of tensor $\ell_{2,1}$ norm is the tensor $\ell_{2,\infty}$ norm, which is defined as $\|\bcA\|_{2,\infty} = \max_j \|\bcA(:,j,:)\|_F$. The infinity norm of $\bcA$ is $\|\bcA\|_{\infty} = \max_{i,j,k} |\bcA_{i,j,k}|$.

\begin{definition}[Tensor pseudo-inverse \cite{ZhouLFLY.PAMI2021}]
For an arbitrary tensor $\bcA \in \C^{n_1 \times n_2 \times n_3}$, its pseudo-inverse under $L$ in \eqref{eqn:mode3prod} is defined as a tensor $\bcA^{\dagger} \in \C^{n_2 \times n_1 \times n_3}$ which satisfies ($i$) $\bcA \ast_{\boldsymbol{L}} \bcA^{\dagger} \ast_{\boldsymbol{L}} \bcA = \bcA$, ($ii$) $\bcA^{\dagger} \ast_{\boldsymbol{L}} \bcA \ast_{\boldsymbol{L}} \bcA^{\dagger} = \bcA^{\dagger}$, ($iii$) $(\bcA \ast_{\boldsymbol{L}} \bcA^{\dagger})^H = \bcA \ast_{\boldsymbol{L}} \bcA^{\dagger}$, and ($iv$) $(\bcA^{\dagger} \ast_{\boldsymbol{L}} \bcA)^H = \bcA^{\dagger} \ast_{\boldsymbol{L}} \bcA$.
\end{definition}
\begin{definition}[Standard tensor basis \cite{LuPW.CVPR2019}]    \label{def:stdtensorbasis}
The tensor \textbf{column basis} with respect to the transform $L$, denoted as $\mathring{\boldsymbol{\ce}}_i$, is a tensor of size $n_1 \times 1 \times n_3$ with the entries of the $(i,1)$-th tube of $L(\mathring{\boldsymbol{\ce}}_i)$ equaling 1 and the rest equaling 0. Similarly, the \textbf{row basis} $\mathring{\boldsymbol{\ce}}_j^H$ is of size $1 \times n_2 \times n_3$ with the entries of the $(1,j)$-th tube of $L(\mathring{\boldsymbol{\ce}}_j^H)$ equaling to 1 and the rest equaling to 0. The \textbf{tube basis} $\dot{\boldsymbol{\ce}}_k$ is a tensor of size $1 \times 1 \times n_3$ with the $(1,1,k)$-th entry of $L(\dot{\boldsymbol{\ce}}_k)$ equaling 1 and the rest equaling 0.
\end{definition}

Denote $\bar{\boldsymbol{\ce}}_{ijk}$ as a unit tensor with only the $(i,j,k)$-th entry equaling 1 and others equaling 0. Based on Definition~\ref{def:stdtensorbasis}, $\bar{\boldsymbol{\ce}}_{ijk}$ can be expressed as $\bar{\boldsymbol{\ce}}_{ijk} = L(\mathring{\boldsymbol{\ce}}_i \ast_{\boldsymbol{L}} \dot{\boldsymbol{\ce}}_k \ast_{\boldsymbol{L}} \mathring{\boldsymbol{\ce}}_j^H)$. Then for any tensor $\bcA \in \R^{n_1 \times n_2 \times n_3}$, we have $\bcA_{i,j,k} = \langle \bcA, \bar{\boldsymbol{\ce}}_{ijk} \rangle$ and $\bcA = \sum_{i,j,k} \langle \bcA, \bar{\boldsymbol{\ce}}_{ijk} \rangle \bar{\boldsymbol{\ce}}_{ijk}$.

Next, we define some commonly used operators in this document. Assume that $\bcU_{\bcX} \ast_{\boldsymbol{L}} \bcS_{\bcX} \ast_{\boldsymbol{L}} \bcV_{\bcX}^H$, $\bcU_0 \ast_{\boldsymbol{L}} \bcS_0 \ast_{\boldsymbol{L}} \bcV_0^H$, and $\bcU \ast_{\boldsymbol{L}} \bcS \ast_{\boldsymbol{L}} \bcV^H$ are the skinny t-SVDs of $\bcX$, $\bcL_0$ and $\bcZ$, respectively. The projection onto the column space $\bcU$ is given by $\bcP_{\bcU} (\bcA) = \bcU \ast_{\boldsymbol{L}} \bcU^H \ast_{\boldsymbol{L}} \bcA$, and similarly for the row space $\bcP_{\bcV} (\bcA) = \bcA \ast_{\boldsymbol{L}} \bcV \ast_{\boldsymbol{L}} \bcV^H$. Sometimes, we need to apply $\bcP_{\bcV}$ on the left side of a tensor. This special operator is denoted by $\bcP_{\bcV}^L (\bcA) = \bcV \ast_{\boldsymbol{L}} \bcV^H \ast_{\boldsymbol{L}} \bcA$. The tensor $\bcP_{\mathbf{\Theta}} (\bcA)$ is obtained from $\bcA$ by setting tensor column $\bcA(:,j,:)$ to zero for all $j \notin \mathbf{\Theta}$. The projection onto the space spanned by $\bcU$ and $\bcV$ is given by $\bcP_{\bcT} (\bcA) = \bcP_{\bcU} (\bcA) + \bcP_{\bcV} (\bcA) - \bcP_{\bcU} \bcP_{\bcV} (\bcA)$, where $\bcP_{\bcU} \bcP_{\bcV} (\bcA) = \bcU \ast_{\boldsymbol{L}} \bcU^H \ast_{\boldsymbol{L}} \bcA \ast_{\boldsymbol{L}} \bcV \ast_{\boldsymbol{L}} \bcV^H$. The complementary operators, $\bcP_{\bcU^{\perp}}$, $\bcP_{\bcV^{\perp}}$, $\bcP_{\bcV^{\perp}}^L$ and $\bcP_{\bcT^{\perp}}$ are defined as usual. Finally, we use $\mathbf{\Theta}^c$ to denote the complement of $\mathbf{\Theta}$ and $\bcP_{\mathbf{\Theta}^{\perp}}$ is the projection onto $\mathbf{\Theta}^c$.

\subsection{Preliminaries}

We consider the invertible linear transform $L : \R^{n_1 \times n_2 \times n_3} \to \C^{n_1 \times n_2 \times n_3}$ defined as follows:
\begin{align}
\xoverline{\bcA} = L(\bcA) = \bcA \times_3 \boldsymbol{L},
\end{align}
where $\times_3$ denotes the mode-3 tensor-matrix product, and $\boldsymbol{L} \in \C^{n_3 \times n_3}$ is an arbitrary invertible matrix. In this document, we always use $\boldsymbol{L}$ satisfying
\begin{align}    \label{eqn:lconstraintcopy}
\boldsymbol{L} \boldsymbol{L}^H = \boldsymbol{L}^H \boldsymbol{L} = \tau \bI_{n_3},
\end{align}
where $\tau > 0$. Using \eqref{eqn:lconstraintcopy}, we have the following properties:
\begin{align*}
\|\bcA\|_F = \frac{1}{\sqrt{\tau}} \|\xoverline{\bcA}\|_F = \frac{1}{\sqrt{\tau}} \|\widebar{\bA}\|_F \quad \text{and} \quad \langle \bcA, \bcB \rangle = \frac{1}{\tau} \langle \xoverline{\bcA}, \xoverline{\bcB} \rangle = \frac{1}{\tau} \langle \widebar{\bA}, \widebar{\bB} \rangle.
\end{align*}

\section{OPTIMIZATION DETAILS OF OR-TLRR}
\label{sec:compalgodetail}

In this section, we elaborate the details of our optimization strategy to solve problem \eqref{eqn:ortlrrreducedprob} in the manuscript. To solve \eqref{eqn:ortlrrreducedprob} in an efficient manner, we resort to ``variable splitting" of $\bcZ'$, which transforms \eqref{eqn:ortlrrreducedprob} into the following:
\begin{align}    \label{eqn:convertform}
\min_{\bcZ',\bcJ,\bcE} \|\bcZ'\|_{\ast} + \lambda \|\bcE\|_{2,1} \quad \text{s.t.} \quad \bcZ' = \bcJ, \bcX = \bcD \ast_{\boldsymbol{L}} \bcJ + \bcE.
\end{align}
The augmented Lagrangian function of \eqref{eqn:convertform} is
\begin{align}    \label{eqn:augform}
&\cL_1(\bcZ',\bcJ,\bcE,\bcY_1,\bcY_2,\beta) = \|\bcZ'\|_{\ast} + \lambda \|\bcE\|_{2,1} + \langle \bcY_1, \bcZ' - \bcJ \rangle + \langle \bcY_2, \bcX - \bcD \ast_{\boldsymbol{L}} \bcJ - \bcE \rangle    \nonumber \\
&\qquad + \frac{\beta}{2} \big( \|\bcZ' - \bcJ\|_F^2 + \|\bcX - \bcD \ast_{\boldsymbol{L}} \bcJ - \bcE\|_F^2 \big),
\end{align}
where the tensors $\bcY_1$ and $\bcY_2$ comprise Lagrange multipliers and $\beta > 0$ is a penalty parameter. The optimization of \eqref{eqn:augform} can be done iteratively by minimizing $\cL_1$ with respect to $\bcZ'$, $\bcJ$ and $\bcE$ over one tensor at a time while keeping the others fixed.

\textbf{Updating $\bcZ'$:} When other variables are fixed, the problem of updating $\bcZ'$ can be written as
\begin{align*}
\bcZ'^{(t+1)} = \argmin_{\bcZ'} \|\bcZ'\|_{\ast} + \frac{\beta^{(t)}}{2} \|\bcB_1^{(t+1)} - \bcZ'\|_F^2,
\end{align*}
where $\bcB_1^{(t+1)} = \bcJ^{(t)} - \frac{\bcY_1^{(t)}}{\beta^{(t)}}$. We can optimize its equivalent problem:
\begin{align*}
\widebar{\bZ}'^{(t+1)} = \argmin_{\widebar{\bZ}'} \frac{1}{n_3} \Big( \|\widebar{\bZ}'\|_{\ast} + \frac{\beta^{(t)}}{2} \|\widebar{\bB}_1^{(t+1)} - \widebar{\bZ}'\|_F^2 \Big),
\end{align*}
where $\widebar{\bB}_1^{(t+1)} = \mathtt{bdiag} (\xoverline{\bcB}_1^{(t+1)})$ and $\xoverline{\bcB}_1^{(t+1)} = L(\bcB_1^{(t+1)})$. Since $\widebar{\bZ}'$ is a block diagonal matrix, we only need to update all the diagonal block matrices $\widebar{\bZ}'^{(k)}$ by
\begin{align}    \label{eqn:solZ}
\widebar{\bZ}'^{(k)^{(t+1)}} = \Upsilon_{\frac{1}{\beta^{(t)}}} (\widebar{\bB}_1^{(k)^{(t+1)}}),
\end{align}
where $\Upsilon_{\zeta} (\cdot)$ denotes singular value thresholding (SVT) operator \cite{CaiCS.SIAM2010}. Finally, we can compute $\bcZ'^{(t+1)} = L^{-1} (\xoverline{\bcZ}'^{(t+1)})$.

\textbf{Updating $\bcE$:} We can update $\bcE$ by solving
\begin{align*}
\bcE^{(t+1)} = \argmin_{\bcE} \lambda \|\bcE\|_{2,1} + \frac{\beta^{(t)}}{2} \|\bcB_2^{(t+1)} - \bcE\|_F^2,
\end{align*}
where $\bcB_2^{(t+1)} = \bcX - \bcD \ast_{\boldsymbol{L}} \bcJ^{(t)} + \frac{\bcY_2^{(t)}}{\beta^{(t)}}$. Then the closed-form solution for $\bcE$ is given by
\begin{align}    \label{eqn:solE}
\bcE_{(j)}^{(t+1)} = \begin{cases}
\frac{ \| \bcB_{2(j)}^{(t+1)} \|_F - \frac{\lambda}{\beta^{(t)}} } {\| \bcB_{2(j)}^{(t+1)} \|_F} \bcB_{2(j)}^{(t+1)}, & \text{if} ~ \| \bcB_{2(j)}^{(t+1)} \|_F > \frac{\lambda}{\beta^{(t)}}, \\
\mathbf{0}, & \text{otherwise},
\end{cases}
\end{align}
where $\bcB_{2(j)}^{(t+1)} = \bcB_2^{(t+1)}(:,j,:)$.

\textbf{Updating $\bcJ$:} Keeping other tensors in \eqref{eqn:augform} fixed, the subproblem of updating $\bcJ$ at the ($t+1$)-th iteration has the form
\begin{align}    \label{eqn:probJ}
\bcJ^{(t+1)} & = \argmin_{\bcJ} \|\bcP_1^{(t+1)} - \bcJ\|_F^2 + \|\bcP_2^{(t+1)} - \bcD \ast_{\boldsymbol{L}} \bcJ\|_F^2    \nonumber \\
& = (\bcD^H \ast_{\boldsymbol{L}} \bcD + \bcI_{r_{\bcX}})^{-1} \ast_{\boldsymbol{L}} \Big( \bcP_1^{(t+1)} + \bcD^H \ast_{\boldsymbol{L}} \bcP_2^{(t+1)} \Big),
\end{align}
where $\bcP_1^{(t+1)} = \bcZ'^{(t+1)} + \frac{\bcY_1^{(t)}}{\beta^{(t)}}$ and $\bcP_2^{(t+1)} = \bcX - \bcE^{(t+1)} + \frac{\bcY_2^{(t)}}{\beta^{(t)}}$. We can transform \eqref{eqn:probJ} into the transform domain and it can be shown that the $k$-th frontal slice of $\xoverline{\bcJ}^{(t+1)}$, denoted by $\widebar{\bJ}^{(k)^{(t+1)}}$, has the following closed-form solution:
\begin{align}    \label{eqn:solJ}
\widebar{\bJ}^{(k)^{(t+1)}} = \widebar{\bQ}^{(k)} \Big( \widebar{\bP}_1^{(k)^{(t+1)}} + (\widebar{\bD}^{(k)})^H \widebar{\bP}_2^{(k)^{(t+1)}} \Big),
\end{align}
where $\bcQ = (\bcD^H \ast_{\boldsymbol{L}} \bcD + \bcI_{r_{\bcX}})^{-1}$. The update for $\bcJ^{(t+1)}$ can be expressed as $\bcJ^{(t+1)} = L^{-1} (\xoverline{\bcJ}^{(t+1)})$.

After obtaining a solution $(\bcZ'_{\star}, \bcE_{\star})$ to the problem \eqref{eqn:ortlrrreducedprob}, the optimal solution to \eqref{eqn:ortlrrprob} in the manuscript can be recovered by $(\bcV_{\bcX} \ast_{\boldsymbol{L}} \bcZ'_{\star}, \bcE_{\star})$. The ADMM framework for solving \eqref{eqn:ortlrrprob} in the manuscript can be summarized in Algorithm~\ref{algo:ortlrr}. Since problem \eqref{eqn:convertform} is a convex problem, whose convergence behavior can be guaranteed by leveraging the results in \cite{LinLS.NIPS2011}.

\textbf{Complexity Analysis.} At each iteration, when updating $\bcJ^{(t+1)}$, the computational cost for the matrix product and inverse transform is $\cO(r_{\bcX} n_1 n_2 n_3 + r_{\bcX} (n_1 + n_2) n_3^2)$. The major cost of updating $\bcZ'^{(t+1)}$ includes $n_3$ SVDs on $r_{\bcX} \times n_2$ matrices at the cost of $r_{\bcX}^2 n_2 n_3$ and computing the inverse transform at the cost of $\cO(r_{\bcX} n_2 n_3^2)$. As for updating $\bcE^{(t+1)}$, the step of computing tensor product $\bcD \ast_{\boldsymbol{L}} \bcJ^{(t+1)}$ takes $\cO(r_{\bcX} n_1 n_2 n_3 + r_{\bcX} (n_1 + n_2) n_3^2)$. Thus the total cost of Algorithm~\ref{algo:ortlrr} is $\cO(r_{\bcX} n_1 n_2 n_3 + r_{\bcX} (n_1 + n_2) n_3^2)$ per iteration. We conclude this section by noting that for some special invertible linear transforms $L$, e.g., DFT, since the application of DFT on an $n_3$-dimensional vector requires $\cO(n_3 \log(n_3))$ operations, the per-iteration complexity is $\cO(r_{\bcX} n_1 n_2 n_3 + r_{\bcX} (n_1 + n_2) n_3 \log(n_3))$.

\algsetup{indent=0.2em}
\begin{algorithm}[ht]
\caption{Outlier-Robust Tensor LRR (OR-TLRR)}
\label{algo:ortlrr}%
\textbf{Input:} Data tensor $\bcX \in \R^{n_1 \times n_2 \times n_3}$, linear transform $L$, and parameter $\lambda$. \\
\textbf{Initialize:} $\bcD = \bcU_{\bcX} \ast_{\boldsymbol{L}} \bcS_{\bcX}$ with the skinny t-SVD $\bcU_{\bcX} \ast_{\boldsymbol{L}} \bcS_{\bcX} \ast_{\boldsymbol{L}} \bcV_{\bcX}^H$ of $\bcX$, $\bcZ^{(0)} = \bcJ^{(0)} = \bcY_1^{(0)} = \mathbf{0} \in \R^{r_{\bcX} \times n_2 \times n_3}$, $\bcE^{(0)} = \bcY_2^{(0)} = \mathbf{0} \in \R^{n_1 \times n_2 \times n_3}$, $\gamma = 1.1$, $\beta^{(0)} = 10^{-5}$, $\beta_{\mathrm{max}} = 10^8$, $\epsilon = 10^{-5}$, and $t = 0$.

\begin{algorithmic}[1]
\WHILE{not converged}
\STATE Fix other variables and update $\bcZ'$ via \eqref{eqn:solZ}.
\STATE Fix other variables and update $\bcE$ via \eqref{eqn:solE}.
\STATE Fix other variables and update $\bcJ$ via \eqref{eqn:solJ}.
\STATE Update Lagrange multipliers: \\
$\bcY_1^{(t+1)} = \bcY_1^{(t)} + \beta^{(t)} (\bcZ'^{(t+1)} - \bcJ^{(t+1)})$, \\
$\bcY_2^{(t+1)} = \bcY_2^{(t)} + \beta^{(t)} (\bcX - \bcD \ast_{\boldsymbol{L}} \bcJ^{(t+1)} - \bcE^{(t+1)})$.
\STATE $\beta^{(t+1)} = \min(\beta_{\mathrm{max}}, \gamma \beta^{(t)})$.
\STATE Check the convergence conditions:
\begin{displaymath}
\text{max} \left\{\!\!\!\!\begin{array}{l}
\|\bcZ'^{(t+1)} - \bcZ'^{(t)}\|_{\infty}, \|\bcJ^{(t+1)} - \bcJ^{(t)}\|_{\infty} \\
\|\bcE^{(t+1)} - \bcE^{(t)}\|_{\infty}, \|\bcZ'^{(t+1)} - \bcJ^{(t+1)}\|_{\infty} \\
\|\bcX - \bcD \ast_{\boldsymbol{L}} \bcJ^{(t+1)} - \bcE^{(t+1)}\|_{\infty}
\end{array}\!\!\!\!\!\right\} \leq \epsilon.
\end{displaymath}
\STATE $t = t+1$.
\ENDWHILE
\end{algorithmic}
\textbf{Output:} $\bcZ_{\star} = \bcV_{\bcX} \ast_{\boldsymbol{L}} \bcZ'^{(t+1)}$ and $\bcE_{\star} = \bcE^{(t+1)}$.
\end{algorithm}

\section{OPTIMIZATION DETAILS OF OR-TLRR WITH MISSING ENTRIES}
\label{sec:missalgodetail}

In this section, we present our algorithm to solve problem \eqref{eqn:ortlrrmissreducedprob} in the manuscript. By introducing two auxiliary variables $\bcJ$ and $\bcH$, we make the objective function of \eqref{eqn:ortlrrmissreducedprob} separable and reformulate \eqref{eqn:ortlrrmissreducedprob} as
\begin{align}    \label{eqn:convertmissform}
& \min_{\bcZ',\bcJ,\bcH,\bcE} \|\bcZ'\|_{\ast} + \lambda \|\mathscr{P}_{\mathbf{\Omega}}(\bcE)\|_{2,1}    \nonumber \\
&\quad \text{s.t.} \quad \bcZ' = \bcJ, \bcH = \bcD \ast_{\boldsymbol{L}} \bcJ + \bcE, \mathscr{P}_{\mathbf{\Omega}}(\bcX_{\text{miss}}) = \mathscr{P}_{\mathbf{\Omega}}(\bcH).
\end{align}
The corresponding augmented Lagrangian function of \eqref{eqn:convertmissform} is
\begin{align}    \label{eqn:augmissform}
&\cL_2(\bcZ',\bcJ,\bcH,\bcE,\bcY_1,\bcY_2,\beta) = \|\bcZ'\|_{\ast} + \lambda \|\mathscr{P}_{\mathbf{\Omega}}(\bcE)\|_{2,1} + \Psi_{\mathbf{\Omega}}(\bcX_{\mathrm{miss}} - \bcH) + \langle \bcY_1, \bcZ' - \bcJ \rangle + \langle \bcY_2, \bcH - \bcD \ast_{\boldsymbol{L}} \bcJ - \bcE \rangle     \nonumber \\
&\qquad\qquad + \frac{\beta}{2} \big( \|\bcZ' - \bcJ\|_F^2 + \|\bcH - \bcD \ast_{\boldsymbol{L}} \bcJ - \bcE\|_F^2 \big),
\end{align}
where $\Psi_{\mathbf{\Omega}}(\cdot)$ denotes the indicator function
\begin{displaymath}
\Psi_{\mathbf{\Omega}}(\bcA - \bcB) = \left\{ \begin{array}{ll}
0, & \text{if $\bcA_{\mathbf{\Omega}} = \bcB_{\mathbf{\Omega}}$,} \\
\infty, & \text{otherwise}.
\end{array} \right.
\end{displaymath}
The difference between the optimization of \eqref{eqn:augmissform} and that of \eqref{eqn:augform} mainly lies in the step of updating $\bcE$ and one additional step for updating $\bcH$. Concretely, the $\bcH$-subproblem is
\begin{align}    \label{eqn:probH}
\bcH^{(t+1)} = \argmin_{\bcH} \|\bcH - \bcD \ast_{\boldsymbol{L}} \bcJ^{(t)} - \bcE^{(t)} + \frac{\bcY_2^{(t)}}{\beta^{(t)}}\|_F^2 + \Psi_{\mathbf{\Omega}}(\bcX_{\mathrm{miss}} - \bcH).
\end{align}
Thus, we have $\bcH^{(t+1)} = \mathscr{P}_{\mathbf{\Omega}^c}(\bcD \ast_{\boldsymbol{L}} \bcJ^{(t)} + \bcE^{(t)} - \frac{\bcY_2^{(t)}}{\beta^{(t)}}) + \mathscr{P}_{\mathbf{\Omega}}(\bcX_{\mathrm{miss}})$, where $\mathbf{\Omega}^c$ is the complement of $\mathbf{\Omega}$. On the other hand, the variable $\bcE$ can be updated by solving
\begin{align}    \label{eqn:probE}
\bcE^{(t+1)} = \argmin_{\bcE} \lambda \|\mathscr{P}_{\mathbf{\Omega}}(\bcE)\|_{2,1} + \frac{\beta^{(t)}}{2} \|\bcB^{(t+1)} - \bcE\|_F^2,
\end{align}
where $\bcB^{(t+1)} = \bcH^{(t+1)} - \bcD \ast_{\boldsymbol{L}} \bcJ^{(t)} + \frac{\bcY_2^{(t)}}{\beta^{(t)}}$. Then the closed-form solution of $\bcE$ is given by
\begin{align*}
\bcE_{(j)}^{(t+1)} = \begin{cases}
\frac{ \| \bcW_{(j)} \odot \bcB_{(j)}^{(t+1)} \|_F - \frac{\lambda}{\beta^{(t)}} } {\| \bcW_{(j)} \odot \bcB_{(j)}^{(t+1)} \|_F} \bcW_{(j)} \odot \bcB_{(j)}^{(t+1)} + (1 - \bcW_{(j)}) \odot \bcB_{(j)}^{(t+1)}, ~ \text{if} ~ \| \bcW_{(j)} \odot \bcB_{(j)}^{(t+1)} \|_F > \frac{\lambda}{\beta^{(t)}}, \\
(1 - \bcW_{(j)}) \odot \bcB_{(j)}^{(t+1)}, \qquad \text{otherwise},
\end{cases}
\end{align*}
where $\bcB_{(j)}^{(t+1)} = \bcB^{(t+1)}(:,j,:)$. The implementation of the ADMM algorithm is outlined in Algorithm~\ref{algo:ortlrr-ewzf}. We dub this approach \emph{Outlier-Robust Tensor LRR by Entry-Wise Zero-Fill} (OR-TLRR-EWZF). We conclude this section by noting that we can also robustify TLRR \cite{ZhouLFLY.PAMI2021} by replacing $\lambda \|\mathscr{P}_{\mathbf{\Omega}}(\bcE)\|_{2,1}$ in \eqref{eqn:ortlrrmissprob} with $\lambda \|\mathscr{P}_{\mathbf{\Omega}}(\bcE)\|_1$ as a means of performing TLRR with missing data. Effectively, the overall optimization process also relies on Algorithm~\ref{algo:ortlrr-ewzf}, with the difference being that the closed-form solution of $\bcE$ can now be computed by
\begin{align*}
\bcE_{(j)}^{(t+1)} = \Pi_{\frac{\lambda}{\beta^{(t)}}} \Big( \bcW_{(j)} \odot \bcB_{(j)}^{(t+1)} \Big) + (1 - \bcW_{(j)}) \odot \bcB_{(j)}^{(t+1)},
\end{align*}
where $\Pi_{\frac{\lambda}{\beta^{(t)}}} (\cdot)$ is the soft-thresholding operator \cite{Donoho.TIT1995}. We call the resulting algorithm TLRR-EWZF in our experiments.

\algsetup{indent=0.2em}
\begin{algorithm}[t]
\caption{Outlier-Robust Tensor LRR by Entry-Wise Zero-Fill (OR-TLRR-EWZF)}
\label{algo:ortlrr-ewzf}%
\textbf{Input:} Partially observed data tensor $\mathscr{P}_{\mathbf{\Omega}}(\bcX) \in \R^{n_1 \times n_2 \times n_3}$, linear transform $L$, and parameter $\lambda$. \\
\textbf{Initialize:} $\bcD = \bcU_{\bcX} \ast_{\boldsymbol{L}} \bcS_{\bcX}$ with the skinny t-SVD $\bcU_{\bcX} \ast_{\boldsymbol{L}} \bcS_{\bcX} \ast_{\boldsymbol{L}} \bcV_{\bcX}^H$ of $\bcX_{\mathrm{miss}}$, $\bcZ^{(0)} = \bcJ^{(0)} = \bcY_1^{(0)} = \mathbf{0} \in \R^{r_{\bcX} \times n_2 \times n_3}$, $\bcE^{(0)} = \bcH^{(0)} = \bcY_2^{(0)} = \mathbf{0} \in \R^{n_1 \times n_2 \times n_3}$, $\gamma = 1.1$, $\beta^{(0)} = 10^{-5}$, $\beta_{\mathrm{max}} = 10^8$, $\epsilon = 10^{-5}$, and $t = 0$.

\begin{algorithmic}[1]
\WHILE{not converged}
\STATE Fix other variables and update $\bcH$ by solving \eqref{eqn:probH}.
\STATE Fix other variables and update $\bcZ'$ by solving:
\begin{align*}
\bcZ'^{(t+1)} = \argmin_{\bcZ'} \|\bcZ'\|_{\ast} + \frac{\mu^{(t)}}{2} \|\bcZ' - \bcJ^{(t)} + \frac{\bcY_1^{(t)}}{\beta^{(t)}}\|_F^2.
\end{align*}
\STATE Fix other variables and update $\bcE$ by solving \eqref{eqn:probE}.
\STATE Fix other variables and update $\bcJ$ by solving:
\begin{align*}
\bcJ^{(t+1)} = \argmin_{\bcJ} \|\bcZ'^{(t+1)} - \bcJ + \frac{\bcY_1^{(t)}}{\beta^{(t)}}\|_F^2 + \|\bcH^{(t+1)} - \bcD \ast_{\boldsymbol{L}} \bcJ - \bcE^{(t+1)} + \frac{\bcY_2^{(t)}}{\beta^{(t)}}\|_F^2.
\end{align*}
\STATE Update Lagrange multipliers: \\
$\bcY_1^{(t+1)} = \bcY_1^{(t)} + \beta^{(t)} (\bcZ'^{(t+1)} - \bcJ^{(t+1)})$, \\
$\bcY_2^{(t+1)} = \bcY_2^{(t)} + \beta^{(t)} (\bcH^{(t+1)} - \bcD \ast_{\boldsymbol{L}} \bcJ^{(t+1)} - \bcE^{(t+1)})$.
\STATE $\beta^{(t+1)} = \min(\beta_{\mathrm{max}}, \gamma \beta^{(t)})$.
\STATE Check the convergence conditions:
\begin{displaymath}
\text{max} \left\{\!\!\!\!\begin{array}{l}
\|\bcZ'^{(t+1)} - \bcZ'^{(t)}\|_{\infty}, \|\bcJ^{(t+1)} - \bcJ^{(t)}\|_{\infty} \\
\|\bcE^{(t+1)} - \bcE^{(t)}\|_{\infty}, \|\bcH^{(t+1)} - \bcH^{(t)}\|_{\infty} \\
\|\bcZ'^{(t+1)} - \bcJ^{(t+1)}\|_{\infty} \\
\|\bcH^{(t+1)} - \bcD \ast_{\boldsymbol{L}} \bcJ^{(t+1)} - \bcE^{(t+1)}\|_{\infty}
\end{array}\!\!\!\!\!\right\} \leq \epsilon.
\end{displaymath}
\STATE $t = t+1$.
\ENDWHILE
\end{algorithmic}
\textbf{Output:} $\bcZ_{\star} = \bcV_{\bcX} \ast_{\boldsymbol{L}} \bcZ'^{(t+1)}$ and $\bcE_{\star} = \bcE^{(t+1)}$.
\end{algorithm}

\section{PROOF OF THEOREM~\ref{thm:linearrep}}
\label{sec:prooflinearrep}

\begin{proof}
Let $\bx_j = \mathtt{vec} (\bcX_{(j)})$ and $\bX_j = \mathtt{squeeze} (\bcX_{(j)}) \in \R^{n_1 \times n_3}$. Then, we can obtain
\begin{align*}
\bx_j = \begin{bmatrix}
\bX_j \be_1 \\
\vdots \\
\bX_j \be_{n_3}
\end{bmatrix},
\end{align*}
where $\be_q \in \R^{n_3}$ is the $q$-th standard basis of $\R^{n_3}$ whose $q$-th entry is 1 and the rest is 0. Next, we define a tensor $\bcA \in \R^{n_1 \times p \times n_3}$ with $\bcA_{(j)} = \mathtt{ivec} (\ba_j)$, and let $\bA_j = \mathtt{squeeze} (\bcA_{(j)}) \in \R^{n_1 \times n_3}$, $j = 1, \dots, p$. Then we can write $\bx_j = \bA \bz_j$ as its equivalent form:
\begin{align*}
\begin{bmatrix}
\bX_j \be_1 \\
\vdots \\
\bX_j \be_{n_3}
\end{bmatrix}
=
\begin{bmatrix}
\bA_1 \be_1 & \ldots & \bA_p \be_1 \\
\vdots & \ddots & \vdots \\
\bA_1 \be_{n_3} & \ldots & \bA_p \be_{n_3}
\end{bmatrix}
\bz_j.
\end{align*}
For the $s$-th row of $\boldsymbol{L}$ ($s = 1, \dots, n_3$), we have
\begin{align}
\bX_j \boldsymbol{l}_{s,T}^T & = \sum_{q=1}^{n_3} \boldsymbol{L}_{s,q} \bX_j \be_q  \nonumber \\
& = \sum_{q=1}^{n_3} \boldsymbol{L}_{s,q} [\bA_1 \be_q, \dots, \bA_p \be_q] \bz_j  \nonumber \\
& = [\bA_1, \dots, \bA_p] \sum_{q=1}^{n_3} \boldsymbol{L}_{s,q}
\begin{bmatrix}
\be_q & &  \\
 & \ddots &  \\
 & & \be_q
\end{bmatrix}
\bz_j  \nonumber \\
& = [\bA_1, \dots, \bA_p]
\begin{bmatrix}
\boldsymbol{l}_{s,T}^T & &  \\
 & \ddots &  \\
 & & \boldsymbol{l}_{s,T}^T
\end{bmatrix}
\bz_j  \nonumber \\
& = [\bA_1 \boldsymbol{l}_{s,T}^T, \dots, \bA_p \boldsymbol{l}_{s,T}^T] \bz_j.
\end{align}
We further define
\begin{align*}
\bA_{\boldsymbol{l}_{s,T}} = [\bA_1 \boldsymbol{l}_{s,T}^T, \dots, \bA_p \boldsymbol{l}_{s,T}^T] \in \C^{n_1 \times n_3}
\end{align*}
and a tensor $\xoverline{\bcZ} \in \R^{p \times n_2 \times n_3}$ with $\xoverline{\bcZ}(:,j,k) = \bz_j \in \R^p$, $k = 1, \dots, n_3$. The tensor $\bcZ$ can be computed by setting $\bcZ = L^{-1} (\xoverline{\bcZ})$. Note that the block diagonal matrix $\widebar{\bA}$ can now be written as
\begin{align*}
\widebar{\bA} =
\begin{bmatrix}
\bA_{\boldsymbol{l}_{1,T}} &  &  \\
 & \ddots &  \\
 & & \bA_{\boldsymbol{l}_{n_3,T}}
\end{bmatrix}.
\end{align*}
Let $\widebar{\bX}_j = \mathtt{bdiag} (L(\bcX_{(j)}))$ and $\widebar{\bZ}_j = \mathtt{bdiag} (L(\bcZ_{(j)}))$, we can establish that
\begin{align*}
\widebar{\bX}_j & =
\begin{bmatrix}
\bX_j \boldsymbol{l}_{1,T}^T & &  \\
 & \ddots &  \\
 & & \bX_j \boldsymbol{l}_{n_3,T}^T
\end{bmatrix}  \nonumber \\
& =
\begin{bmatrix}
\bA_{\boldsymbol{l}_{1,T}} &  &  \\
 & \ddots &  \\
 & & \bA_{\boldsymbol{l}_{n_3,T}}
\end{bmatrix}
\begin{bmatrix}
\bz_j & &  \\
 & \ddots &  \\
 & & \bz_j
\end{bmatrix}  \nonumber \\
& = \widebar{\bA} \widebar{\bZ}_j.
\end{align*}
So for any $\bz_j$, we have
\begin{align}    \label{eqn:linearcomb}
\bcX_{(j)} = \bcA \ast_{\boldsymbol{L}} \bcZ_{(j)}, \quad \forall j = 1, \dots, n_2.
\end{align}
Conversely, for any tensor linear representation \eqref{eqn:linearcomb}, $\bcZ_{(j)}$ may not satisfy $\xoverline{\bcZ}(:,j,k) = \bz_j$, i.e., there may not exist $\bz_j$ such that $\bx_j = \bA \bz_j$. Thus, if the linear representation relationship in vector space holds, then there exists a feasible solution such that the tensor linear representation also holds. However, there is no guarantee that the reverse is also true.
\end{proof}

\section{PROOFS OF LEMMA~\ref{lem:cleansol} AND LEMMA~\ref{lem:sublemma}}
\label{sec:proofsubspacelemma}

Before proving Lemma~\ref{lem:cleansol}, we first give the following lemma which is used in the proof.

\begin{lemma}    \label{lemma:tlrrnoiselesssol}
Let $\bcA \in \R^{n_1 \times n_4 \times n_3}$ be a nonzero dictionary and $\bcX = \bcA \ast_{\boldsymbol{L}} \bcZ$ has feasible solutions, i.e., $\bcX \in \text{Range} (\bcA)$. Then the following problem
\begin{align*}
\min_{\bcZ} \|\bcZ\|_{\ast} \quad \text{s.t.} \quad \bcX = \bcA \ast_{\boldsymbol{L}} \bcZ
\end{align*}
has a unique minimizer defined by
\begin{align*}
\bcZ_{\star} = \bcA^{\dagger} \ast_{\boldsymbol{L}} \bcX.
\end{align*}
\end{lemma}

\begin{proof}
For the problem
\begin{align}    \label{eqn:tnnminimize}
\min_{\bcZ} \|\bcZ\|_{\ast} \quad \text{s.t.} \quad \bcX = \bcA \ast_{\boldsymbol{L}} \bcZ
\end{align}
is equivalent to
\begin{align}    \label{eqn:rankminimize}
\min_{\widebar{\bZ}} \|\widebar{\bZ}\|_{\ast} \quad \text{s.t.} \quad \widebar{\bX} = \widebar{\bA} \widebar{\bZ}.
\end{align}
Since $\widebar{\bX}$, $\widebar{\bA}$ and $\widebar{\bZ}$ are three block-diagonal matrices, \eqref{eqn:rankminimize} can be divided into $n_3$ simple problems.
\begin{align}    \label{eqn:rankminimize2}
\min_{\widebar{\bZ}^{(k)}} \|\widebar{\bZ}^{(k)}\|_{\ast} \quad \text{s.t.} \quad \widebar{\bX}^{(k)} = \widebar{\bA}^{(k)} \widebar{\bZ}^{(k)}, k = 1, \dots, n_3.
\end{align}
Since $\bcX \in \text{Range} (\bcA)$, we have $\widebar{\bX}^{(k)} \in \text{Span} (\widebar{\bA}^{(k)})$, where $\text{Span} (\bB)$ denotes the linear space spanned by the columns of a matrix $\bB$ \cite{LiuLYSYM.PAMI2013}. Then, by \cite[Theorem 4.1]{LiuLYSYM.PAMI2013}, we have that $\widebar{\bZ}^{(k)} = (\widebar{\bA}^{(k)})^{\dagger} \widebar{\bX}^{(k)}$ is the unique optimal solution to \eqref{eqn:rankminimize2}. Hence, we can obtain the unique solution $\widebar{\bZ} = (\widebar{\bA})^{\dagger} \widebar{\bX}$ to problem \eqref{eqn:rankminimize}. Furthermore, the unique solution to \eqref{eqn:tnnminimize} is $\bcZ = \bcA^{\dagger} \ast_{\boldsymbol{L}} \bcX$. The proof is completed.
\end{proof}

\subsection{Proof of Lemma~\ref{lem:cleansol}}

\begin{proof}
Note that OR-TLRR problem \eqref{eqn:ortlrrprob} always has feasible solutions, e.g., $(\bcZ = \mathbf{0}, \bcE = \bcX)$. Thus, an optimal solution, denoted as $(\bcZ_{\star}, \bcE_{\star})$, exists. By Lemma~\ref{lemma:tlrrnoiselesssol}, we have
\begin{align*}
\bcZ_{\star} & = \argmin_{\bcZ} \| \bcZ \|_{\ast} \quad \text{s.t.} \quad \bcX - \bcE_{\star} = \bcX \ast_{\boldsymbol{L}} \bcZ  \\
& = \bcX^{\dagger} \ast_{\boldsymbol{L}} (\bcX - \bcE_{\star}),
\end{align*}
which simply leads to $\bcZ_{\star} \in \bcP_{\bcV_{\bcX}}^L$.
\end{proof}

\subsection{Proof of Lemma~\ref{lem:sublemma}}

\begin{proof}
Suppose the skinny t-SVDs of $\bcX$, $\bcL_0$ and $\bcE_0$ are $\bcU_{\bcX} \ast_{\boldsymbol{L}} \bcS_{\bcX} \ast_{\boldsymbol{L}} \bcV_{\bcX}^H$, $\bcU_0 \ast_{\boldsymbol{L}} \bcS_0 \ast_{\boldsymbol{L}} \bcV_0^H$ and $\bcU_{\bcE} \ast_{\boldsymbol{L}} \bcS_{\bcE} \ast_{\boldsymbol{L}} \bcV_{\bcE}^H$, respectively. Let $\bcU_0^{\perp}$ and $\bcU_{\bcE}^{\perp}$ be the orthogonal complements of $\bcU_0$ and $\bcU_{\bcE}$, respectively. Since $\text{Range} (\bcL_0)$ and $\text{Range} (\bcE_0)$ are independent, $\text{Span} (\widebar{\bU}_0^{(k)})$ and $\text{Span} (\widebar{\bU}_{\bcE}^{(k)})$ are independent, $\forall k = 1, \dots, n_3$, which further suggests that $[\widebar{\bU}_0^{(k)^{\perp}}, \widebar{\bU}_{\bcE}^{(k)^{\perp}}]$ spans the whole space. Thus the following system has feasible solutions $\widebar{\bY}_0^{(k)}$ and $\widebar{\bY}_{\bcE}^{(k)}$:
\begin{align*}
\widebar{\bU}_0^{(k)^{\perp}} (\widebar{\bU}_0^{(k)^{\perp}})^H \widebar{\bY}_0^{(k)} + \widebar{\bU}_{\bcE}^{(k)^{\perp}} (\widebar{\bU}_{\bcE}^{(k)^{\perp}})^H \widebar{\bY}_{\bcE}^{(k)} = \bI_{n_1}.
\end{align*}
Denote $\widebar{\bY}^{(k)} = \bI_{n_1} - \widebar{\bU}_0^{(k)^{\perp}} (\widebar{\bU}_0^{(k)^{\perp}})^H \widebar{\bY}_0^{(k)}$, then it can be verified that
\begin{align*}
(\widebar{\bL}_0^{(k)})^H \widebar{\bY}^{(k)} = (\widebar{\bL}_0^{(k)})^H ~ \text{and} ~ (\widebar{\bE}_0^{(k)})^H \widebar{\bY}^{(k)} = \mathbf{0}.
\end{align*}
We have that $(\widebar{\bX}^{(k)})^H \widebar{\bY}^{(k)} = (\widebar{\bL}_0^{(k)})^H$, $\forall k = 1, \dots, n_3$. Hence the system $\bcX^H \ast_{\boldsymbol{L}} \bcY = \bcL_0^H$ has a feasible solution $\bcY$, which simply leads to $\bcV_{\bcX} \ast_{\boldsymbol{L}} \bcS_{\bcX}^{\dagger} \ast_{\boldsymbol{L}} \bcU_{\bcX}^H \ast_{\boldsymbol{L}} \bcY \ast_{\boldsymbol{L}} \bcU_0 \ast_{\boldsymbol{L}} \bcS_0 = \bcV_0$ and thus $\bcV_0 \in \bcP_{\bcV_{\bcX}}^L$.
\end{proof}

\section{PROOFS OF THEOREM~\ref{thm:condortlrr}}
\label{sec:proofrecoverytheorem}

Now we prove Theorem~\ref{thm:condortlrr} in manuscript. Section~\ref{ssec:equivcond} identifies the necessary and sufficient conditions (called equivalent conditions) for any pair $(\bcZ, \bcE)$ to produce the exact recovery results. Section~\ref{ssec:dualcond} gives the dual conditions under which OR-TLRR succeeds. Then we construct a dual certificate in Section~\ref{ssec:dualcert} such that the dual conditions hold. Section~\ref{ssec:lemmaproof} gives the proofs of some lemmas which are used in Section~\ref{ssec:dualcert}. For convenience, we consider the general Bernoulli sampling. Specifically, for $\mathbf{\Theta}_0 = \{ j | \delta_j = 1 \}$, $\delta_j$ are i.i.d. Bernoulli variables and take value 1 with probability $\rho$ and 0 with probability $1 - \rho$. Thus, we denote the Bernoulli sampling by $\mathbf{\Theta}_0 \sim \text{Ber} (\rho)$.

\subsection{Equivalent Conditions}
\label{ssec:equivcond}

\begin{theorem}    \label{thm:zecond}
Let the pair $(\bcZ', \bcE')$ satisfy $\bcX = \bcX \ast_{\boldsymbol{L}} \bcZ' + \bcE'$. Denote the skinny t-SVD of $\bcZ'$ as $\bcU' \ast_{\boldsymbol{L}} \bcS' \ast_{\boldsymbol{L}} \bcV'^H$, and the column support of $\bcE'$ as $\mathbf{\Theta}'$. If $\bcP_{\bcV_0}^L (\bcZ') = \bcZ'$ and $\bcP_{\mathbf{\Theta}_0} (\bcE') = \bcE'$, then $\bcU' \ast_{\boldsymbol{L}} \bcU'^H = \bcV_0 \ast_{\boldsymbol{L}} \bcV_0^H$ and $\mathbf{\Theta}' = \mathbf{\Theta}_0$.
\end{theorem}

\begin{proof}
Since $\bcP_{\bcV_0}^L (\bcZ') = \bcZ'$ implies that $\bcU'$ is a subspace of $\bcV_0$, to prove $\bcU' \ast_{\boldsymbol{L}} \bcU'^H = \bcV_0 \ast_{\boldsymbol{L}} \bcV_0^H$, it is sufficient to prove $\text{rank}_t (\bcZ') \geq \text{rank}_t (\bcL_0)$. Note that $\bcP_{\mathbf{\Theta}_0^{\perp}} (\bcX) = \bcL_0$. We have
\begin{align*}
\bcL_0 = \bcP_{\mathbf{\Theta}_0^{\perp}} (\bcX) = \bcP_{\mathbf{\Theta}_0^{\perp}} (\bcX \ast_{\boldsymbol{L}} \bcZ' + \bcE') = \bcP_{\mathbf{\Theta}_0^{\perp}} (\bcX \ast_{\boldsymbol{L}} \bcZ') = \bcX \ast_{\boldsymbol{L}} \bcP_{\mathbf{\Theta}_0^{\perp}} (\bcZ').
\end{align*}
Therefore, we can establish that
\begin{align*}
\text{rank}_t (\bcL_0) = \text{rank}_t (\bcX \ast_{\boldsymbol{L}} \bcP_{\mathbf{\Theta}_0^{\perp}} (\bcZ')) \leq \text{rank}_t (\bcP_{\mathbf{\Theta}_0^{\perp}} (\bcZ')) \leq \text{rank}_t (\bcZ').
\end{align*}
Next, since $\bcP_{\mathbf{\Theta}_0} (\bcE') = \bcE'$ implies that $\mathbf{\Theta}' \subseteq \mathbf{\Theta}_0$, if we can prove $\mathbf{\Theta}_0 \cap (\mathbf{\Theta}')^c = \emptyset$, then we can have $\mathbf{\Theta}' = \mathbf{\Theta}_0$. In order to do so, we first have
\begin{align*}
\bcP_{\mathbf{\Theta}_0} (\bcL_0) = \mathbf{0} & \Rightarrow \bcU_0 \ast_{\boldsymbol{L}} \bcS_0 \ast_{\boldsymbol{L}} \bcP_{\mathbf{\Theta}_0} (\bcV_0^H) = \mathbf{0}  \\
& \Rightarrow \bcP_{\mathbf{\Theta}_0} (\bcV_0^H) = \mathbf{0}  \\
& \Rightarrow \bcV_0 \ast_{\boldsymbol{L}} \bcP_{\mathbf{\Theta}_0} (\bcV_0^H) = \mathbf{0}.
\end{align*}
Also, $\bcV_0 \in \bcP_{\bcV_{\bcX}}^L$ implies that $\bcV_0 \ast_{\boldsymbol{L}} \bcV_0^H = \bcV_0 \ast_{\boldsymbol{L}} \bcV_0^H \ast_{\boldsymbol{L}} \bcV_{\bcX} \ast_{\boldsymbol{L}} \bcV_{\bcX}^H$, which then leads to $\bcV_0 \ast_{\boldsymbol{L}} \bcV_0^H \ast_{\boldsymbol{L}} \bcV_{\bcX} \ast_{\boldsymbol{L}} \bcP_{\mathbf{\Theta}_0} (\bcV_{\bcX}^H) = \bcV_0 \ast_{\boldsymbol{L}} \bcP_{\mathbf{\Theta}_0} (\bcV_0^H) = \mathbf{0}$. Because $\mathbf{\Theta}_0 \cap (\mathbf{\Theta}')^c \subseteq \mathbf{\Theta}_0$, we can obtain $\bcV_0 \ast_{\boldsymbol{L}} \bcV_0^H \ast_{\boldsymbol{L}} \bcV_{\bcX} \ast \bcP_{\mathbf{\Theta}_0 \cap (\mathbf{\Theta}')^c} (\bcV_{\bcX}^H) = \mathbf{0}$. On the other hand, note that $\mathbf{\Theta}_0 \cap (\mathbf{\Theta}')^c \subseteq (\mathbf{\Theta}')^c$, we have the following:
\begin{align*}
& \bcX = \bcX \ast_{\boldsymbol{L}} \bcZ' + \bcE'  \\
\Rightarrow & \bcP_{\mathbf{\Theta}_0 \cap (\mathbf{\Theta}')^c} (\bcX) = \bcX \ast_{\boldsymbol{L}} \bcP_{\mathbf{\Theta}_0 \cap (\mathbf{\Theta}')^c} (\bcZ')  \\
\Rightarrow & \bcU_{\bcX} \ast_{\boldsymbol{L}} \bcS_{\bcX} \ast_{\boldsymbol{L}} \bcP_{\mathbf{\Theta}_0 \cap (\mathbf{\Theta}')^c} (\bcV_{\bcX}^H)  \\
&\qquad = \bcU_{\bcX} \ast_{\boldsymbol{L}} \bcS_{\bcX} \ast_{\boldsymbol{L}} \bcV_{\bcX}^H \ast_{\boldsymbol{L}} \bcP_{\mathbf{\Theta}_0 \cap (\mathbf{\Theta}')^c} (\bcZ')  \\
\Rightarrow & \bcP_{\mathbf{\Theta}_0 \cap (\mathbf{\Theta}')^c} (\bcV_{\bcX}^H) = \bcV_{\bcX}^{\ast} \ast_{\boldsymbol{L}} \bcP_{\mathbf{\Theta}_0 \cap (\mathbf{\Theta}')^c} (\bcZ')  \\
\Rightarrow & \bcV_0 \ast_{\boldsymbol{L}} \bcV_0^H \ast_{\boldsymbol{L}} \bcV_{\bcX} \ast_{\boldsymbol{L}} \bcP_{\mathbf{\Theta}_0 \cap (\mathbf{\Theta}')^c} (\bcV_{\bcX}^H)  \\
&\qquad = \bcV_0 \ast_{\boldsymbol{L}} \bcV_0^H \ast_{\boldsymbol{L}} \bcV_{\bcX} \ast_{\boldsymbol{L}} \bcV_{\bcX}^H \ast_{\boldsymbol{L}} \bcP_{\mathbf{\Theta}_0 \cap (\mathbf{\Theta}')^c} (\bcZ') \\
\Rightarrow & \bcV_0 \ast_{\boldsymbol{L}} \bcV_0^H \ast_{\boldsymbol{L}} \bcV_{\bcX} \ast_{\boldsymbol{L}} \bcV_{\bcX}^H \ast_{\boldsymbol{L}} \bcP_{\mathbf{\Theta}_0 \cap (\mathbf{\Theta}')^c} (\bcZ') = \mathbf{0}  \\
\overset{{\scalemath{0.5}{\tcircle{1}}}}{\Rightarrow} & \bcV_0 \ast_{\boldsymbol{L}} \bcV_0^H \ast_{\boldsymbol{L}} \bcP_{\mathbf{\Theta}_0 \cap (\mathbf{\Theta}')^c} (\bcZ') = \mathbf{0}  \\
\overset{{\scalemath{0.5}{\tcircle{2}}}}{\Rightarrow} & \bcP_{\mathbf{\Theta}_0 \cap (\mathbf{\Theta}')^c} (\bcZ') = \mathbf{0},
\end{align*}
where {\textcircled{\small{1}}} holds because $\bcZ' \in \bcP_{\bcV_{\bcX}}^L$ and {\textcircled{\small{2}}} holds because $\bcZ' = \bcV_0 \ast_{\boldsymbol{L}} \bcV_0^H \ast_{\boldsymbol{L}} \bcZ'$. It follows from $\bcX = \bcL_0 + \bcE_0$ that $\bcP_{\mathbf{\Theta}_0 \cap (\mathbf{\Theta}')^c} (\bcE_0) = \bcP_{\mathbf{\Theta}_0 \cap (\mathbf{\Theta}')^c} (\bcX - \bcL_0) = \bcP_{\mathbf{\Theta}_0 \cap (\mathbf{\Theta}')^c} (\bcX) = \bcX \ast_{\boldsymbol{L}} \bcP_{\mathbf{\Theta}_0 \cap (\mathbf{\Theta}')^c} (\bcZ') = \mathbf{0}$. Hence, $\mathbf{\Theta}_0 \cap (\mathbf{\Theta}')^c = \emptyset$.
\end{proof}

This theorem implies that the exact recovery is equivalent to two constraints: $\bcP_{\bcV_0}^L (\bcZ) = \bcZ$ and $\bcP_{\mathbf{\Theta}_0} (\bcE) = \bcE$. As will be seen in the following, this can largely facilitate the proof of Theorem~\ref{thm:condortlrr}.

\subsection{Dual Conditions}
\label{ssec:dualcond}

We now prove the dual conditions of the OR-TLRR problem \eqref{eqn:ortlrrprob}.

\begin{lemma} [Dual conditions for OR-TLRR]    \label{lem:condortlrr}
Assume that $\text{Range} (\bcL_0) = \text{Range} (\bcP_{\mathbf{\Theta}_0} (\bcL_0))$, $\bcE_0(:,j,:) \notin \text{Range} (\bcL_0)$ for $j \in \mathbf{\Theta}_0$, and $(\bcZ_{\star}, \bcE_{\star}) = (\bcV_0 \ast_{\boldsymbol{L}} \bcV_0^H + \bcH, \bcE_0 - \bcX \ast_{\boldsymbol{L}} \bcH)$ is an arbitrary solution to \eqref{eqn:ortlrrprob}. Let $\hat{\bcZ} = \bcV_0 \ast_{\boldsymbol{L}} \bcV_0^H + \bcP_{\mathbf{\Theta}_0} \bcP_{\bcV_0}^L (\bcH)$ and $\hat{\bcE} = \bcE_0 - \bcX \ast_{\boldsymbol{L}} \bcP_{\mathbf{\Theta}_0} \bcP_{\bcV_0}^L (\bcH)$. Suppose that $\|\bcP_{\hat{\mathbf{\Theta}}} \bcP_{\hat{\bcV}}\| < 1$, $\lambda > 2 \sqrt{\frac{\mu r}{n_2 \tau}}$, and $\hat{\bcL} = \bcX \ast_{\boldsymbol{L}} \hat{\bcZ}$ obeys the tensor column-incoherence condition. Then the column space of $\bcZ_{\star}$ is the same as the row space of $\bcL_0$ and $\bcE_{\star}$ has the same column indices as those of $\bcE_0$, provided that there exists a pair $(\bcW, \bcF)$ obeying
\begin{align*}
\bcW = \lambda (\bcX^{\ast} \ast_{\boldsymbol{L}} \cB(\bcE) + \bcF),
\end{align*}
with $\bcP_{\hat{\bcV}} (\bcW) = \mathbf{0}$, $\| \bcW \| \leq 1/2$, $\bcP_{\hat{\mathbf{\Theta}}} (\bcF) = \mathbf{0}$ and $\| \bcF \|_{2,\infty} \leq 1/2$.
\end{lemma}

\begin{proof}
The subgradients of tensor nuclear norm and tensor $\ell_{2,1}$ norm can be written as
\begin{align*}
& \partial_{\hat{\bcZ}} \| \bcZ \|_{\ast} = \{ \hat{\bcU} \ast_{\boldsymbol{L}} \hat{\bcV}^H + \hat{\bcQ} | \bcP_{\hat{\bcT}} (\hat{\bcQ}) = \mathbf{0}, \| \hat{\bcQ} \| \leq 1 \},  \\
& \partial_{\hat{\bcE}} \| \bcE \|_{2,1} = \{ \cB(\hat{\bcE}) + \hat{\bcC} | \bcP_{\hat{\mathbf{\Theta}}} (\hat{\bcC}) = \mathbf{0}, \| \hat{\bcC} \|_{2,\infty} \leq 1 \}.
\end{align*}
We denote the skinny t-SVDs of $\hat{\bcZ}$ and $\hat{\bcL}$ as $\hat{\bcU} \ast_{\boldsymbol{L}} \hat{\bcS} \ast_{\boldsymbol{L}} \hat{\bcV}^H$ and $\bcU_{\hat{\bcL}} \ast_{\boldsymbol{L}} \bcS_{\hat{\bcL}} \ast_{\boldsymbol{L}} \bcV_{\hat{\bcL}}^H$, respectively. Since $\hat{\bcL} = \bcX \ast_{\boldsymbol{L}} \hat{\bcZ}$, we have
\begin{align*}
\bcV_{\hat{\bcL}}^H = \bcS_{\hat{\bcL}}^{\dagger} \ast_{\boldsymbol{L}} \bcU_{\hat{\bcL}}^H \ast_{\boldsymbol{L}} \bcX \ast_{\boldsymbol{L}} \hat{\bcU} \ast_{\boldsymbol{L}} \hat{\bcS} \ast_{\boldsymbol{L}} \hat{\bcV}^H,
\end{align*}
which implies that the horizontal slices, i.e., the tensor rows of $\hat{\bcV}^H$ span the tensor rows of $\bcV_{\hat{\bcL}}^H$. Notice that $\bcV_{\hat{\bcL}}^H$ is row orthonormal, there exists a row orthonormal tensor $\bcR$ such that $\hat{\bcV}^H = \bcR \ast_{\boldsymbol{L}} \bcV_{\hat{\bcL}}^H$. Then we have
\begin{align*}
\| \hat{\bcU} \ast \hat{\bcV}^H \|_{2,\infty} & = \max_j \| \hat{\bcU} \ast_{\boldsymbol{L}} \hat{\bcV}^H \ast_{\boldsymbol{L}} \mathring{\boldsymbol{\ce}}_j \|_F = \max_j \| \hat{\bcV}^H \ast_{\boldsymbol{L}} \mathring{\boldsymbol{\ce}}_j \|_F  \\
& = \max_j \| \bcR \ast_{\boldsymbol{L}} \bcV_{\hat{\bcL}}^H \ast_{\boldsymbol{L}} \mathring{\boldsymbol{\ce}}_j \|_F \leq \max_j \| \bcV_{\hat{\bcL}}^H \ast_{\boldsymbol{L}} \mathring{\boldsymbol{\ce}}_j \|_F \leq \sqrt{\frac{\mu r}{n_2 \tau}}.
\end{align*}
Let $\bcH_1 = \bcP_{\mathbf{\Theta}_0} \bcP_{\bcV_0}^L (\bcH)$ and $\bcH_2 = \bcP_{\mathbf{\Theta}_0^{\perp}} \bcP_{\bcV_0}^L (\bcH) + \bcP_{\mathbf{\Theta}_0^{\perp}} \bcP_{\bcV_0^{\perp}}^L (\bcH) + \bcP_{\mathbf{\Theta}_0} \bcP_{\bcV_0^{\perp}}^L (\bcH)$. Note that $(\hat{\bcZ}, \hat{\bcE})$ satisfies $\bcX = \bcX \ast_{\boldsymbol{L}} \hat{\bcZ} + \hat{\bcE}$, $\bcP_{\bcV_0}^L (\hat{\bcZ}) = \hat{\bcZ}$ and $\bcP_{\mathbf{\Theta}_0} (\hat{\bcE}) = \hat{\bcE}$. Then by Theorem~\ref{thm:zecond}, we have $\bcP_{\hat{\bcU}} = \bcP_{\bcV_0}^L$ and $\hat{\mathbf{\Theta}} = \mathbf{\Theta}_0$. We can obtain that
\begin{align*}
& \| \bcZ_{\star} \|_{\ast} + \lambda \| \bcE_{\star} \|_{2,1} - \| \hat{\bcZ} \|_{\ast} - \lambda \| \hat{\bcE} \|_{2,1}  \\
\geq & \langle \hat{\bcU} \ast_{\boldsymbol{L}} \hat{\bcV}^H + \hat{\bcQ}, \bcZ_{\star} - \hat{\bcZ} \rangle + \lambda \langle \cB(\hat{\bcE}) + \hat{\bcC}, \bcE_{\star} - \hat{\bcE} \rangle  \\
= & \langle \hat{\bcU} \ast_{\boldsymbol{L}} \hat{\bcV}^H + \hat{\bcQ}, \bcH_2 \rangle - \lambda \langle \cB(\hat{\bcE}) + \hat{\bcC}, \bcX^H \ast_{\boldsymbol{L}} \bcH_2 \rangle  \\
= & \langle \hat{\bcU} \ast_{\boldsymbol{L}} \hat{\bcV}^H, \bcP_{\mathbf{\Theta}_0^{\perp}} (\bcH) \rangle + \langle \hat{\bcQ}, \bcP_{\bcV_0^{\perp}}^L (\bcH) \rangle - \lambda \langle \cB(\hat{\bcE}), \bcX \ast_{\boldsymbol{L}} \bcP_{\bcV_0^{\perp}}^L (\bcH) \rangle - \lambda \langle \hat{\bcC}, \bcX \ast_{\boldsymbol{L}} \bcP_{\mathbf{\Theta}_0^{\perp}} (\bcH) \rangle  \\
= & \langle \hat{\bcU} \ast_{\boldsymbol{L}} \hat{\bcV}^H, \bcP_{\mathbf{\Theta}_0^{\perp}} (\bcH) \rangle + \langle \hat{\bcQ}, \bcP_{\bcV_0^{\perp}}^L (\bcH) \rangle - \lambda \langle \bcX^H \ast_{\boldsymbol{L}} \cB(\hat{\bcE}), \bcP_{\bcV_0^{\perp}}^L (\bcH) \rangle - \lambda \langle \bcX^H \ast_{\boldsymbol{L}} \hat{\bcC}, \bcP_{\mathbf{\Theta}_0^{\perp}} (\bcH) \rangle.
\end{align*}

By the duality between the tensor nuclear norm and the tensor spectral norm, there exists a tensor $\bcQ_0$ such that $\langle \bcQ_0, \bcP_{\hat{\bcV}^{\perp}} \bcP_{\bcV_0^{\perp}}^L (\bcH) \rangle = \| \bcP_{\hat{\bcV}^{\perp}} \bcP_{\bcV_0^{\perp}}^L (\bcH) \|_{\ast}$ and $\| \bcQ_0 \| = 1$. We set $\hat{\bcQ} = \bcP_{\bcV_0^{\perp}}^L \bcP_{\hat{\bcV}^{\perp}} (\bcQ_0) \in \bcP_{\hat{\bcT}^{\perp}}$. Similarly, thanks to the duality between the $\ell_{2,1}$ norm and $\ell_{2,\infty}$ norm, we can pick $\hat{\bcC} \in \bcP_{\hat{\mathbf{\Theta}}^{\perp}}$ such that $\langle \bcX^H \ast_{\boldsymbol{L}} \hat{\bcC}, \bcP_{\mathbf{\Theta}_0^{\perp}} (\bcH) \rangle = - \| \bcP_{\mathbf{\Theta}_0^{\perp}} (\bcH) \|_{2,1}$. On the other hand, we can establish that
\begin{align*}
& |\langle \bcX^H \ast_{\boldsymbol{L}} \cB(\hat{\bcE}), \bcP_{\bcV_0^{\perp}}^L (\bcH) \rangle|  \\
= & |\langle \bcW - \lambda \bcF, \bcP_{\bcV_0^{\perp}}^L (\bcH) \rangle|  \\
\leq & |\langle \bcW, \bcP_{\bcV_0^{\perp}}^L (\bcH) \rangle| + \lambda |\langle \bcF, \bcP_{\bcV_0^{\perp}}^L (\bcH) \rangle|  \\
\leq & \frac{1}{2} \| \bcP_{\hat{\bcV}^{\perp}} \bcP_{\bcV_0^{\perp}}^L (\bcH) \|_{\ast} + \frac{\lambda}{2} \| \bcP_{\hat{\mathbf{\Theta}}^{\perp}} \bcP_{\bcV_0^{\perp}}^L (\bcH) \|_{2,1},
\end{align*}
where the last inequality holds because $\bcW \in \bcP_{\hat{\bcV}^{\perp}}$ and $\bcF \in \bcP_{\hat{\mathbf{\Theta}}^{\perp}}$. Hence, we have
\begin{align*}
& \| \bcZ_{\star} \|_{\ast} + \lambda \| \bcE_{\star} \|_{2,1} - \| \hat{\bcZ} \|_{\ast} - \lambda \| \hat{\bcE} \|_{2,1}  \\
\geq & \| \bcP_{\hat{\bcV}^{\perp}} \bcP_{\bcV_0^{\perp}}^L (\bcH) \|_{\ast} - \sqrt{\frac{\mu r}{n_2 \tau}} \| \bcP_{\mathbf{\Theta}_0^{\perp}} (\bcH) \|_{2,1} + \lambda \| \bcP_{\mathbf{\Theta}_0^{\perp}} (\bcH) \|_{2,1} - \frac{1}{2} \| \bcP_{\hat{\bcV}^{\perp}} \bcP_{\bcV_0^{\perp}}^L (\bcH) \|_{\ast} - \frac{\lambda}{2} \| \bcP_{\hat{\mathbf{\Theta}}^{\perp}} \bcP_{\bcV_0^{\perp}}^L (\bcH) \|_{2,1}   \\
= & \frac{1}{2} \| \bcP_{\hat{\bcV}^{\perp}} \bcP_{\bcV_0^{\perp}}^L (\bcH) \|_{\ast} - \sqrt{\frac{\mu r}{n_2 \tau}} \| \bcP_{\mathbf{\Theta}_0^{\perp}} (\bcH) \|_{2,1} + \lambda \| \bcP_{\mathbf{\Theta}_0^{\perp}} (\bcH) \|_{2,1} - \frac{\lambda}{2} \| \bcP_{\hat{\mathbf{\Theta}}^{\perp}} \bcP_{\bcV_0^{\perp}}^L (\bcH) \|_{2,1}  \\
\geq & \frac{1}{2} \| \bcP_{\hat{\bcV}^{\perp}} \bcP_{\bcV_0^{\perp}}^L (\bcH) \|_{\ast} + \Big( \frac{\lambda}{2} - \sqrt{\frac{\mu r}{n_2 \tau}} \Big) \| \bcP_{\mathbf{\Theta}_0^{\perp}} (\bcH) \|_{2,1}.
\end{align*}

Note that $(\bcZ_{\star}, \bcE_{\star}) = (\bcV_0 \ast_{\boldsymbol{L}} \bcV_0^H + \bcH, \bcE_0 - \bcX \ast_{\boldsymbol{L}} \bcH)$ is an arbitrary optimal solution, the above inequality shows $\| \bcP_{\hat{\bcV}^{\perp}} \bcP_{\bcV_0^{\perp}}^L (\bcH) \|_{\ast} = \| \bcP_{\hat{\mathbf{\Theta}}^{\perp}} \bcP_{\bcV_0^{\perp}}^L (\bcH) \|_{2,1} = \| \bcP_{\mathbf{\Theta}_0^{\perp}} (\bcH) \|_{2,1} = 0$, i.e., $\bcP_{\bcV_0^{\perp}}^L (\bcH) \in \bcP_{\hat{\bcV}} \cap \bcP_{\hat{\mathbf{\Theta}}}$ and $\bcH \in \bcP_{\mathbf{\Theta}_0}$. By the assumption that $\| \bcP_{\hat{\mathbf{\Theta}}} \bcP_{\hat{\bcV}} \| < 1$, we have $\bcP_{\bcV_0^{\perp}}^L (\bcH) = \mathbf{0}$, and so $\bcH \in \bcP_{\bcV_0}^L$. Thus, the solution $(\bcZ_{\star}, \bcE_{\star})$ also satisfies $\bcP_{\bcV_0}^L (\bcZ_{\star}) = \bcZ_{\star}$ and $\bcP_{\mathbf{\Theta}_0} (\bcE_{\star}) = \bcE_{\star}$. It can be concluded that the solution $(\bcZ_{\star}, \bcE_{\star})$ also exactly recovers the row space of $\bcL_0$ and the column support of $\bcE_0$.
\end{proof}

Lemma~\ref{lem:condortlrr} implies that if we can find a dual certificate $\bcW$ obeying
\begin{equation}    \label{eqn:dualcert}
\begin{split}
& (a) ~ \bcW \in \bcP_{\hat{\bcV}^{\perp}},  \\
& (b) ~ \bcP_{\hat{\mathbf{\Theta}}} (\bcW) = \lambda \bcX^H \ast_{\boldsymbol{L}} \cB(\hat{\bcE}),  \\
& (c) ~ \| \bcW \| \leq 1/2,  \\
& (d) ~ \| \bcP_{\hat{\mathbf{\Theta}}^{\perp}} (\bcW) \|_{2,\infty} \leq \lambda/2,
\end{split}
\end{equation}
then we can exactly recover the row space of $\bcL_0$ and the column support of outlier tensor $\bcE_0$.

\subsection{Dual Certification via the Least Squares}
\label{ssec:dualcert}

Before we construct the dual certificate $\bcW$, we first give some key lemmas, which will be proved in Section~\ref{ssec:lemmaproof}.

\begin{lemma}    \label{lem:projerrorbound}
Assume $\hat{\mathbf{\Theta}} \sim \text{Ber} (\rho)$. Then with probability at least $1 - n_{(1)}^{-10}$,
\begin{align*}
\| \bcP_{\hat{\bcV}} - \frac{1}{\rho} \bcP_{\hat{\bcV}} \bcP_{\hat{\mathbf{\Theta}}} \bcP_{\hat{\bcV}} \| \leq \epsilon,
\end{align*}
provided that $\rho \geq c_2 \mu r n_3 \log(n_{(1)}) / (\epsilon^2 n_2)$ for some numerical constant $c_2$.
\end{lemma}

\begin{corollary}    \label{coro:pthetapv}
Assume $\hat{\mathbf{\Theta}} \sim \text{Ber} (\rho)$. Then with probability at least $1 - n_{(1)}^{-10}$,
\begin{align*}
\| \bcP_{\hat{\mathbf{\Theta}}} \bcP_{\hat{\bcV}} \|^2 \leq (1 - \rho) \epsilon + \rho < 1,
\end{align*}
provided that $1 - \rho \geq c_2 \mu r n_3 \log(n_{(1)}) / (\epsilon^2 n_2)$ for some numerical constant $c_2$.
\end{corollary}

Now we construct the dual certificate $\bcW$ and verify its validity.
\begin{lemma}    \label{lem:dualcert}
Suppose that $\hat{\mathbf{\Theta}} \sim \text{Ber} (\rho)$ and the assumptions in Theorem~\ref{thm:condortlrr} are satisfied. Then with high probability,
\begin{align*}
\bcW = \lambda \bcP_{\hat{\bcV}^{\perp}} \sum_{k=0}^{+ \infty} (\bcP_{\hat{\mathbf{\Theta}}} \bcP_{\hat{\bcV}} \bcP_{\hat{\mathbf{\Theta}}})^k (\bcX^H \ast_{\boldsymbol{L}} \cB(\hat{\bcE}))
\end{align*}
obeys the dual conditions \eqref{eqn:dualcert}.
\end{lemma}

\begin{proof}
Note that by Corollary~\ref{coro:pthetapv}, we have $\| \bcP_{\hat{\mathbf{\Theta}}} \bcP_{\hat{\bcV}} \bcP_{\hat{\mathbf{\Theta}}} \| = \| \bcP_{\hat{\mathbf{\Theta}}} \bcP_{\hat{\bcV}} \|^2 < 1$. Thus, $\bcW$ is well defined. Now we verify the four conditions in \eqref{eqn:dualcert} in turn.

\textbf{Proof of (a):} It is easy to verify that $\bcW \in \bcP_{\hat{\bcV}^{\perp}}$.

\textbf{Proof of (b):} By the construction of $\bcW$, we have
\begin{align*}
\bcP_{\hat{\mathbf{\Theta}}} (\bcW) & = \lambda \bcP_{\hat{\mathbf{\Theta}}} \bcP_{\hat{\bcV}^{\perp}} \sum_{k=0}^{+ \infty} (\bcP_{\hat{\mathbf{\Theta}}} \bcP_{\hat{\bcV}} \bcP_{\hat{\mathbf{\Theta}}})^k (\bcX^H \ast_{\boldsymbol{L}} \cB(\hat{\bcE}))  \\
& = \lambda \bcP_{\hat{\mathbf{\Theta}}} (\bcI_{n_2} - \bcP_{\hat{\bcV}}) \sum_{k=0}^{+ \infty} (\bcP_{\hat{\mathbf{\Theta}}} \bcP_{\hat{\bcV}} \bcP_{\hat{\mathbf{\Theta}}})^k (\bcX^H \ast_{\boldsymbol{L}} \cB(\hat{\bcE}))  \\
& = \lambda \bcP_{\hat{\mathbf{\Theta}}} (\bcI_{n_2} - \bcP_{\hat{\mathbf{\Theta}}} \bcP_{\hat{\bcV}} \bcP_{\hat{\mathbf{\Theta}}}) \sum_{k=0}^{+ \infty} (\bcP_{\hat{\mathbf{\Theta}}} \bcP_{\hat{\bcV}} \bcP_{\hat{\mathbf{\Theta}}})^k (\bcX^H \ast_{\boldsymbol{L}} \cB(\hat{\bcE}))  \\
& = \lambda \bcP_{\hat{\mathbf{\Theta}}} \Big( \sum_{k=0}^{+ \infty} (\bcP_{\hat{\mathbf{\Theta}}} \bcP_{\hat{\bcV}} \bcP_{\hat{\mathbf{\Theta}}})^k - \sum_{k=1}^{+ \infty} (\bcP_{\hat{\mathbf{\Theta}}} \bcP_{\hat{\bcV}} \bcP_{\hat{\mathbf{\Theta}}})^k \Big) (\bcX^H \ast_{\boldsymbol{L}} \cB(\hat{\bcE}))  \\
& = \lambda \bcP_{\hat{\mathbf{\Theta}}} (\bcX^H \ast_{\boldsymbol{L}} \cB(\hat{\bcE}))  \\
& = \lambda \bcX^H \ast_{\boldsymbol{L}} \cB(\hat{\bcE}).
\end{align*}
Thus, $\bcW$ obeys the condition \eqref{eqn:dualcert}(b).

\textbf{Proof of (c):} Let $\bcG = \sum_{k=0}^{+ \infty} (\bcP_{\hat{\mathbf{\Theta}}} \bcP_{\hat{\bcV}} \bcP_{\hat{\mathbf{\Theta}}})^k$. Since $\| \cB(\hat{\bcE}) \| \leq \sqrt{\log (n_2)} / \theta$, we have
\begin{align*}
\| \bcW \| \leq \lambda \| \bcP_{\hat{\bcV}^{\perp}} \| \| \bcG \| \| \bcX \| \| \cB(\hat{\bcE}) \| \leq \frac{\lambda}{1 - \sigma^2} \| \bcX \| \frac{\sqrt{\log (n_2)}}{\theta},
\end{align*}
where $\sigma = \sqrt{\rho + \epsilon (1 - \rho)}$. If $\lambda \leq \frac{\theta (1 - \sigma^2)}{2 \| \bcX \| \sqrt{\log (n_2)}}$, then $\| \bcW \| \leq 1/2$. Note that in Lemma~\ref{lem:condortlrr}, we require that $\lambda > 2 \sqrt{\frac{\mu r}{n_2 \tau}}$. Thus, $\lambda \in \Big( 2 \sqrt{\frac{\mu r}{n_2 \tau}}, \frac{\theta (1 - \sigma^2)}{2 \| \bcX \| \sqrt{\log (n_2)}} \Big]$.

\textbf{Proof of (d):} We can rewrite $\bcP_{\hat{\mathbf{\Theta}}^{\perp}} (\bcW)$ as
\begin{align*}
\bcP_{\hat{\mathbf{\Theta}}^{\perp}} (\bcW) & = \lambda \bcP_{\hat{\mathbf{\Theta}}^{\perp}} \bcP_{\hat{\bcV}^{\perp}} \bcG (\bcX^H \ast_{\boldsymbol{L}} \cB(\hat{\bcE}))  \\
& = \lambda \bcP_{\hat{\mathbf{\Theta}}^{\perp}} (\bcI_{n_2} - \bcP_{\hat{\bcV}}) \bcG (\bcX^H \ast_{\boldsymbol{L}} \cB(\hat{\bcE}))  \\
& = - \lambda \bcP_{\hat{\mathbf{\Theta}}^{\perp}} \bcP_{\hat{\bcV}} \bcG (\bcX^H \ast_{\boldsymbol{L}} \cB(\hat{\bcE})).
\end{align*}
Note that $\bar{\boldsymbol{\ce}}_{ijk}$ is the unit tensor with the $(i,j,k)$-th entry equaling to 1 and the rest equaling to 0. Thus, the $(i,j)$-th tube of $\bar{\boldsymbol{\ce}}_{ijk}$ is the only nonzero tube and it equals $L(\dot{\boldsymbol{\ce}}_k)$, which further suggests that the ($i,j$)-th tube of $L(\bar{\boldsymbol{\ce}}_{ijk})$ equals $L(L(\dot{\boldsymbol{\ce}}_k))$. We have
\begin{align*}
& \max_j \sum_{i,k} \| \bcP_{\hat{\bcV}} (\bar{\boldsymbol{\ce}}_{ijk}) \|_F^2  \\
= & \max_j \sum_{i,k} \| L(\mathring{\boldsymbol{\ce}}_i \ast_{\boldsymbol{L}} \dot{\boldsymbol{\ce}}_k \ast_{\boldsymbol{L}} \mathring{\boldsymbol{\ce}}_j^H) \ast_{\boldsymbol{L}} \hat{\bcV} \ast_{\boldsymbol{L}} \hat{\bcV}^H \|_F^2  \\
= & \max_j \sum_{i,k} \| L(\mathring{\boldsymbol{\ce}}_i \ast_{\boldsymbol{L}} \dot{\boldsymbol{\ce}}_k \ast_{\boldsymbol{L}} \mathring{\boldsymbol{\ce}}_j^H) \ast_{\boldsymbol{L}} \hat{\bcV} \|_F^2  \\
= & \max_j \sum_{i,k} \frac{1}{\tau} \| L(\mathring{\boldsymbol{\ce}}_i) \odot L(L(\dot{\boldsymbol{\ce}}_k)) \odot L(\mathring{\boldsymbol{\ce}}_j^H) \odot L(\hat{\bcV}) \|_F^2  \\
= & \max_j \sum_{i,k} \frac{1}{\tau} \| L(L(\dot{\boldsymbol{\ce}}_k)) \odot L(\mathring{\boldsymbol{\ce}}_j^H) \odot L(\hat{\bcV}) \|_F^2.
\end{align*}
Define $\bcQ = \mathring{\boldsymbol{\ce}}_j^H \ast_{\boldsymbol{L}} \hat{\bcV}$. Then $\| L(L(\dot{\boldsymbol{\ce}}_k)) \odot \xoverline{\bcQ} \|_F^2 = \sum_q |\boldsymbol{L}_{q,k}|^2 \| \xoverline{\bcQ}(:,:,q) \|_F^2 \leq \sum_q (\max_q |\boldsymbol{L}_{q,k}|^2) \| \xoverline{\bcQ}(:,:,q) \|_F^2 \leq \tau \sum_q \| \xoverline{\bcQ}(:,:,q) \|_F^2 = \tau \| \xoverline{\bcQ} \|_F^2 = \tau^2 \| \bcQ \|_F^2$. Hence,
\begin{align*}
\max_j \sum_{i,k} \| \bcP_{\hat{\bcV}} (\bar{\boldsymbol{\ce}}_{ijk}) \|_F^2 \leq \tau \max_j \sum_{i,k} \| \mathring{\boldsymbol{\ce}}_j^H \ast_{\boldsymbol{L}} \hat{\bcV} \|_F^2 = \frac{\mu r n_1 n_3}{n_2}.
\end{align*}
Then, we can obtain
\begin{align*}
& \| \bcP_{\hat{\bcV}} \bcG (\bcX^H \ast_{\boldsymbol{L}} \cB(\hat{\bcE})) \|_{2,\infty}^2  \\
= & \max_b \sum_{i,k} \Big\langle \bcP_{\hat{\bcV}} \bcG (\bcX^H \ast_{\boldsymbol{L}} \cB(\hat{\bcE})), \bar{\boldsymbol{\ce}}_{ibk} \Big\rangle^2  \\
= & \max_b \sum_{i,k} \Big\langle \bcX^H \ast_{\boldsymbol{L}} \cB(\hat{\bcE}), \bcG \bcP_{\hat{\bcV}} (\bar{\boldsymbol{\ce}}_{ibk}) \Big\rangle^2  \\
= & \max_b \sum_{i,k} \sum_j \Big\langle \bcX^H \ast_{\boldsymbol{L}} \cB(\hat{\bcE}) \ast_{\boldsymbol{L}} \mathring{\boldsymbol{\ce}}_j, \bcG \bcP_{\hat{\bcV}} (\bar{\boldsymbol{\ce}}_{ibk}) \ast_{\boldsymbol{L}} \mathring{\boldsymbol{\ce}}_j \Big\rangle^2  \\
\leq & \max_b \sum_{i,j,k} \| \bcX^H \ast_{\boldsymbol{L}} \cB(\hat{\bcE}) \ast_{\boldsymbol{L}} \mathring{\boldsymbol{\ce}}_j \|_F^2 \| \bcG \bcP_{\hat{\bcV}} (\bar{\boldsymbol{\ce}}_{ibk}) \ast_{\boldsymbol{L}} \mathring{\boldsymbol{\ce}}_j \|_F^2  \\
\leq & \max_b \sum_{i,j,k} \| \bcX^H \ast_{\boldsymbol{L}} \cB(\hat{\bcE}) \|_{2,\infty}^2 \| \bcG \bcP_{\hat{\bcV}} (\bar{\boldsymbol{\ce}}_{ibk}) \ast_{\boldsymbol{L}} \mathring{\boldsymbol{\ce}}_j \|_F^2  \\
= & \max_b \sum_{i,k} \| \bcX^H \ast_{\boldsymbol{L}} \cB(\hat{\bcE}) \|_{2,\infty}^2 \| \bcG \bcP_{\hat{\bcV}} (\bar{\boldsymbol{\ce}}_{ibk}) \|_F^2  \\
\leq & \max_b \sum_{i,k} \| \bcX \|^2 \| \bcG \|^2 \| \bcP_{\hat{\bcV}} (\bar{\boldsymbol{\ce}}_{ibk}) \|_F^2  \\
\leq & \frac{\mu r n_1 n_3}{n_2} (\frac{1}{1 - \sigma^2})^2 \| \bcX \|^2  \\
\leq & \frac{1}{4},
\end{align*}
where the last inequality holds because we require $r \leq n_2 (1 - \sigma^2)^2 / (4 \mu n_1 n_3 \| \bcX \|^2)$. In the meanwhile, when proving Corollary~\ref{coro:pthetapv}, we require $r \leq (1 - \rho) \epsilon^2 n_2 / (c_2 \mu n_3 \log(n_{(1)}))$. Thus, we can further obtain
\begin{align*}
\| \bcP_{\hat{\mathbf{\Theta}}^{\perp}} (\bcW) \|_{2,\infty} & = \lambda \| \bcP_{\hat{\mathbf{\Theta}}^{\perp}} \bcP_{\hat{\bcV}} \bcG (\bcX^H \ast_{\boldsymbol{L}} \cB(\hat{\bcE})) \|_{2,\infty}  \\
& \leq \lambda \| \bcP_{\hat{\bcV}} \bcG (\bcX^H \ast_{\boldsymbol{L}} \cB(\hat{\bcE})) \|_{2,\infty}  \\
& = \frac{\lambda}{2}.
\end{align*}
So $\bcW$ obeys the condition \eqref{eqn:dualcert}(d).

\textbf{Checking the ranges of $\lambda$ and $r$:} When we prove Corollary~\ref{coro:pthetapv} and bound $\| \bcP_{\hat{\mathbf{\Theta}}^{\perp}} (\bcW) \|_{2,\infty}$, we require
\begin{align*}
r \leq \min \Big( \frac{n_2 (1 - \sigma^2)^2}{4 \mu n_1 n_3 \| \bcX \|^2}, \frac{(1 - \rho) \epsilon^2 n_2}{c_2 \mu n_3 \log(n_{(1)})} \Big),
\end{align*}
Thus we have $r \leq \frac{\rho_r n_2}{\mu n_1 n_3 \| \bcX \|^2}$, where $\rho_r$ is a constant. On the other hand, we require
\begin{align*}
\lambda \in \Big( 2 \sqrt{\frac{\mu r}{n_2 \tau}}, \frac{\theta (1 - \sigma^2)}{2 \| \bcX \| \sqrt{\log (n_2)}} \Big].
\end{align*}
Let $\rho \leq 1/2 - \epsilon$, we then have $2 (1 - \sigma^2) \geq 1$. Accordingly, we can obtain
\begin{align*}
\lambda \in \Big( \min \Big( \frac{1}{\sqrt{\tau n_1 n_3} \| \bcX \|}, \frac{2 \epsilon}{\sqrt{c_2 \tau n_3 \log(n_{(1)})}} \Big) , \frac{\theta}{4 \sqrt{\log(n_2)} \| \bcX \|} \Big].
\end{align*}
So we can set $\lambda = \frac{\theta}{4 \sqrt{\log(n_{(1)})} \| \bcX \|}$. The proof is completed.
\end{proof}

\subsection{Proofs of Some Lemmas}
\label{ssec:lemmaproof}

We first introduce Lemma~\ref{lem:Bernstein} which will be used in the proof of Lemma~\ref{lem:projerrorbound}.

\begin{lemma} \cite{Tropp.FoCM2012}    \label{lem:Bernstein}
Consider a finite sequence $\{\bX_i \in \R^{d_1 \times d_2} \}$ of independent, random matrices that satisfy the assumption that $\mathbb{E}[\bX_i] = \mathbf{0}$ and $\| \bX_i \| \leq \nu$. Let $\omega = \max(\| \sum_i \mathbb{E}[\bX_i \bX_i^H] \|, \| \sum_i \mathbb{E}[\bX_i^H \bX_i] \|)$. Then, for any $t \geq 0$, we have
\begin{align*}
\mathbb{P} \Big[ \| \sum_i \bX_i \| \geq t \Big] & \leq (d_1 + d_2) \exp \Big( - \frac{t^2}{2 \omega + \frac{2}{3} \nu t} \Big)  \\
& \leq (d_1 + d_2) \exp \Big( - \frac{3 t^2}{8 \omega} \Big), \quad \text{for} ~ t \leq \omega / \nu.
\end{align*}
\end{lemma}

\subsubsection{Proof of Lemma~\ref{lem:projerrorbound}}

\begin{proof}
Define a set $\phi = \{ (\bcZ_1, \bcZ_2) ~ | ~ \| \bcZ_1 \|_F \leq 1, \bcZ_2 = \pm \bcZ_1 \}$. Since $(\rho^{-1} \bcP_{\hat{\bcV}} \bcP_{\hat{\mathbf{\Theta}}} \bcP_{\hat{\bcV}} - \bcP_{\hat{\bcV}})$ is a self-adjoint operator, we can use the variational form of the operator norm to compute its operator norm as follows:
\begin{align*}
& \| \frac{1}{\rho} \bcP_{\hat{\bcV}} \bcP_{\hat{\mathbf{\Theta}}} \bcP_{\hat{\bcV}} - \bcP_{\hat{\bcV}} \|  \\
= & \sup_{\phi} \sum_{i,j,k} \Big( \frac{\delta_j}{\rho} - 1 \Big) \langle \bcP_{\hat{\bcV}} (\bcZ_1), \bar{\boldsymbol{\ce}}_{ijk} \rangle \langle \bcP_{\hat{\bcV}} (\bcZ_2), \bar{\boldsymbol{\ce}}_{ijk} \rangle  \\
= & \sup_{\phi} \sum_{i,j,k} \frac{\delta_j - \rho}{\rho \tau^2} \langle \mathtt{bdiag} \Big( \overline{\bcP_{\hat{\bcV}} (\bar{\boldsymbol{\ce}}_{ijk})} \Big), \widebar{\bZ}_1 \rangle \langle \mathtt{bdiag} \Big( \overline{\bcP_{\hat{\bcV}} (\bar{\boldsymbol{\ce}}_{ijk})} \Big), \widebar{\bZ}_2 \rangle  \\
= & \Big\| \sum_{i,j,k} \frac{\delta_j - \rho}{\rho \tau} \mathtt{bdiag} \Big( \overline{\bcP_{\hat{\bcV}} (\bar{\boldsymbol{\ce}}_{ijk})} \Big) \cdot \mathtt{bdiag} \Big( \overline{\bcP_{\hat{\bcV}} (\bar{\boldsymbol{\ce}}_{ijk})} \Big)  \Big\|  \\
= & \Big\| \sum_j \frac{\delta_j - \rho}{\rho \tau} \sum_{i,k} \mathtt{bdiag} \Big( \overline{\bcP_{\hat{\bcV}} (\bar{\boldsymbol{\ce}}_{ijk})} \Big) \cdot \mathtt{bdiag} \Big( \overline{\bcP_{\hat{\bcV}} (\bar{\boldsymbol{\ce}}_{ijk})} \Big)  \Big\|,
\end{align*}
where $\delta_j$ obeys i.i.d. Bernoulli distribution $\text{Ber} (\rho)$. Define a set $\psi = \{ (\widebar{\bD}_1, \widebar{\bD}_2) ~ | ~ \| \widebar{\bD}_1 \|_F \leq 1, \widebar{\bD}_2 = \pm \widebar{\bD}_1 \}$. Note that $\mathtt{bdiag} \Big( \overline{\bcP_{\hat{\bcV}} (\bar{\boldsymbol{\ce}}_{ijk})} \Big)$ is a diagonal matrix and thus $\widebar{\bD}_1$ and $\widebar{\bD}_2$ are also diagonal matrices and there exist two tensors $\bcD_1$ and $\bcD_2$ such that $\widebar{\bD}_1 = \mathtt{bdiag} (L(\bcD_1))$ and $\widebar{\bD}_2 = \mathtt{bdiag} (L(\bcD_2))$. For brevity, we further define
\begin{align*}
\tilde{\bH}_j = \frac{\delta_j - \rho}{\rho \tau} \sum_{i,k} \mathtt{bdiag} \Big( \overline{\bcP_{\hat{\bcV}} (\bar{\boldsymbol{\ce}}_{ijk})} \Big) \cdot \mathtt{bdiag} \Big( \overline{\bcP_{\hat{\bcV}} (\bar{\boldsymbol{\ce}}_{ijk})} \Big),
\end{align*}
and have
\begin{align*}
\| \frac{1}{\rho} \bcP_{\hat{\bcV}} \bcP_{\hat{\mathbf{\Theta}}} \bcP_{\hat{\bcV}} - \bcP_{\hat{\bcV}} \| = \|  \sum_j \tilde{\bH}_j \|.
\end{align*}
By the above definitions, $\tilde{\bH}_j$ is self-adjoint and $\mathbb{E} [\tilde{\bH}_j] = \mathbf{0}$. To prove the result by the non-commutative Bernstein inequality, we need to bound $\| \tilde{\bH}_j \|$ and $\| \sum_j \mathbb{E} [\tilde{\bH}_j \tilde{\bH}_j^H] \|$. Firstly, we give an useful inequality:
\begin{align*}
& \| \bcP_{\hat{\bcV}} (\bcA) \|_{2,\infty}^2  \\
= & \max_j \sum_{i,k} \langle \bcA, \bcP_{\hat{\bcV}} (\bar{\boldsymbol{\ce}}_{ijk}) \rangle^2  \\
= & \max_j \sum_{i,k} \langle \bcA, L(\mathring{\boldsymbol{\ce}}_i \ast_{\boldsymbol{L}} \dot{\boldsymbol{\ce}}_k \ast_{\boldsymbol{L}} \mathring{\boldsymbol{\ce}}_j^H) \ast_{\boldsymbol{L}} \hat{\bcV} \ast_{\boldsymbol{L}} \hat{\bcV}^H \rangle^2  \\
= & \max_j \sum_{i,k} \frac{1}{\tau^2} \langle \xoverline{\bcA}, L(\mathring{\boldsymbol{\ce}}_i) \odot L(L(\dot{\boldsymbol{\ce}}_k)) \odot L(\mathring{\boldsymbol{\ce}}_j^H) \odot L(\hat{\bcV}) \odot L(\hat{\bcV}^H) \rangle^2  \\
= & \max_j \sum_{i,k} \frac{1}{\tau^2} \langle L(L(\dot{\boldsymbol{\ce}}_k^H)) \odot L(\mathring{\boldsymbol{\ce}}_i^H) \odot \xoverline{\bcA}, L(\mathring{\boldsymbol{\ce}}_j^H) \odot L(\hat{\bcV}) \odot L(\hat{\bcV}^H) \rangle^2  \\
\leq & \max_j \sum_{i,k} \frac{1}{\tau^2} \| L(L(\dot{\boldsymbol{\ce}}_k^H)) \odot L(\mathring{\boldsymbol{\ce}}_i^H) \odot \xoverline{\bcA} \|_F^2 \| L(\mathring{\boldsymbol{\ce}}_j^H) \odot L(\hat{\bcV}) \odot L(\hat{\bcV}^H) \|_F^2  \\
= & \max_j \sum_{i,k} \frac{1}{\tau^2} \| L(L(\dot{\boldsymbol{\ce}}_k^H)) \odot \xoverline{\bcA}(i,:,:) \|_F^2 \| L(\mathring{\boldsymbol{\ce}}_j^H) \odot L(\hat{\bcV}) \|_F^2  \\
= & \max_j \sum_{i,k} \frac{1}{\tau} \| L(L(\dot{\boldsymbol{\ce}}_k^H)) \odot \xoverline{\bcA}(i,:,:) \|_F^2 \| \mathring{\boldsymbol{\ce}}_j^H \ast_{\boldsymbol{L}} \hat{\bcV} \|_F^2  \\
\leq & \max_j \sum_{i,k} \| \xoverline{\bcA}(i,:,:) \|_F^2 \| \mathring{\boldsymbol{\ce}}_j^H \ast_{\boldsymbol{L}} \hat{\bcV} \|_F^2  \\
= & \max_j n_3 \| \xoverline{\bcA} \|_F^2 \| \mathring{\boldsymbol{\ce}}_j^H \ast_{\boldsymbol{L}} \hat{\bcV} \|_F^2  \\
= & n_3 \tau \| \bcA \|_F^2 \max_j \| \mathring{\boldsymbol{\ce}}_j^H \ast_{\boldsymbol{L}} \hat{\bcV} \|_F^2  \\
\leq & \frac{\mu r n_3}{n_2} \| \bcA \|_F^2.
\end{align*}
Then we can bound $\| \tilde{\bH}_j \|$ in the following:
\begin{align*}
& \| \tilde{\bH}_j \|  \\
= & \sup_{\psi} \frac{\delta_j - \rho}{\rho \tau} \sum_{i,k} \langle \mathtt{bdiag} \Big( \overline{\bcP_{\hat{\bcV}} (\bar{\boldsymbol{\ce}}_{ijk})} \Big), \widebar{\bD}_1 \rangle \langle \mathtt{bdiag} \Big( \overline{\bcP_{\hat{\bcV}} (\bar{\boldsymbol{\ce}}_{ijk})} \Big), \widebar{\bD}_2 \rangle  \\
\leq & \sup_{\psi} \frac{\tau}{\rho} |\delta_j - \rho| \sum_{i,k} | \langle \bcD_1, \bcP_{\hat{\bcV}} (\bar{\boldsymbol{\ce}}_{ijk}) \rangle | | \langle \bcD_2, \bcP_{\hat{\bcV}} (\bar{\boldsymbol{\ce}}_{ijk}) \rangle |  \\
= & \sup_{\psi} \frac{\tau}{\rho} |\delta_j - \rho| \sum_{i,k} | \langle \bcP_{\hat{\bcV}} (\bcD_1), \bar{\boldsymbol{\ce}}_{ijk} \rangle |^2  \\
\leq & \sup_{\psi} \frac{\tau}{\rho} \sum_{i,k} \langle \bcP_{\hat{\bcV}} (\bcD_1), \bar{\boldsymbol{\ce}}_{ijk} \rangle^2  \\
\leq & \sup_{\psi} \frac{\tau}{\rho} \| \bcP_{\hat{\bcV}} (\bcD_1) \|_{2,\infty}^2  \\
\leq & \sup_{\psi} \frac{\tau}{\rho} \frac{\mu r n_3}{n_2} \| \bcD_1 \|_F^2  \\
\leq & \frac{\mu r n_3}{\rho n_2} := \nu.
\end{align*}

Note that $\mathbb{E}[(\rho^{-1} \delta_j - 1)^2] = \rho^{-1} (1 - \rho) \leq \rho^{-1}$. We further obtain
\begin{align*}
& \| \sum_j \mathbb{E} [\tilde{\bH}_j \tilde{\bH}_j^H] \|  \\
= & \sup_{\psi} \sum_j \mathbb{E}[(\rho^{-1} \delta_j - 1)^2] \tau^2 \Big( \sum_{i,k} \langle \bcD_1, \bcP_{\hat{\bcV}} (\bar{\boldsymbol{\ce}}_{ijk}) \rangle \langle \bcD_2, \bcP_{\hat{\bcV}} (\bar{\boldsymbol{\ce}}_{ijk}) \rangle \Big)^2  \\
\leq & \frac{\tau^2}{\rho} \sup_{\psi} \sum_j \Big( \sum_{i,k} | \langle \bcD_1, \bcP_{\hat{\bcV}} (\bar{\boldsymbol{\ce}}_{ijk}) \rangle | | \langle \bcD_2, \bcP_{\hat{\bcV}} (\bar{\boldsymbol{\ce}}_{ijk}) \rangle | \Big)^2  \\
= & \frac{\tau^2}{\rho} \sup_{\psi} \sum_j \Big( \sum_{i,k} \langle \bcD_1, \bcP_{\hat{\bcV}} (\bar{\boldsymbol{\ce}}_{ijk}) \rangle^2 \Big)^2  \\
= & \frac{\tau^2}{\rho} \sup_{\psi} \sum_j \| \bcP_{\hat{\bcV}} (\bcD_1) \|_{2,\infty}^2 \Big( \sum_{i,k} \langle \bcD_1, \bcP_{\hat{\bcV}} (\bar{\boldsymbol{\ce}}_{ijk}) \rangle^2 \Big)  \\
= & \frac{\tau^2}{\rho} \frac{\mu r n_3}{\tau n_2} \sup_{\psi} \sum_j \Big( \sum_{i,k} \langle \bcD_1, \bcP_{\hat{\bcV}} (\bar{\boldsymbol{\ce}}_{ijk}) \rangle^2 \Big)  \\
= & \frac{\tau \mu r n_3}{\rho n_2} \sup_{\psi} \sum_j \Big( \sum_{i,k} \langle \bcD_1, \bcP_{\hat{\bcV}} (\bar{\boldsymbol{\ce}}_{ijk}) \rangle^2 \Big)  \\
\leq & \frac{\tau \mu r n_3}{\rho n_2} \sup_{\psi} \| \bcD_1 \|_F^2  \\
\leq & \frac{\mu r n_3}{\rho n_2} := \omega.
\end{align*}
Since $\epsilon$ is small, $\omega / \nu = 1 > \epsilon$. By Lemma~\ref{lem:Bernstein}, we have
\begin{align*}
& \mathbb{P} \Big( \| \bcP_{\hat{\bcV}} - \frac{1}{\rho} \bcP_{\hat{\bcV}} \bcP_{\hat{\mathbf{\Theta}}} \bcP_{\hat{\bcV}} \| > \epsilon \Big)  \\
= & \mathbb{P} \Big( \| \sum_j \mathbb{E} [\tilde{\bH}_j] \| > \epsilon \Big)  \\
\leq & (n_1 + n_2) n_3 \exp \Big( - \frac{3 \epsilon^2}{8 \omega} \Big)  \\
\leq & (n_1 + n_2) n_3 \exp \Big( - \frac{3 \epsilon^2 \rho n_2}{8 \mu r n_3} \Big).
\end{align*}
Let $\rho \geq c_2 \mu r n_3 \log(n_{(1)}) / (\epsilon^2 n_2)$. Then the following inequality holds.
\begin{align*}
& \mathbb{P} \Big( \| \bcP_{\hat{\bcV}} - \frac{1}{\rho} \bcP_{\hat{\bcV}} \bcP_{\hat{\mathbf{\Theta}}} \bcP_{\hat{\bcV}} \| \leq \epsilon \Big)  \\
= & 1 - \Big( \| \bcP_{\hat{\bcV}} - \frac{1}{\rho} \bcP_{\hat{\bcV}} \bcP_{\hat{\mathbf{\Theta}}} \bcP_{\hat{\bcV}} \| > \epsilon \Big)  \\
\geq & 1 - (n_1 + n_2) n_3 \exp \Big( - \frac{3 \epsilon^2 \rho n_2}{8 \mu r n_3} \Big)  \\
\geq & 1 - 2 n_3 (n_{(1)})^{- \frac{3 c_2}{8} + 1}.
\end{align*}
By choosing an appropriate $c_2$, we have $\mathbb{P} (\| \bcP_{\hat{\bcV}} - \frac{1}{\rho} \bcP_{\hat{\bcV}} \bcP_{\hat{\mathbf{\Theta}}} \bcP_{\hat{\bcV}} \| \leq \epsilon) \geq 1 - n_{(1)}^{-10}$. The proof is completed.
\end{proof}

\subsubsection{Proof of Corollary~\ref{coro:pthetapv}}

\begin{proof}
From Lemma~\ref{lem:projerrorbound} we have
\begin{align*}
\| \bcP_{\hat{\bcV}} - \frac{1}{1 - \rho} \bcP_{\hat{\bcV}} \bcP_{\hat{\mathbf{\Theta}}} \bcP_{\hat{\bcV}} \| \leq \epsilon,
\end{align*}
provided that $1 - \rho \geq c_2 \mu r n_3 \log(n_{(1)}) / (\epsilon^2 n_2)$. Note that $\bcI = \bcP_{\hat{\mathbf{\Theta}}} + \bcP_{\hat{\mathbf{\Theta}}^{\perp}}$, we have
\begin{align*}
\| \bcP_{\hat{\bcV}} - \frac{1}{1 - \rho} \bcP_{\hat{\bcV}} \bcP_{\hat{\mathbf{\Theta}}^{\perp}} \bcP_{\hat{\bcV}} \| = \frac{1}{1 - \rho} \| \rho \bcP_{\hat{\bcV}} - \bcP_{\hat{\bcV}} \bcP_{\hat{\mathbf{\Theta}}} \bcP_{\hat{\bcV}} \|.
\end{align*}
Then, by the triangular inequality
\begin{align*}
\| \bcP_{\hat{\bcV}} \bcP_{\hat{\mathbf{\Theta}}} \bcP_{\hat{\bcV}} \| \leq \| \bcP_{\hat{\bcV}} \bcP_{\hat{\mathbf{\Theta}}} \bcP_{\hat{\bcV}} - \rho \bcP_{\hat{\bcV}} \| + \| \rho \bcP_{\hat{\bcV}} \| \leq (1 - \rho) \epsilon + \rho.
\end{align*}
The proof is completed by using $\| \bcP_{\hat{\mathbf{\Theta}}} \bcP_{\hat{\bcV}} \|^2 = \| \bcP_{\hat{\bcV}} \bcP_{\hat{\mathbf{\Theta}}} \bcP_{\hat{\bcV}} \|$.
\end{proof}

\section{ADDITIONAL EXPERIMENTS}
\label{sec:moreexperiment}

\subsection{Application to Outlier Detection}

Here, we provide results for the missing data experiments on the COIL20-MIRFLICKR-25k dataset in Table~\ref{tab:coil20-miss}.

\begin{table}[htbp]
\tiny
\centering
\caption{AUC of outlier detection and clustering results (ACC, NMI and PUR) on COIL20-MIRFLICKR-25k dataset for missing data experiments.}
\setlength{\tabcolsep}{0.1em}{\begin{tabular}{c||c|c|c|c||c|c|c|c||c|c|c|c||c|c|c|c}
\hline
\# Outlier & \multicolumn{8}{c||}{100} & \multicolumn{8}{c}{200}  \\
\hline
\multirow{2}{*}{Methods} & \multicolumn{4}{c||}{$\delta = 10\%$} & \multicolumn{4}{c||}{$\delta = 20\%$} & \multicolumn{4}{c||}{$\delta = 10\%$} & \multicolumn{4}{c}{$\delta = 20\%$}  \\
\cline{2-17}
 & AUC & ACC & NMI & PUR & AUC & ACC & NMI & PUR & AUC & ACC & NMI & PUR & AUC & ACC & NMI & PUR  \\
\hline
TNN+TRPCA-DFT & 0.7686 & 0.5661 & 0.7055 & 0.6422 & 0.7676 & 0.5621 & 0.7021 & 0.6407 & 0.7393 & 0.5942 & 0.7092 & 0.6499 & 0.7374 & 0.5905 & 0.7087 & 0.6467  \\
TNN+TRPCA-DCT & 0.7739 & 0.5354 & 0.6907 & 0.6322 & 0.7726 & 0.5428 & 0.6923 & 0.6343 & 0.7442 & 0.5740 & 0.7028 & 0.6395 & 0.7422 & 0.5732 & 0.6976 & 0.6339  \\
TNN+OR-TPCA-DFT & 0.0218 & 0.5613 & 0.6566 & 0.5980 & 0.0158 & 0.5708 & 0.6621 & 0.6049 & 0.0070 & 0.5780 & 0.6675 & 0.6089 & 0.0050 & 0.5734 & 0.6646 & 0.6048  \\
TNN+TLRR-DFT & 0.8170 & 0.5554 & 0.6417 & 0.5997 & \textbf{0.8158} & 0.5616 & 0.6505 & 0.6080 & 0.8032 & 0.5600 & 0.6396 & 0.5986 & \textbf{0.8009} & 0.5675 & 0.6488 & 0.6069  \\
TNN+TLRR-DCT & 0.7871 & 0.6268 & 0.7188 & 0.6711 & 0.7876 & 0.6251 & 0.7177 & 0.6695 & 0.7670 & 0.6692 & 0.7415 & 0.7036 & 0.7694 & 0.6680 & 0.7432 & 0.7036  \\
\hline
TNN+OR-TLRR-DFT & 0.7888 & 0.6667 & 0.7499 & 0.6935 & 0.7841 & 0.6653 & 0.7474 & 0.6923 & 0.7653 & 0.6657 & 0.7409 & 0.6868 & 0.7604 & 0.6636 & 0.7408 & 0.6867  \\
TNN+OR-TLRR-DCT & 0.7657 & 0.5757 & 0.6970 & 0.6376 & 0.7617 & 0.5798 & 0.6980 & 0.6405 & 0.7361 & 0.6038 & 0.7119 & 0.6546 & 0.7351 & 0.6116 & 0.7156 & 0.6608  \\
TLRR-EWZF-DFT & 0.8095 & 0.5438 & 0.6255 & 0.5849 & 0.7933 & 0.5160 & 0.6132 & 0.5621 & 0.7923 & 0.5682 & 0.6380 & 0.6034 & 0.7752 & 0.5497 & 0.6182 & 0.5757  \\
TLRR-EWZF-DCT & 0.7610 & 0.6446 & 0.7512 & 0.6781 & 0.7269 & 0.6618 & 0.7518 & 0.6913 & 0.7065 & 0.6391 & 0.7482 & 0.6712 & 0.6648 & 0.6742 & 0.7505 & 0.6931  \\
OR-TLRR-EWZF-DFT & 0.7919 & 0.6430 & 0.7346 & 0.6719 & 0.7905 & 0.6051 & 0.7153 & 0.6434 & 0.7658 & 0.6685 & 0.7426 & 0.6885 & 0.7630 & 0.6489 & 0.7299 & 0.6753  \\
OR-TLRR-EWZF-DCT & \textbf{0.8382} & \textbf{0.6750} & \textbf{0.7586} & \textbf{0.7118} & 0.7955 & \textbf{0.6763} & \textbf{0.7548} & \textbf{0.6981} & \textbf{0.8225} & \textbf{0.6803} & \textbf{0.7598} & \textbf{0.7042} & 0.7496 & \textbf{0.6887} & \textbf{0.7603} & \textbf{0.7038}  \\
\hline
\end{tabular}}
\label{tab:coil20-miss}
\end{table}

\subsection{Application to Image Clustering}

\begin{table}[htbp]
\centering
\caption{Clustering results (ACC, NMI and PUR) on the noisy Umist dataset.}
\setlength{\tabcolsep}{0.1em}{\begin{tabular}{c||c|c|c||c|c|c||c|c|c||c|c|c}
\hline
\multirow{2}{*}{Methods} & \multicolumn{3}{c||}{$\sigma^2 = 0.1$} & \multicolumn{3}{c||}{$\sigma^2 = 0.2$} & \multicolumn{3}{c||}{$\sigma^2 = 0.3$} & \multicolumn{3}{c}{$\sigma^2 = 0.4$}  \\
\cline{2-13}
 & ACC & NMI & PUR & ACC & NMI & PUR & ACC & NMI & PUR & ACC & NMI & PUR  \\
\hline
R-PCA & 0.4252 & 0.5825 & 0.4739 & 0.4230 & 0.5820 & 0.4712 & 0.4185 & 0.5794 & 0.4679 & 0.4215 & 0.5835 & 0.4724  \\
OR-PCA & 0.3955 & 0.5762 & 0.4544 & 0.3985 & 0.5776 & 0.4558 & 0.3950 & 0.5703 & 0.4504 & 0.3975 & 0.5740 & 0.4554  \\
LRR & 0.3400 & 0.4358 & 0.3952 & 0.3255 & 0.4044 & 0.3748 & 0.3178 & 0.3924 & 0.3643 & 0.3170 & 0.3898 & 0.3634  \\
SSC & 0.4848 & 0.6662 & 0.5762 & 0.4810 & 0.6618 & 0.5727 & 0.4846 & 0.6631 & 0.5767 & 0.4762 & 0.6595 & 0.5706  \\
SSC-OMP & 0.4144 & 0.5670 & 0.5093 & 0.4062 & 0.5573 & 0.5006 & 0.4203 & 0.5668 & 0.5105 & 0.4094 & 0.5574 & 0.5061  \\
TRPCA-DFT & 0.5701 & 0.7186 & 0.6324 & 0.5683 & 0.7129 & 0.6305 & 0.5581 & 0.7037 & 0.6219 & 0.5624 & 0.7052 & 0.6258  \\
TRPCA-DCT & 0.5747 & 0.7218 & 0.6389 & 0.5692 & 0.7167 & 0.6358 & 0.5649 & 0.7104 & 0.6324 & 0.5673 & 0.7090 & 0.6336  \\
OR-TPCA & 0.3240 & 0.3642 & 0.3648 & 0.3170 & 0.3556 & 0.3588 & 0.3170 & 0.3508 & 0.3570 & 0.3195 & 0.3530 & 0.3565  \\
TLRR-DFT & 0.4437 & 0.5426 & 0.5143 & 0.4216 & 0.5100 & 0.4810 & 0.4198 & 0.5057 & 0.4778 & 0.4177 & 0.5020 & 0.4729  \\
TLRR-DCT & 0.4956 & 0.6425 & 0.5657 & 0.4834 & 0.6300 & 0.5573 & 0.4712 & 0.6087 & 0.5458 & 0.4720 & 0.6045 & 0.5437  \\
\hline
OR-TLRR-DFT & \textbf{0.5992} & \textbf{0.7429} & \textbf{0.6667} & \textbf{0.5955} & \textbf{0.7408} & \textbf{0.6628} & \textbf{0.5820} & \textbf{0.7320} & \textbf{0.6525} & \textbf{0.5902} & \textbf{0.7317} & \textbf{0.6570}  \\
OR-TLRR-DCT & 0.4930 & 0.6618 & 0.5493 & 0.4907 & 0.6583 & 0.5452 & 0.4909 & 0.6581 & 0.5448 & 0.4984 & 0.6595 & 0.5546  \\
\hline
\end{tabular}}
\label{tab:umistnoisy}
\end{table}

\begin{table}[htbp]
\centering
\caption{Clustering results (ACC, NMI and PUR) on the noisy USPS dataset.}
\setlength{\tabcolsep}{0.1em}{\begin{tabular}{c||c|c|c||c|c|c||c|c|c}
\hline
\multirow{2}{*}{Methods} & \multicolumn{3}{c||}{SR = 0.3} & \multicolumn{3}{c||}{SR = 0.5} & \multicolumn{3}{c}{SR = 0.7}  \\
\cline{2-10}
 & ACC & NMI & PUR & ACC & NMI & PUR & ACC & NMI & PUR  \\
\hline
R-PCA & 0.4633 & 0.4100 & 0.4912 & 0.4572 & 0.4044 & 0.4876 & 0.4566 & 0.3999 & 0.4816  \\
OR-PCA & 0.4489 & 0.4015 & 0.4807 & 0.4482 & 0.4014 & 0.4806 & 0.4468 & 0.3953 & 0.4773  \\
LRR & 0.2911 & 0.2146 & 0.3111 & 0.2503 & 0.1603 & 0.2632 & 0.2540 & 0.1375 & 0.2664  \\
SSC & 0.3890 & 0.4065 & 0.4618 & 0.4025 & 0.4119 & 0.4670 & 0.4266 & 0.4125 & 0.4690  \\
SSC-OMP & 0.1424 & 0.0984 & 0.1939 & 0.1623 & 0.1149 & 0.2147 & 0.2105 & 0.1517 & 0.2624  \\
TRPCA-DFT & 0.3417 & 0.2902 & 0.3767 & 0.3362 & 0.2797 & 0.3744 & 0.3506 & 0.2756 & 0.3739  \\
TRPCA-DCT & 0.3399 & 0.2903 & 0.3768 & 0.3380 & 0.2795 & 0.3748 & 0.3564 & 0.2817 & 0.3781  \\
OR-TPCA & 0.3219 & 0.2533 & 0.3468 & 0.3342 & 0.2685 & 0.3573 & 0.3588 & 0.2787 & 0.3754  \\
TLRR-DFT & 0.2762 & 0.1624 & 0.2892 & 0.2129 & 0.0899 & 0.2233 & 0.1795 & 0.0518 & 0.1870  \\
TLRR-DCT & 0.3526 & 0.2688 & 0.3965 & 0.3203 & 0.2190 & 0.3557 & 0.2610 & 0.1386 & 0.2802  \\
\hline
OR-TLRR-DFT & \textbf{0.4876} & \textbf{0.4228} & \textbf{0.5251} & \textbf{0.5055} & \textbf{0.4263} & \textbf{0.5348} & \textbf{0.4929} & \textbf{0.4252} & \textbf{0.5317}  \\
OR-TLRR-DCT & 0.4069 & 0.3452 & 0.4460 & 0.4074 & 0.3532 & 0.4555 & 0.4246 & 0.3519 & 0.4638  \\
\hline
\end{tabular}}
\label{tab:uspsnoisy}
\end{table}

In order to study the robustness of our methods for dealing with sample-specific corruptions, we add white Gaussian noise to 20\% of the images in the Umist dataset, and we use $\mathbf{\Theta}$ to denote the indices of the selected images. We use $\bcX^{\mathrm{noise}}$ to denote the noisy samples. For each sample $\bcX^{\mathrm{noise}}(:,j,:)$, $j \in \mathbf{\Theta}$, $\bcX^{\mathrm{noise}}(:,j,:)$ can be expressed as $\bcX^{\mathrm{noise}}(:,j,:) = \mathtt{ivec} (\mathtt{vec} (\bcX(:,j,:)) + \boldsymbol{\varphi}_j)$, where $\boldsymbol{\varphi}_j \in \R^{n_1n_3}$ is additive white Gaussian noise with zero mean and variance $\sigma^2 \| \bcX(:,j,:) \|_F^2$. In this experiment, $\sigma^2$ ranges from 0.1 to 0.4 and for each $\sigma^2$, we generate $\bcX^{\mathrm{noise}}$ 10 times. We can again observe from Table~\ref{tab:umistnoisy} that OR-TLRR-DFT achieves the highest clustering accuracy, which suggests the effectiveness of the proposed method for clustering the corrupted data.

We finally consider the USPS dataset, which contains 9298 images of $16 \times 16$ handwritten digits. In this experiment, 100 images are randomly selected from each of the 10 digits; hence $n_2 = 1000$. For each image, we randomly shift the digit horizontally with respect to the center by 3 pixels either side (left or right). Then we generate noisy data by adding white Gaussian noise to different percentages of the shifted image samples, where the sampling rate (SR) is set to be 0.3, 0.5 and 0.7. The noise level $\sigma^2$ is now fixed to be 0.2. One can infer from Table~\ref{tab:uspsnoisy} that OR-TLRR-DFT consistently outperforms other compared methods. By introducing column-sparse noise, the robustness of the proposed method is also verified.

\end{document}

%% file: _defs.tex
\newcommand{\ba}{\mathbf{a}}
\newcommand{\bA}{\mathbf{A}}
\newcommand{\cA}{\mathcal{A}}

\newcommand{\bcA}{\boldsymbol{\cA}}

\newcommand{\bB}{\mathbf{B}}
\newcommand{\cB}{\mathcal{B}}

\newcommand{\bcB}{\boldsymbol{\cB}}

\newcommand{\C}{\mathbb{C}}
\newcommand{\bC}{\mathbf{C}}
\newcommand{\cC}{\mathcal{C}}

\newcommand{\bcC}{\boldsymbol{\cC}}

\newcommand{\bD}{\mathbf{D}}
\newcommand{\cD}{\mathcal{D}}

\newcommand{\bcD}{\boldsymbol{\cD}}

\newcommand{\be}{\mathbf{e}}
\newcommand{\ce}{\mathfrak{e}}

\newcommand{\bE}{\mathbf{E}}
\newcommand{\cE}{\mathcal{E}}

\newcommand{\bcE}{\boldsymbol{\cE}}

\newcommand{\cF}{\mathcal{F}}

\newcommand{\bcF}{\boldsymbol{\cF}}

\newcommand{\cG}{\mathcal{G}}

\newcommand{\bcG}{\boldsymbol{\cG}}

\newcommand{\cH}{\mathcal{H}}

\newcommand{\bH}{\mathbf{H}}

\newcommand{\bcH}{\boldsymbol{\cH}}

\newcommand{\bI}{\mathbf{I}}
\newcommand{\cI}{\mathcal{I}}

\newcommand{\bcI}{\boldsymbol{\cI}}

\newcommand{\bJ}{\mathbf{J}}
\newcommand{\cJ}{\mathcal{J}}

\newcommand{\bcJ}{\boldsymbol{\cJ}}

\newcommand{\bL}{\mathbf{L}}
\newcommand{\cL}{\mathcal{L}}

\newcommand{\bcL}{\boldsymbol{\cL}}

\newcommand{\cN}{\mathcal{N}}

\newcommand{\cO}{\mathcal{O}}

\newcommand{\bP}{\mathbf{P}}
\newcommand{\cP}{\mathcal{P}}

\newcommand{\bcP}{\boldsymbol{\cP}}

\newcommand{\bQ}{\mathbf{Q}}
\newcommand{\cQ}{\mathcal{Q}}

\newcommand{\bcQ}{\boldsymbol{\cQ}}

\newcommand{\R}{\mathbb{R}}

\newcommand{\cR}{\mathcal{R}}

\newcommand{\bcR}{\boldsymbol{\cR}}

\newcommand{\bS}{\mathbf{S}}
\newcommand{\cS}{\mathcal{S}}

\newcommand{\bcS}{\boldsymbol{\cS}}

\newcommand{\cT}{\mathcal{T}}

\newcommand{\bcT}{\boldsymbol{\cT}}

\newcommand{\cU}{\mathcal{U}}
\newcommand{\bU}{\mathbf{U}}

\newcommand{\bcU}{\boldsymbol{\cU}}

\newcommand{\bV}{\mathbf{V}}
\newcommand{\cV}{\mathcal{V}}

\newcommand{\bcV}{\boldsymbol{\cV}}

\newcommand{\cW}{\mathcal{W}}

\newcommand{\bcW}{\boldsymbol{\cW}}

\newcommand{\bx}{\mathbf{x}}

\newcommand{\bX}{\mathbf{X}}
\newcommand{\cX}{\mathcal{X}}

\newcommand{\bcX}{\boldsymbol{\cX}}

\newcommand{\bY}{\mathbf{Y}}
\newcommand{\cY}{\mathcal{Y}}

\newcommand{\bcY}{\boldsymbol{\cY}}

\newcommand{\bz}{\mathbf{z}}

\newcommand{\bZ}{\mathbf{Z}}
\newcommand{\cZ}{\mathcal{Z}}

\newcommand{\bcZ}{\boldsymbol{\cZ}}

\DeclareMathOperator*{\argmin}{arg\,min}

\newcommand{\tr}{\mathrm{tr}}


%% file: aistats2024_TW_Final.bbl
\begin{thebibliography}{}

\bibitem[Cai et~al., 2010]{CaiCS.SIAM2010}
Cai, J.-F., Cand\`{e}s, E.~J., and Shen, Z. (2010).
\newblock A singular value thresholding algorithm for matrix completion.
\newblock {\em SIAM Journal on Optimization}, 20(4):1956--1982.

\bibitem[Cand{\`e}s et~al., 2011]{CandesLMW.JACM2011}
Cand{\`e}s, E.~J., Li, X., Ma, Y., and Wright, J. (2011).
\newblock Robust principal component analysis?
\newblock {\em Journal of the ACM}, 58(3):1--37.

\bibitem[Donoho, 1995]{Donoho.TIT1995}
Donoho, D.~L. (1995).
\newblock De-noising by soft-thresholding.
\newblock {\em IEEE Transactions on Information Theory}, 41(3):613--627.

\bibitem[Elhamifar and Vidal, 2013]{ElhamifarV.PAMI2013}
Elhamifar, E. and Vidal, R. (2013).
\newblock Sparse subspace clustering: Algorithm, theory, and applications.
\newblock {\em IEEE Transactions on Pattern Analysis and Machine Intelligence},
  35(11):2765--2781.

\bibitem[Eriksson et~al., 2012]{ErikssonBN.AISTATS2012}
Eriksson, B., Balzano, L., and Nowak, R. (2012).
\newblock High-rank matrix completion.
\newblock In {\em Proceedings of the International Conference on Artificial
  Intelligence and Statistics}, pages 373--381.

\bibitem[Fu et~al., 2016]{FuGTLH.TNNLS2016}
Fu, Y., Gao, J., Tien, D., Lin, Z., and Hong, X. (2016).
\newblock Tensor {LRR} and sparse coding-based subspace clustering.
\newblock {\em IEEE Transactions on Neural Networks and Learning Systems},
  27(10):2120--2133.

\bibitem[Huiskes and Lew, 2008]{HuiskesL.ACM2008}
Huiskes, M.~J. and Lew, M.~S. (2008).
\newblock The {MIR} flickr retrieval evaluation.
\newblock In {\em Proceedings of the ACM International Conference on Multimedia
  Information Retrieval}, pages 39--43.

\bibitem[Hull, 1994]{Hull.PAMI1994}
Hull, J.~J. (1994).
\newblock A database for handwritten text recognition research.
\newblock {\em IEEE Transactions on Pattern Analysis and Machine Intelligence},
  5(16):550--554.

\bibitem[Ji et~al., 2017]{JiZLSR.NIPS2017}
Ji, P., Zhang, T., Li, H., Salzmann, M., and Reid, I. (2017).
\newblock Deep subspace clustering networks.
\newblock In {\em Proceedings of the Annual Conference on Neural Information
  Processing Systems}, volume~30, pages 24--33.

\bibitem[Ji et~al., 2019]{JiVH.ICCV2019}
Ji, X., Vedaldi, A., and Henriques, J. (2019).
\newblock Invariant information clustering for unsupervised image
  classification and segmentation.
\newblock In {\em Proceedings of the IEEE/CVF Conference on Computer Vision},
  pages 9865--9874.

\bibitem[Jiang et~al., 2020]{JiangNZH.TIP2020}
Jiang, T.-X., Ng, M.~K., Zhao, X.-L., and Huang, T.-Z. (2020).
\newblock Framelet representation of tensor nuclear norm for third-order tensor
  completion.
\newblock {\em IEEE Transactions on Image Processing}, 29:7233--7244.

\bibitem[Jiang et~al., 2023]{JiangZZN.TNNLS2023}
Jiang, T.-X., Zhao, X.-L., Zhang, H., and Ng, M.~K. (2023).
\newblock Dictionary learning with low-rank coding coefficients for tensor
  completion.
\newblock {\em IEEE Transactions on Neural Networks and Learning Systems},
  34(2):932--946.

\bibitem[Kernfeld et~al., 2014]{KernfeldAK.arxiv2014}
Kernfeld, E., Aeron, S., and Kilmer, M. (2014).
\newblock Clustering multi-way data: A novel algebraic approach.
\newblock {\em arXiv preprint}.

\bibitem[Kernfeld et~al., 2015]{KernfeldKA.LAA2015}
Kernfeld, E., Kilmer, M., and Aeron, S. (2015).
\newblock Tensor-tensor products with invertible linear transforms.
\newblock {\em Linear Algebra and its Applications}, 485:545--570.

\bibitem[Kilmer et~al., 2013]{KilmerBHH.SIAM2013}
Kilmer, M.~E., Braman, K., Hao, N., and Hoover, R.~C. (2013).
\newblock Third-order tensors as operators on matrices: A theoretical and
  computational framework with applications in imaging.
\newblock {\em SIAM Journal on Matrix Analysis and Applications},
  34(1):148--172.

\bibitem[Kilmer and Martin, 2011]{KilmerM.LAA2011}
Kilmer, M.~E. and Martin, C.~D. (2011).
\newblock Factorization strategies for third-order tensors.
\newblock {\em Linear Algebra and its Applications}, 435(3):641--658.

\bibitem[Kolda and Bader, 2009]{KoldaB.Rev2009}
Kolda, T.~G. and Bader, B.~W. (2009).
\newblock Tensor decompositions and applications.
\newblock {\em SIAM Review}, 51(3):455--500.

\bibitem[Kong et~al., 2021]{KongLL.ML2021}
Kong, H., Lu, C., and Lin, Z. (2021).
\newblock Tensor {Q}-rank: new data dependent definition of tensor rank.
\newblock {\em Machine Learning}, 110(7):1867--1900.

\bibitem[Li et~al., 2021]{LiHLPZP.AAAI2021}
Li, Y., Hu, P., Liu, Z., Peng, D., Zhou, J.~T., and Peng, X. (2021).
\newblock Contrastive clustering.
\newblock In {\em Proceedings of the AAAI Conference on Artificial
  Intelligence}, volume~35, pages 8547--8555.

\bibitem[Lin et~al., 2011]{LinLS.NIPS2011}
Lin, Z., Liu, R., and Su, Z. (2011).
\newblock Linearized alternating direction method with adaptive penalty for
  low-rank representation.
\newblock In {\em Proceedings of the Annual Conference on Neural Information
  Processing Systems}, pages 612--620.

\bibitem[Liu et~al., 2013a]{LiuLYSYM.PAMI2013}
Liu, G., Lin, Z., Yan, S., Sun, J., Yu, Y., and Ma, Y. (2013a).
\newblock Robust recovery of subspace structures by low-rank representation.
\newblock {\em IEEE Transactions on Pattern Analysis and Machine Intelligence},
  35(1):171--184.

\bibitem[Liu et~al., 2012]{LiuXY.AISTATS2012}
Liu, G., Xu, H., and Yan, S. (2012).
\newblock Exact subspace segmentation and outlier detection by low-rank
  representation.
\newblock In {\em Proceedings of the International Conference on Artificial
  Intelligence and Statistics}, pages 703--711.

\bibitem[Liu and Yan, 2011]{LiuY.ICCV2011}
Liu, G. and Yan, S. (2011).
\newblock Latent low-rank representation for subspace segmentation and feature
  extraction.
\newblock In {\em Proceedings of the IEEE Conference on Computer Vision}, pages
  1615--1622.

\bibitem[Liu et~al., 2013b]{LiuMWY.PAMI2013}
Liu, J., Musialski, P., Wonka, P., and Ye, J. (2013b).
\newblock Tensor completion for estimating missing values in visual data.
\newblock {\em IEEE Transactions on Pattern Analysis and Machine Intelligence},
  35(1):208--220.

\bibitem[Lu, 2021]{Lu.ICCV2021}
Lu, C. (2021).
\newblock Transforms based tensor robust {PCA}: Corrupted low-rank tensors
  recovery via convex optimization.
\newblock In {\em Proceedings of the IEEE/CVF Conference on Computer Vision},
  pages 1145--1152.

\bibitem[Lu et~al., 2020]{LuFCLLY.PAMI2020}
Lu, C., Feng, J., Chen, Y., Liu, W., Lin, Z., and Yan, S. (2020).
\newblock Tensor robust principal component analysis with a new tensor nuclear
  norm.
\newblock {\em IEEE Transactions on Pattern Analysis and Machine Intelligence},
  42(4):925--938.

\bibitem[Lu et~al., 2019a]{LuFLMY.PAMI2019}
Lu, C., Feng, J., Lin, Z., Mei, T., and Yan, S. (2019a).
\newblock Subspace clustering by block diagonal representation.
\newblock {\em IEEE Transactions on Pattern Analysis and Machine Intelligence},
  41(2):487--501.

\bibitem[Lu et~al., 2018]{LuFLY.IJCAI2018}
Lu, C., Feng, J., Lin, Z., and Yan, S. (2018).
\newblock Exact low tubal rank tensor recovery from {G}aussian measurements.
\newblock In {\em Proceedings of the International Joint Conferences on
  Artificial Intelligence}, pages 2504--2510.

\bibitem[Lu et~al., 2019b]{LuPW.CVPR2019}
Lu, C., Peng, X., and Wei, Y. (2019b).
\newblock Low-rank tensor completion with a new tensor nuclear norm induced by
  invertible linear transforms.
\newblock In {\em Proceedings of the IEEE/CVF Conference on Computer Vision and
  Pattern Recognition}, pages 5996--6004.

\bibitem[Manning et~al., 2010]{ManningRS.2010}
Manning, C., Raghavan, P., and Sch{\"{u}}tze, H. (2010).
\newblock {\em Introduction to information retrieval}.
\newblock Cambridge University Press.

\bibitem[Mardani et~al., 2015]{MardaniMG.TSP2015}
Mardani, M., Mateos, G., and Giannakis, G.~B. (2015).
\newblock Subspace learning and imputation for streaming big data matrices and
  tensors.
\newblock {\em IEEE Transactions on Signal Processing}, 63(10):2663--2677.

\bibitem[Qin et~al., 2022]{Qin.etal.TIP2022}
Qin, W., Wang, H., Zhang, F., Wang, J., Luo, X., and Huang, T. (2022).
\newblock Low-rank high-order tensor completion with applications in visual
  data.
\newblock {\em IEEE Transactions on Image Processing}, 31:2433--2448.

\bibitem[Shi and Malik, 2000]{ShiM.PAMI2000}
Shi, J. and Malik, J. (2000).
\newblock Normalized cuts and image segmentation.
\newblock {\em IEEE Transactions on Pattern Analysis and Machine Intelligence},
  22(8):888--905.

\bibitem[Song et~al., 2020]{SongNZ.NLAA2020}
Song, G., Ng, M.~K., and Zhang, X. (2020).
\newblock Robust tensor completion using transformed tensor singular value
  decomposition.
\newblock {\em Numerical Linear Algebra with Applications}, 27(3):e2299.

\bibitem[Su et~al., 2023]{SuHWL.SP2023}
Su, Y., Hong, Z., Wu, X., and Lu, C. (2023).
\newblock Invertible linear transforms based adaptive multi-view subspace
  clustering.
\newblock {\em Signal Processing}, 209:109014.

\bibitem[Tang et~al., 2014]{TangLSZ.TNNLS2014}
Tang, K., Liu, R., Su, Z., and Zhang, J. (2014).
\newblock Structure-constrained low-rank representation.
\newblock {\em IEEE Transactions on Neural Networks and Learning Systems},
  25(12):2167--2179.

\bibitem[Tropp, 2012]{Tropp.FoCM2012}
Tropp, J.~A. (2012).
\newblock User-friendly tail bounds for sums of random matrices.
\newblock {\em Foundations of Computational Mathematics}, 12:389--434.

\bibitem[Vinh et~al., 2010]{VinhEB.JMLR2010}
Vinh, N.~X., Epps, J., and Bailey, J. (2010).
\newblock Information theoretic measures for clusterings comparison: Variants,
  properties, normalization and correction for chance.
\newblock {\em Journal of Machine Learning Research}, 11:2837--2854.

\bibitem[Wang et~al., 2021]{Wang.etal.TIP2021}
Wang, J.-L., Huang, T.-Z., Zhao, X.-L., Jiang, T.-X., and Ng, M.~K. (2021).
\newblock Multi-dimensional visual data completion via low-rank tensor
  representation under coupled transform.
\newblock {\em IEEE Transactions on Image Processing}, 30:3581--3596.

\bibitem[Wu, 2020]{Wu.PR2020}
Wu, T. (2020).
\newblock Graph regularized low-rank representation for submodule clustering.
\newblock {\em Pattern Recognition}, 100:107145.

\bibitem[Wu, 2023]{Wu.TCSVT2023}
Wu, T. (2023).
\newblock Online tensor low-rank representation for streaming data clustering.
\newblock {\em IEEE Transactions on Circuits and Systems for Video Technology},
  33(2):602--617.

\bibitem[Wu and Bajwa, 2018]{WuB.SPL2018}
Wu, T. and Bajwa, W.~U. (2018).
\newblock A low tensor-rank representation approach for clustering of imaging
  data.
\newblock {\em IEEE Signal Processing Letters}, 25(8):1196--1200.

\bibitem[Xia et~al., 2021]{XiaCSL.TNNLS2021}
Xia, G., Chen, B., Sun, H., and Liu, Q. (2021).
\newblock Nonconvex low-rank kernel sparse subspace learning for keyframe
  extraction and motion segmentation.
\newblock {\em IEEE Transactions on Neural Networks and Learning Systems},
  32(4):1612--1626.

\bibitem[Xie et~al., 2016]{XieGF.ICML2016}
Xie, J., Girshick, R., and Farhadi, A. (2016).
\newblock Unsupervised deep embedding for clustering analysis.
\newblock In {\em Proceedings of the International Conference on Machine
  Learning}, pages 478--487.

\bibitem[Xie et~al., 2018]{XieTZLZQ.IJCV2018}
Xie, Y., Tao, D., Zhang, W., Liu, Y., Zhang, L., and Qu, Y. (2018).
\newblock On unifying multi-view self-representations for clustering by tensor
  multi-rank minimization.
\newblock {\em International Journal of Computer Vision}, 126(11):1157--1179.

\bibitem[Yang et~al., 2015]{YangRV.ICML2015}
Yang, C., Robinson, D., and Vidal, R. (2015).
\newblock Sparse subspace clustering with missing entries.
\newblock In {\em Proceedings of the International Conference on Machine
  Learning}, pages 2463--2472.

\bibitem[Yang et~al., 2022]{Yang.etal.TNNLS2022}
Yang, J.-H., Chen, C., Dai, H.-N., Ding, M., Wu, Z.-B., and Zheng, Z. (2022).
\newblock Robust corrupted data recovery and clustering via generalized
  transformed tensor low-rank representation.
\newblock {\em IEEE Transactions on Neural Networks and Learning Systems}.

\bibitem[You et~al., 2016]{YouRV.CVPR2016}
You, C., Robinson, D.~P., and Vidal, R. (2016).
\newblock Scalable sparse subspace clustering by orthogonal matching pursuit.
\newblock In {\em Proceedings of the IEEE/CVF Conference on Computer Vision and
  Pattern Recognition}, pages 3918--3927.

\bibitem[Zhang et~al., 2015]{ZhangLZC.AAAI2015}
Zhang, H., Lin, Z., Zhang, C., and Chang, E. (2015).
\newblock Exact recoverability of robust {PCA} via outlier pursuit with tight
  recovery bounds.
\newblock In {\em Proceedings of the AAAI Conference on Artificial
  Intelligence}.

\bibitem[Zhang et~al., 2018]{ZhangLJLS.SPL2018}
Zhang, J., Li, X., Jing, P., Liu, J., and Su, Y. (2018).
\newblock Low-rank regularized heterogeneous tensor decomposition for subspace
  clustering.
\newblock {\em IEEE Signal Processing Letters}, 25(3):333--337.

\bibitem[Zhou and Feng, 2017]{ZhouF.CVPR2017}
Zhou, P. and Feng, J. (2017).
\newblock Outlier-robust tensor {PCA}.
\newblock In {\em Proceedings of the IEEE/CVF Conference on Computer Vision and
  Pattern Recognition}, pages 2263--2271.

\bibitem[Zhou et~al., 2021]{ZhouLFLY.PAMI2021}
Zhou, P., Lu, C., Feng, J., Lin, Z., and Yan, S. (2021).
\newblock Tensor low-rank representation for data recovery and clustering.
\newblock {\em IEEE Transactions on Pattern Analysis and Machine Intelligence},
  43(5):1718--1732.

\end{thebibliography}
